\documentclass{article}

\usepackage{microtype}
\usepackage{graphicx}
\usepackage{subcaption}
\usepackage{booktabs} %

\usepackage{enumitem}

\usepackage{hyperref}

\usepackage[preprint]{icml2026}

\usepackage{amsmath}
\usepackage{amssymb}
\usepackage{mathtools}
\usepackage{amsthm}
\usepackage{bbm}
\usepackage{thm-restate}

\usepackage{comment}

\usepackage[textwidth=2cm,textsize=tiny]{todonotes}

\usepackage[capitalize,noabbrev]{cleveref}

\theoremstyle{plain}
\newtheorem{theorem}{Theorem}[section]
\newtheorem{proposition}[theorem]{Proposition}
\newtheorem{lemma}[theorem]{Lemma}

\theoremstyle{definition}
\newtheorem{definition}[theorem]{Definition}
\newtheorem{assumption}[theorem]{Assumption}
\theoremstyle{remark}
\newtheorem{remark}[theorem]{Remark}

\newcommand{\cause}[1]{&{\color{gray}\,~\downarrow{}\;{\text{\small #1}}} \nonumber\\}

\newtheorem{example}{Example}

\newcommand{\NN}{{\mathbb N}} %
\newcommand{\RR}{{\mathbb R}} %

\newcommand{\BB}{{\mathbb B}}
\newcommand{\TT}{{\mathbb T}}

\newcommand{\CC}{{\mathbb C}}
\newcommand{\WW}{{\mathbb W}}

\newcommand{\EE}{{\mathbb E}} %
\newcommand{\VV}{{\mathbb V}} %
\newcommand{\PP}{{\mathbb P}} %

\newcommand{\loss}{{\mathcal L}} %
\newcommand{\risk}{{\mathcal R}} %

\newcommand{\sA}{{\mathcal A}}
\newcommand{\sB}{{\mathcal B}}
\newcommand{\sD}{{\mathcal D}}
\newcommand{\sC}{{\mathcal C}}
\newcommand{\sX}{{\mathcal X}}
\newcommand{\sY}{{\mathcal Y}}

\newcommand{\sZ}{{\mathcal Z}}
\newcommand{\sK}{{\mathcal K}}

\newcommand{\sG}{{\mathcal G}}
\newcommand{\sN}{{\mathcal N}}
\newcommand{\sM}{{\mathcal M}}
\newcommand{\sO}{{\mathcal O}}

\newcommand{\sigmoid}{S}

\def\eqsp{\enspace}

\newcommand{\bigabs}[1]{\Big| #1 \Big|}
\newcommand{\norm}[1]{\left\|{#1}\right\|} %

\newcommand{\Id}[1]{\mathrm{I}_{#1}}
\DeclareMathOperator{\Tr}{Tr}

\newcommand{\funcspace}[1]{{\Phi_{M,#1}(\RR^p)}}

\usepackage[textsize=tiny]{todonotes}

\icmltitlerunning{Learning with Locally Private Examples by IWP-SGD}

\begin{document}

\twocolumn[
  \icmltitle{Learning with Locally Private Examples by \\Inverse Weierstrass Private Stochastic Gradient Descent}

  \icmlsetsymbol{equal}{*}

  \begin{icmlauthorlist}
    \icmlauthor{Jean Dufraiche}{InriaLille}
    \icmlauthor{Paul Mangold}{Xparis}
    \icmlauthor{Michael Perrot}{InriaLille}
    \icmlauthor{Marc Tommasi}{InriaLille}
  \end{icmlauthorlist}

  \icmlaffiliation{InriaLille}{Univ. Lille, Inria, CNRS, Centrale Lille, UMR 9189 - CRIStAL, F-59000 Lille, France}
  \icmlaffiliation{Xparis}{CMAP, CNRS, École polytechnique, Institut Polytechnique de Paris, 91120 Palaiseau, France}

  \icmlcorrespondingauthor{Jean Dufraiche}{jean.dufraiche@inria.fr}

  \icmlkeywords{Machine Learning, ICML}

  \vskip 0.3in
]

\printAffiliationsAndNotice{}  %

\begin{abstract}
    Releasing data once and for all under noninteractive Local Differential Privacy (LDP) enables complete data reusability, but the resulting noise may create bias in subsequent analyses. In this work, we leverage the Weierstrass transform to characterize this bias in binary classification. We prove that inverting this transform leads to a bias-correction method to compute unbiased estimates of nonlinear functions on examples released under LDP. We then build a novel stochastic gradient descent algorithm called Inverse Weierstrass Private SGD (IWP-SGD). It converges to the true population risk minimizer at a rate of $\mathcal{O}(1/n)$, with $n$ the number of examples. We empirically validate IWP-SGD on binary classification tasks using synthetic and real-world datasets.
\end{abstract}

\section{Introduction}

Machine Learning (ML) models are increasingly deployed in domains involving sensitive data, such as healthcare, speech recognition, prediction, and forecasting. 
These models are vulnerable to inference attacks that allow adversaries to extract information about individual training examples \citep{MIA_survey}. 
This has motivated the use of Differential Privacy (DP) \citep{dwork_DP} as a rigorous standard to assess privacy in ML.
To achieve meaningful guarantees, DP typically requires data to be centralized by a trusted curator, in charge of enforcing privacy. 
Unfortunately, this raises several risks: the trusted authority may fall victim to attacks that lead to major data breaches \citep{breach_data,data_breach_2}, and data may be misappropriated by untrustworthy third parties that do not prioritize privacy. 

Local Differential Privacy (LDP) \citep{kasiviswanathan2010learnprivately,minimax_ldp_2} addresses this challenge by requiring each data holder to privatize their data locally
before release, effectively ensuring privacy without relying on a trusted curator. 
While this provides strong privacy guarantees, applying it in ML requires adapting the downstream learning process.
Existing methods can be categorized into interactive and noninteractive approaches. In \emph{interactive} methods, the learner adaptively queries data holders over multiple rounds, incurring a communication cost \citep{noisy_oracle}. In contrast, \emph{noninteractive} methods require each user to release one or several privatized versions of their data in a single shot, eliminating the need for further interaction during learning  \citep{zheng2017collectonceuseeffectively}.

In practice, designing LDP mechanisms involves two considerations: whether downstream learning tasks are known, and whether data release can be adapted during learning.
In some scenarios, the \emph{downstream learning problem is known}, and task-specific algorithms can be designed to correct for LDP noise; however, this limits the potential for the data to be reused for other purposes.
In contrast, many real-world scenarios \emph{involve unknown downstream tasks} or require that data remain reusable in the long run.
This motivates the use of task-agnostic, noninteractive LDP methods.

In \emph{task-agnostic LDP}, each data holder publishes a one-time privatized representation of their data, without prior knowledge of the downstream learning task.
Such a mechanism allows institutions or users to safely publish privatized datasets that remain usable for future analyses, for example, hospitals sharing medical records for research purposes. Yet, despite its generality, noninteractive and task-agnostic LDP raises a significant challenge.
Indeed, learning from noisy (private) data may bias the process, as previously identified in supervised learning with noisy features \citep{bishop} and labels \citep{noisy_labels}.
Naively applying standard ML frameworks to privatized data may yield suboptimal models: new algorithms tailored for noninteractive and task-agnostic LDP are thus needed. 

\paragraph{Contributions.}
In this paper, we develop a principled view of learning under noninteractive and task-agnostic LDP, and design new algorithms for locally private ML.
We show that standard LDP mechanisms can be viewed, in expectation, as functional transforms: the Gaussian mechanism corresponds to the Weierstrass transform, while Randomized Response induces what we call the Bernoulli transform. This perspective allows us to fully characterize the bias induced by LDP on data-dependent computations. 

Crucially, \emph{inverting these transforms} allows the design of algorithms that provably mitigate privacy-induced bias, yielding unbiased estimators for the underlying data-dependent quantities.
In learning contexts, we leverage the \emph{inverse} of these transforms to construct unbiased gradient estimators for loss functions. 
Applying this principle to first-order optimization, we introduce \emph{Inverse Weierstrass Private Stochastic Gradient Descent (IWP-SGD)}.
We show that IWP-SGD asymptotically recovers, in expectation over the noise, the population risk minimizer of the original, non-private problem, as the number of samples $n$ grows to infinity. Interestingly, the convergence rate of IWP-SGD scales as $\sO(1/n)$ similarly to classic interactive LDP approaches \citep{noisy_oracle}. Finally, we empirically validate IWP-SGD on binary classification tasks using synthetic and real-world datasets. To the best of our knowledge, this is the first method that asymptotically recovers the non-private population risk minimizer in a fully task-agnostic LDP setting using a single privatized release per data point.

Our contributions can be summarized as follows:
\begin{itemize}[leftmargin=*]
    \item We formalize the processing of data released under the Gaussian and Randomized Response mechanisms as transform operators of the intended computations, enabling a unified analysis of their induced bias (Section~\ref{sec:privacy_transform}). This view allows us to fully characterize the bias induced by the Gaussian and Randomized Response in standard risk minimization for binary classification (Section~\ref{sec:bias_evidences}).
    \item We construct an unbiased gradient estimator by inverting the transform associated with gradient computation on LDP records. Using this estimator, we propose Inverse Weierstrass Private SGD (IWP-SGD). We formally analyze IWP-SGD, showing that it recovers, in expectation, the solution to the original problem (Section~\ref{sec:bias_correction}).
    \item We empirically evaluate IWP-SGD on synthetic and real-world binary classification tasks, showing that it successfully removes the bias induced by LDP (Section~\ref{sec:experiments}).
\end{itemize}

\subsection{Related Work}

\paragraph{Interactive LDP Methods.}
Interactive methods permit adaptive communication between learners and data owners, with each owner potentially answering multiple sequential queries. A famous example is distributed SGD, where each data owner shares a noisy version of the gradient computed on their local data \citep{noisy_oracle,minimax_ldp_2}. It is also the root of private SGD-based algorithms in \citet{sparse_reg} to perform sparse linear regression. As the learner explicitly queries gradient evaluations at successive model updates, these approaches do not permit data reusability and suffer from a large communication cost.

\paragraph{Noninteractive and Task-Specific LDP Methods.}
Noninteractive methods, by definition, prohibit adaptive communication: data owners release one or several privatized statistics only once. For example, \citet{wang2020noninteractivelocallyprivatelearning} and \citet{zheng2017collectonceuseeffectively} tackle the estimation of generalized linear models under LDP using Chebychev and Bernstein polynomial approximations of the loss gradients. These approaches use multiple Gaussian-perturbed versions of each data point to construct a biased gradient estimator for which the bias shrinks with the number of noisy data releases. 
In \citet{ERM_non_interactive_loss_noisy}, data owners compute noisy loss evaluations over a grid of model values to estimate the population risk, which is optimized later. As a grid-based method, it suffers from an exponential dimension dependency, and loss evaluations are not reusable for other ML problems involving different losses.
In contrast, our method only requires one noisy release to construct an unbiased estimator, directly reduces noise variance, while allowing reusing the data for a large class of downstream tasks.

\paragraph{Learning with Noisy Data.}
Beyond LDP, several notable works have considered settings where data points are subject to local randomization, although these studies are not directly concerned with privacy. 
\citet{bishop} approximates the bias induced by Gaussian noise addition in the features as an implicit regularization, while \citet{noisy_labels} models label corruption identically to the way we model the Randomized Response mechanism in Section~\ref{subsec:bernoulli}: our bias characterization encompasses both as special cases. 
\citet{scaman2024minimaxexcessriskfirstorder} identifies a learning bias when training and test data distributions differ in a worst-case scenario. Prior works on deconvolution methods \citep{deconv_density_lit} aim to recover the density of data from repeated noisy observations.
In our work, we do not aim to estimate the noiseless distribution, but directly tackle computations performed on noisy inputs.
\paragraph{Noninteractive and Task-Agnostic LDP.}
Several instances of noninteractive and task-agnostic methods exist in the literature. 
\citet{zheng2017collectonceuseeffectively} study a debiasing method for sparse linear regression, while \citet{sparse_reg} consider the case where only labels are private.
\citet{minimax_ldp_2} study the optimal noninteractive and task-agnostic LDP methods for mean and median estimation. Our method proposes a principled solution that generalizes these results to a broader class of learning problems under Gaussian and Randomized Response LDP mechanisms.

\section{Privacy Setting and Notations}
\label{sec:setup}

\paragraph{Notations.} We consider a supervised learning setting with a bounded feature space $\mathcal{X} \subset \RR^{p}$ and a binary label space $\sY=\{-1,1\}$. Let $\sD$ be a joint distribution over $\sX\times \sY$, and let $(x,y)$ be an example drawn from $\sD$. We denote by $\norm{\cdot}$ the euclidean norm and for any subset $\sZ\subset\RR^d$, we write $\norm{\sZ}=\sup_{z\in\sZ}\norm{z}$. The Laplacian of a twice differentiable function $f$ is $\Delta[f] = \sum_i \partial^2_{x_i}[f]$ and its composition $k$ times is denoted $\Delta^k[f] = (\Delta\circ\dots\circ\Delta)[f]$.

\paragraph{Privacy.}
For the remainder of the paper, we consider a privacy setting in which data is released once and for all, without adapting the mechanism to any specific downstream task. We call it the \emph{task-agnostic} setting.
To this end, we leverage Local Differential Privacy, defined as follows.
\begin{definition}[Local Differential Privacy (LDP) \citep{kasiviswanathan2010learnprivately}]
    Let $\sM: E \rightarrow F$ be a randomized algorithm. Let $\epsilon,\delta >0$, the mechanism $\sM$ satisfies $(\epsilon,\delta)$-LDP if, for any $z,z'\in E$ and any subset $\sC\subseteq F$,
    \begin{equation*}
         \PP(\sM(z) \in \sC)\le e^\epsilon \PP(\sM(z')\in \sC) + \delta.
        \label{eq:LDP}
    \end{equation*}
    If $\delta=0$, we say that $\sM$ satisfies $\epsilon$-LDP.
    \label{def:LDP}
\end{definition}
To enforce LDP, one can use the Gaussian mechanism \citep{dwork_DP} for continuous variables.
\begin{proposition}[Gaussian Mechanism]\label{ex:gaussian}
    Assume a bounded subset $\sX\subset\RR^p$. By the Gaussian mechanism, the release of 
    \begin{align*}
        \sG_{\epsilon,\delta}(x) &= x+w,\quad w\sim \sN\left(0,\sigma^2\Id{p}\right),
    \end{align*}
    with $\sigma^2=8\log(1.25/\delta)\norm{\sX}^2/\epsilon^2$ is $(\epsilon,\delta)$-LDP.
\end{proposition}
Similarly, one can use the Randomized Response \citep{RR} to enforce LDP for binary variables.
\begin{proposition}[Randomized Response (RR)]\label{ex:RR}
    Assume $\sY = \{-1,1\}$. By the Randomized Response, the release of
    $$\sB_\epsilon(y) = \begin{cases}
        y \;\text{ with probability }\; \sigmoid(\epsilon) \eqsp,\\
        -y \;\text{ with probability }\; 1-\sigmoid(\epsilon) \eqsp,
    \end{cases}$$
    where $\sigmoid(\epsilon) = 1/(1+e^{-\epsilon})$, is $\epsilon$-LDP.
\end{proposition}
Throughout the paper, we consider the following task-agnostic $(\epsilon,\delta)$-LDP mechanism, which releases continuous features with the Gaussian mechanism, and a binary label with RR. Formally, for an example $(x,y)$ drawn from $\sD$, 
\begin{align}
\label{eq:ldp_release}
(\tilde x, \;\tilde y) &= (\sG_{\epsilon_x,\delta}(x),\; \sB_{\epsilon_y}(y))\eqsp.
\end{align}
The total privacy guarantee of this mechanism is $\epsilon=\epsilon_x+\epsilon_y$, combining budgets over features and labels.

\section{Privacy as a Transform}
\label{sec:privacy_transform}
When learning from the LDP release defined in Equation~\eqref{eq:ldp_release}, any data-dependent quantity, such as a loss or a gradient, can be viewed as a function $h$ evaluated on a randomized version of the data.
Ideally, one would like these quantities to be unbiased, in the sense that their expectation with respect to the privacy noise coincides with the value of $h$ evaluated on the original data $(x,y)$. However, this is generally not the case. Instead, local randomization of data induces a systematic transformation of the function $h$. We formalize this with a transform as follows:
$$\TT_{\epsilon,\delta}[h]:\;(x,\;y)\mapsto \EE_{(\tilde x, \tilde y)}[h(\tilde x, \;\tilde y)].$$
This operator maps the original function $h$ to its average evaluation on noisy releases $(\tilde x,\tilde y)$ of a given data point $(x,y)$. This perspective fully captures the effect of local randomization. As a consequence, bias analysis and correction can be carried out by the study of $\TT_{\epsilon,\delta}$ without any assumptions about the data distribution. In this section, we first study the transforms associated with the Gaussian and Randomized Response mechanisms in isolation.  Then, we present how to characterize the joint effect of both mechanisms through the composition of their respective transforms.

\subsection{Weierstrass Transform: a Tool for Gaussian Noise}

First, we remark that applying a Gaussian noise to the inputs of an arbitrary function and considering the expectation induces a Gaussian smoothing operator known as the Generalized Weierstrass transform \citep{weierstrass}.
\begin{definition}[Generalized Weierstrass transform]
Let $f:\RR^p \rightarrow \RR$. The Weierstrass transform is the function $\WW_{\sigma^2}[f]$ defined for any $x\in\RR^p$ and $\sigma>0$ as
\begin{equation*}
        \WW_{\sigma^2}[f](x) =\EE_{w\sim\sN(0,\sigma^2\Id{p})}\left[ f(x+w)\right].
\end{equation*}
\end{definition}
The alternative parameterization $\sigma^2=2t$ is commonly used in the literature on the Weierstrass transform.
We focus on a class of sufficiently regular functions for which the Weierstrass transform admits a well-defined series representation \citep{weierstrass,heat_equation_1}.
\begin{definition}[Class of Gaussian growing and slowly growing iterated Laplacians function\label{def:gaussian_growth}]
    For constants $M,a>0$, let $\funcspace{a}$ denote the set of infinitely continuously differentiable functions $f$ from $\RR^p$ to $\RR$ such that for any $x\in\RR^p$
    \begin{align}
        |f(x)|&\le M\exp(a||x||^2),\label{eq:gaussian_growth}\\
        |\Delta^k f(x)|&\le A_x \cdot (4a)^k k!,\:\forall k\in \NN, \label{eq:laplace_iterate_growth}
    \end{align}
    for some $A_x>0$ that depends only on $x$.
\end{definition}
Note that \eqref{eq:gaussian_growth} requires $f$ to grow slower than the exponential of a quadratic function, which is a fairly mild condition. The condition \eqref{eq:laplace_iterate_growth} is met, for example, for finite linear combinations of exponentials, sines, cosines, polynomials, and band-limited functions. According to a known result in the study of the heat equation \citep{weierstrass,heat_equation_1}, the Weierstrass transform of functions in $\funcspace{a}$ admits the following series expression.
\begin{restatable}[Series expression of $\WW_{\sigma^2}$]{theorem}{ExprWeier}
    \label{thm:expressions_weierstrass}
    Let $f\in \funcspace{a}$. Then, for any $\sigma^2<1/2a$, the generalized Weierstrass transform $\WW_{\sigma^2}[f]$ admits the following expression
    \begin{equation*}
        \WW_{\sigma^2}[f] = \sum_{k=0}^{\infty} \frac{\sigma^{2k}}{2^kk!} \Delta^k[f]\eqsp.
    \end{equation*}
\end{restatable}
\begin{proof}[Sketch of proof.]
    The proof is given in Appendix~\ref{appendix:weierstrass_expression}. First, using the analyticity of the heat equation solution, we remark that $\WW_{\sigma^2}[f]$ is such an analytic solution. With the parameterization $\sigma^2=2t$, we use the Taylor expression of $t\mapsto\WW_{2t}[f]$ around a positive $t_0>0$. We then take the limit as $t_0$ goes to zero to obtain a formal expression of $\WW_{2t}[f]$. Given that $f$ is in $\funcspace{a}$, the resulting series converges and we can identify it to $\WW_{\sigma^2}[f]$.
\end{proof}

\begin{remark}
    \label{remark:exp_is_licit}
    The condition $\sigma^2<1/2a$ is~not~a~strong limitation since, even for fast-increasing functions, $a$ can be very small. For example, the exponential loss $\smash{f(x) = \exp(-\theta^\top x y})$ is in $\smash{\Phi_{M_a,a}(\RR^p)}$, with $\smash{M_a=\exp( {\norm{\Theta}^2}/{4a})}$, for any $\smash{(\theta,y)\in\Theta \times \sY}$ and any arbitrary small $a>0$ (see Appendix~\ref{app:weierstrass_properties} for more details).
\end{remark}

\subsection{Bernoulli Transform: a Tool for Binary Label Noise}
\label{subsec:bernoulli}
We also define the analogous transform associated with the Randomized Response (RR) for an arbitrary real-valued function of binary inputs. The same transform is studied in \citet{noisy_labels}. For clarity, we call it the \emph{Bernoulli} transform, referencing the random draw of a Bernoulli variable in the RR mechanism.
\begin{definition}[Bernoulli transform]
    \label{def:RR_transform}
    Let $g:\{-1,1\}\to \RR$, we define for any $\epsilon>0$ and any $y\in\{-1,1\}$,
    \begin{align*}
        \BB_\epsilon[g](y) &=\EE_{\sB_\epsilon} \left[g(\sB_\epsilon(y))\right]\\
        &= \sigmoid(\epsilon)g(y) + (1-\sigmoid(\epsilon))g(-y).
    \end{align*}
\end{definition}
The Bernoulli transform is the expected value of a function of a binary data point $y$ under the Randomized Response mechanism with a given privacy budget $\epsilon$.

\subsection{Combining Weierstrass and Bernoulli Transforms}
When computing a function $h:\sX\times\sY\to \RR$ of continuous and binary variables, we can compose the transforms defined for both the Gaussian and RR mechanisms as follows:
\begin{align}
    \label{eq:composite_transform}
    \TT_{\epsilon,\delta}[h](x,y) &= \EE_{(\tilde x, \tilde y)}\left[h\left(\tilde x, \tilde y\right)\right]\\
    \nonumber
    &=\BB_{\epsilon_y}\left[z\mapsto \WW_{\sigma^2}[h(\cdot,z)](x) \right](y).
\end{align}
This transform accounts for the simultaneous effect of Gaussian noise on continuous variables and sign flipping on labels.
It will allow to analyze the combined effects of Gaussian and Randomized Response mechanisms on the population risk in binary classification.

\section{Bias in Risk Minimization}
\label{sec:bias_evidences}
We now turn to the learning problem and study how task-agnostic LDP affects the outcome of risk minimization.
In particular, we focus on the bias incurred by LDP when the goal is to learn the minimizer of the true population risk directly.
This is essential as such bias is intrinsic to the problem at hand, and cannot be compensated for by increasing the number of records used for training.
Depending on the settings, LDP noise may or may not change the population risk minimizer. 
When it does not, the original population risk minimizer can be recovered, provided that enough samples are available. 
In some other problems, however, the injected noise modifies the expected loss, resulting in a biased minimizer that might be far from the original model. When this occurs, increasing the sample size is not sufficient to eliminate the discrepancy between the learned solution and the true population risk minimizer.

In this section, we study the population risk obtained when losses are evaluated on task-agnostic LDP releases generated by the Gaussian and RR mechanisms. By expressing the expected noisy loss as the composition of the Weierstrass and Bernoulli transforms, we explicitly characterize when and how these mechanisms shift the population risk.

\paragraph{Risk Minimization.}
Let $\Theta \subset \RR^k$ be a bounded convex set of model parameters. Let $\ell: \Theta\times\RR^p\times\sY \rightarrow \RR$ be a real-valued function defined for any tuple $(\theta,x,y)\in \Theta\times\sX\times\sY$ as the loss incurred when predicting an example with features $x$ using a model $\theta$, given that the true label is $y$. We can then evaluate the quality of any model $\theta \in \Theta$ with the population risk $\risk$, defined as follows:
\begin{equation*}
    \risk(\theta)= \EE_{(x,y)\sim\sD}[\ell(\theta,x,y)].
    \label{eq:losses}
\end{equation*}
For the remainder of the paper, we consider loss functions satisfying the following regularity assumption.
\begin{assumption}[Loss regularity]
    \label{ass:weierstrass_ok_assumption}
    The functions $x\mapsto\ell(\theta,x,y)$  and $x\mapsto \partial_{\theta_j} \ell(\theta,x,y)$ for $j\in\{1,\dots,k\}$ are in $\funcspace{a}$ (see Definition~\ref{def:gaussian_growth}) for any $(\theta,y)\in\Theta\times\sY$.
\end{assumption}

\paragraph{Bias in Noisy Risk Minimization.}
Let $\epsilon,\delta>0$, for any model $\theta\in\Theta$, we define the expected population risk when the loss is evaluated on the $(\epsilon,\delta)$-LDP release of $(x,y)$ defined in Equation~\eqref{eq:ldp_release} as follows:
\begin{align*}
    \tilde \risk(\theta) &= \EE_{(x,y)\sim\sD} \EE_{(\tilde x,\tilde y)} \left[\ell(\theta,\tilde x, \tilde y)\right].
\end{align*}
We first analyze the pointwise loss $\EE_{(\tilde x,\tilde y)} \left[\ell(\theta,\tilde x, \tilde y)\right]$ for a given pair $(x,y)$ and model $\theta$. Recalling Equation~\eqref{eq:composite_transform} with $h(\cdot,\cdot) = \ell(\theta,\cdot,\cdot)$, we have
\begin{align}
\label{eq:pointwise_bias}
\EE_{(\tilde x, \tilde y)}\left[ \ell\left(\theta,\;\tilde x,\;\tilde y\right)\right] &= \TT_{\epsilon,\delta}[\ell(\theta,\cdot,\cdot)](x,y).
\end{align} 

Developing the expression of the composed transform $\TT_{\epsilon,\delta}$ and averaging over $(x,y)\sim\sD$ yields a relation between the noisy population risk $\tilde\risk$ and the original population risk $\risk$ in the following theorem (proven in Appendix~\ref{appendix:bias_char}).
\begin{restatable}[Bias induced by the Gaussian and Randomized Response mechanisms in binary classification]{theorem}{RegEffect}
    \label{thm:generalized_bishop}
    Let $\Delta_x$ denote the Laplacian with respect to the variable $x$ and assume that $\ell$ satisfies Assumption~\ref{ass:weierstrass_ok_assumption} with $a<1/2\sigma^2$. Recall $S(\epsilon_y)=1/(1+e^{-\epsilon_y})$. For any $\theta\in\Theta$,
    \begin{align*}
         \tilde \risk(\theta) -& \risk(\theta) = \underbrace{\left(1-\sigmoid(\epsilon_y)\right)\left(\EE_{x,y}\left[\ell(\theta,x,-y)\right] - \risk(\theta)\right)}_{\text{label noise contribution}}\\
         &+ \sigmoid(\epsilon_y)\underbrace{\sum_{k=1}^\infty \frac{\sigma^{2k}}{2^k k!}\EE_{x,y}\left[\Delta^k_x\ell(\theta,x,y)\right]}_{\text{feature noise contribution}}\\
         &+ \underbrace{\left(1-\sigmoid(\epsilon_y)\right)\sum_{k=1}^\infty \frac{\sigma^{2k}}{2^k k!}\EE_{x,y}\left[\Delta^k_x\ell(\theta,x,-y)\right]}_{\text{interactions of feature and label noise}} .
    \end{align*}
\end{restatable}
Theorem~\ref{thm:generalized_bishop} shows that Randomized Response on labels induces a mixture between the risks associated with true and corrupted labels, while Gaussian feature noise induces a systematic smoothing of the loss through iterated Laplacians. 
For non-private labels $(\epsilon_y\to\infty)$, our result admits the work of \citet{bishop}, proposed in the more restrictive low noise regime, as a particular case. Indeed, the loss functions they consider admit a second-order Taylor approximation with respect to the features and they derive an approximate expression of the expected risk on noisy data that matches exactly ours truncated at $k=1$. For non-private features $(\epsilon_x\to\infty)$, we recover the corruption of labels in \citet{noisy_labels}.

For some specific loss functions, the bias term in Theorem~\ref{thm:generalized_bishop} has a closed-form expression. For instance, if the derivatives of $\ell$ with respect to the features vanish after a certain order, then $\tilde\risk$ reduces to a finite sum. Even for infinitely differentiable loss functions $\ell$ with non-zero derivatives, we can sometimes derive a closed-form expression of the bias. This is, for example, the case for the exponential loss.

\begin{example}[Exponential loss]
    Consider $\ell(\theta,x,y)=\exp(-\theta^\top xy)$, we have for any $\theta\in\Theta$,
    $$ \tilde \risk(\theta) = e^{\sigma^2\|\theta\|^2/2}\left(\sigmoid(\epsilon_y)\risk(\theta)+(1-\sigmoid(\epsilon_y))\risk(-\theta)\right).$$
    Define $\tilde \theta^*\in \arg\min_\theta \tilde\risk(\theta)$ and $\theta^*\in\arg\min_\theta\risk(\theta)$, applying the logarithm preserves the minimum so $\tilde \theta^*$ also minimizes
    \begin{align*}
        \log\left( \sigmoid(\epsilon_y)\risk\left( \theta\right)+(1-\sigmoid(\epsilon_y))\risk\left(-\theta\right)\right)+\frac{\sigma^2}{2}\norm{\theta}^2.
    \end{align*}
    The term in $\|\theta\|^2$ can be seen as further regularization induced by the feature noise, while the term $\risk(-\theta)$ promotes predicting the wrong label in some cases due to label contamination. Both of these effects steer the optimal solution away from $\theta^*$, creating a gap between $\theta^*$ and $\tilde \theta^*$. We exhibit this gap empirically in Figure~\ref{fig:exp_comp}~and~\ref{fig:exp_comp_folktables} in Section~\ref{sec:experiments}.
\end{example}

Note that, considering $\sD$ as an empirical distribution over a dataset of $n$ examples leads to the same bias characterization in empirical risk minimization. Having characterized the population risk bias induced when learning from $(\epsilon,\delta)$-LDP published examples from Gaussian and RR mechanisms under the lens of the Weierstrass and Bernoulli transforms, we now turn to the natural next step of correcting it.

\section{Bias Correction}
\label{sec:bias_correction}

In this section, we leverage the framework of privacy seen as a transform introduced in Section~\ref{sec:privacy_transform} to design a practical method that corrects the bias we identified in the previous section.
To this end, we start by defining the inverse of the aforementioned transforms.
\begin{restatable}[Inverse of $\BB_\epsilon$ and $\WW_{\sigma^2}$]{theorem}{InvWeierRR}
    \label{thm:inv_RR_weier} %
    Define $\tilde \sigmoid(\epsilon) = 1/(1-e^{-\epsilon})$. Let $g:\sY \to \RR$ and $\epsilon>0$, for any $\tilde y\in\sY$,
    \begin{enumerate}
        \item [(i)]$ \BB^{-1}_\epsilon[g](\tilde y) = \tilde \sigmoid(\epsilon)g(\tilde y) +\big(1-\tilde\sigmoid(\epsilon)\big)g(-\tilde y)$.
    \end{enumerate}
    Let $f\in \funcspace{a}$, for any $\sigma^2<1/4a$ and $\tilde x\in\RR^p$
    \begin{enumerate}
        \item [(ii)]$ \WW^{-1}_{\sigma^2}[f](\tilde x) =\sum_{k=0}^\infty \frac{(-1)^k\sigma^{2k}}{2^kk!} \Delta^k[f](\tilde x).$
    \end{enumerate}
\end{restatable}
Proofs of the two inverse transforms can be found in Appendix~\ref{appendix:inverse_transforms}. Remark that $\sigma^2$ is lower than $1/4a$ instead of $1/2a$. That is because to prove that $\WW_{\sigma^2}^{-1}$ is the inverse of $\WW_{\sigma^2}$, we apply a composition of their two series expression, resulting in a stronger constraint on $\sigma^2$. Computing $\BB^{-1}[g]$ requires two evaluations of $g$ and matches the noisy label correction in \citet[Theorem 5]{noisy_labels}. $\WW^{-1}[f]$ can be computed by deriving a closed-form expression (see Section~\ref{subsec:gen_lin_loss}) or approximated by truncating the sum (see Appendix~\ref{app:no_closed_form}). 

The invertibility of these transforms plays a fundamental role in our study: when computing the inverse transform on noisy data, we obtain an unbiased estimate of the original function.
Indeed, taking the expectation of $\WW^{-1}_{\sigma^2}[f](\tilde x)$ (resp. $\BB^{-1}_\epsilon[g](\tilde y)$) amounts to computing the Weierstrass (resp. Bernoulli) transform, which recovers the function on the original data record $x$ (resp. $y$).
Next, we will show that these two transforms can be applied and inverted sequentially, which will allow us to build novel, unbiased estimators of gradients of binary classification losses.

\paragraph{Composition of $\BB^{-1}$ and $\WW^{-1}$.} For any function $h:\RR^p\times\sY\to \RR$, the inverse of $\TT_{\epsilon,\delta}$ is
\begin{align}
    \label{eq:inverse_composite_transform}
    \nonumber
    \TT_{\epsilon,\delta}^{-1}[h](\tilde x,\tilde y) =\BB_{\epsilon_y}^{-1}\left[ z\mapsto\WW_{\sigma^2}^{-1}[h(\cdot,z)](\tilde x) \right](\tilde y).
\end{align}

It provides a basis for the definition of the following unbiased loss estimator from $(\epsilon,\delta)$-LDP releases via Gaussian and Randomized Response mechanisms, we call the Inverse Weierstrass Private (IWP) loss estimator:
\begin{equation}
    \label{eq:loss_estim}
    \tilde \ell_{\epsilon,\delta}(\theta,\tilde x,\tilde y) = \TT_{\epsilon,\delta}^{-1}[\ell(\theta,\cdot,\cdot)](\tilde x,\tilde y).
\end{equation}
We differentiate it to obtain the IWP gradient estimator
\begin{equation}
    \label{eq:grad_estim}
    \nabla_\theta\tilde \ell_{\epsilon,\delta}(\theta,\tilde x,\tilde y) = \TT_{\epsilon,\delta}^{-1}[\nabla_\theta\ell(\theta,\cdot,\cdot)](\tilde x,\tilde y),
\end{equation}
with the convention that $\TT$ and $\TT^{-1}$ act component-wise on vector-valued functions such as the gradient. We show in Appendix~\ref{appendix:IWP_grad_swap} that the IWP gradient estimator is indeed the gradient of the IWP loss estimator.

The following theorem, proven in Appendix~\ref{app:unbias_IWP}, states the unbiasedness guarantees of both $\tilde\ell_{\epsilon,\delta}$ and $\nabla_\theta\tilde\ell_{\epsilon,\delta}$.
\begin{restatable}[Unbiasedness of IWP loss and gradient estimators]{theorem}{UnbiasedGrad}
\label{thm:unbiased}
    Assume $\ell$ satisfies Assumption~\ref{ass:weierstrass_ok_assumption} with $a<1/4\sigma^2$. Let $\epsilon,\delta>0$, for any pair $(x,y)\in\sX\times\sY$ and $\theta\in\Theta$, define $(\tilde x,\tilde y)$ as a $(\epsilon,\delta)$-LDP release defined in Equation~\eqref{eq:ldp_release}, the IWP loss estimator defined in Equation~\eqref{eq:loss_estim} satisfies:
    \begin{itemize}
        \item [(i)] $\EE_{(\tilde x,\tilde y)} \left[\tilde \ell_{\epsilon,\delta}(\theta,\tilde x,\tilde y)\right] = \ell(\theta,x,y)$,
        \item [(ii)] $\EE_{(\tilde x,\tilde y)}\left[ \nabla_\theta \tilde \ell_{\epsilon,\delta}(\theta,\tilde x,\tilde y)\right] = \nabla_\theta\ell(\theta,x,y)$.
    \end{itemize}
\end{restatable}
Using the IWP gradient estimator, we introduce IWP-SGD in Algorithm~\ref{alg:IWP-SGD}. It is a single-pass projected SGD over a dataset $\tilde D_n=\{(\tilde x_i,\tilde y_i)\}_{i=1}^n$ consisting of $(\epsilon,\delta)$ releases of samples $(x_i,y_i)$ drawn i.i.d. from $\sD$. It relies on the update $\theta_t = \Pi_{\Theta}\left(\theta_{t-1} +\gamma\nabla_\theta \tilde \ell_{\epsilon,\delta}(\theta_t,\tilde x_t,\tilde y_t)\right)$ for each $t$ in $\{1,\dots,n\}$, with $\gamma>0$, $\theta_0\in\Theta$ and $\Pi_\Theta$ the projection on $\Theta$.
\begin{algorithm}[tb]
    \caption{Inverse Weierstrass Private SGD (IWP-SGD)}
    \label{alg:IWP-SGD}
    \begin{algorithmic}
    \STATE {\bfseries Input:} Dataset $\tilde D_n = \{(\tilde x_i,\tilde y_i)\}_{i=1}^n$ of $(\epsilon,\delta)$-LDP released data $(x_i,y_i)\sim\sD$ according to the mechanism of Equation~\eqref{eq:ldp_release}. Initial model $\theta_0\in\Theta$ and step size $\gamma>0$. Loss function $\ell:\Theta\times\sX\times\sY \to \RR$ and projection $\Pi_\Theta$ on the bounded convex set $\Theta$.
    \FOR {$t\in\{1,\dots,n\}$}
        \STATE Compute the IWP gradient $\nabla_\theta\tilde \ell_{\epsilon,\delta}(\theta,\tilde x_t,\tilde y_t)$ (Equation~\eqref{eq:grad_estim}).
        \STATE Update $\theta_{t} = \Pi_\Theta(\theta_{t-1} -\gamma \nabla_\theta\tilde \ell_{\epsilon,\delta}(\theta,\tilde x_t,\tilde y_t))$.
    \ENDFOR
    \STATE {\bfseries Output:} Model after the last update $\theta_n$.
    \end{algorithmic}
\end{algorithm}
Following standard convergence analyses of SGD with unbiased stochastic gradient, we bound the variance of the IWP gradient estimator in the following theorem.
\begin{restatable}[Variance of the IWP gradient estimator]{theorem}{OurGradientVar}
    \label{thm:our_gradient_variance}
    Let $\epsilon,\delta>0$, an original feature-label pair $x,y\in\sX\times\sY$ and its corresponding $(\epsilon,\delta)$-LDP release $(\tilde x,\tilde y)$ defined in Equation~\eqref{eq:ldp_release}. Let $\ell$ satisfy Assumption~\ref{ass:weierstrass_ok_assumption} with $a<1/4\sigma^2$. Given that $\sX$ and $\Theta$ are bounded sets, denoting
    \begin{align}
        \label{eq:C}
        C & =\sup_{(\theta,x,y)\in\Theta\times\sX\times\sY,s<\sigma^2}\max\Big\{ \norm{\nabla_\theta \ell(\theta,x,y)},
        \\[-0.5em]
        \nonumber
       &  \qquad\qquad\qquad\qquad \quad \norm{ \WW_{s}^{-1}\left[\nabla_\theta \nabla_x \ell(\theta,\cdot,y)\right](x)}\Big\},
    \end{align}
    the variance of the IWP gradient estimator admits the following upper bound
    \begin{align}
    \label{eq:IWP_grad_var_bound}
    \nonumber   
    &\EE\norm{\nabla_\theta \tilde\ell_{\epsilon,\delta}(\theta,\tilde x,\tilde y) - \nabla_\theta \ell(\theta,x,y)}^2
     \\
     &\quad \le C^2\left(\sigma^2+4\tilde\sigmoid(\epsilon_y)\left(\tilde\sigmoid(\epsilon_y)-1\right)(1+\sigma^2)\right)\eqsp.
    \end{align}
\end{restatable}
\begin{proof}[Sketch of proof.]
    The proof is given in Appendix~\ref{appendix:variance_IWP}. We prove the result using a combination of the marginal variances from both $\tilde x$ and $\tilde y$ together with total variance law. It yields an exact variance expression that we can bound using Equation~\eqref{eq:C}, relying on the increasing property of $\WW_{\sigma^2}$ proven in Appendix~\ref{app:weierstrass_properties}.
\end{proof}
The resulting variance bound in Theorem~\ref{thm:our_gradient_variance} distinguishes the three contributions of $\WW^{-1}$, $\BB^{-1}$, and the compound contribution of both in the resulting variance. It shows dependency in the feature noise variance and labels privacy budget $\epsilon_y$. A dependency on the loss can be further quantified via the constant $C$. Indeed, the growth rate of $C$ with $\norm{\sX}$ and $\norm{\Theta}$ is affected by the regularity of the loss. For example, for the quadratic loss with linear models, $C$ is of order $\sO(p\norm{\sX}\norm{\Theta})$. Similarly, for the exponential loss with linear models, $C$ is of order $\sO\big(\exp\big(\epsilon_x^2\big)\big(p+\norm{\sX}\norm{\Theta}+ \sigma^2\norm{\Theta}^2\big)\big)$ (see proofs in Appendix~\ref{app:C_bound_GLM}).
Now that we have bounded the variance of the IWP gradient estimator, we can analyze the IWP-SGD convergence guarantees using known results on SGD with unbiased stochastic gradients.

\subsection{Convergence guarantees of IWP-SGD}
We give convergence guarantees of IWP-SGD under the strong convexity and smoothness assumptions (see Appendix~\ref{app:defs} for proper definitions).
\begin{assumption}[Strong convexity and smoothness]\label{ass:strong_convexity_smoothness}
$\risk$ is $\mu$-strongly convex and $\sK$-smooth, with $\sK>0$ and $\mu>0$.
\end{assumption}

These assumptions are common in the analysis of SGD's convergence \citep{moulines2011non,stich_unified} and are used solely for this purpose in this paper.
\begin{restatable}[Convergence guarantees of IWP-SGD]{theorem}{IWPSGDConv}
    \label{thm:conv_IWPSGD}
    Let $\ell$ satisfy Assumption~\ref{ass:weierstrass_ok_assumption} and be such that $\risk$ satisfies Assumption~\ref{ass:strong_convexity_smoothness}. Let the privacy budget be $\epsilon=\epsilon_x+\epsilon
    _y,\;\delta>0$ such that $\sigma^2<1/4a$. Denote $\theta^*=\arg\min_\theta\risk(\theta)$. Assume $\sX$ and $\Theta$ are bounded convex sets and let $C$ be as defined in Equation~\eqref{eq:C}.  For any $n\in \NN$ the number of training samples, initial model $\theta_0\in\Theta$ and step-size $\gamma\le\frac{1}{2\sK}$, Algorithm~\ref{alg:IWP-SGD} is $(\epsilon,\delta)$-LDP and its output $\theta_n$ satisfies
    \begin{align*}
        \EE\|\theta_n - &\theta^*\|^2   \le (1-\gamma\mu)^n\|\theta_0-\theta^*\|^2\\
        &+ \sO\left(\frac{\gamma C^2}{\mu}\tilde\sigmoid(\epsilon_y)\left(\tilde\sigmoid(\epsilon_y)-1\right) \frac{\log(1.25/\delta)}{\epsilon_x^2}\right).
    \end{align*}
    In addition, for an appropriate step size $\gamma=\sO(\log(n)/n)$,
    \begin{align*}
        \EE\|\theta_n - &\theta^*\|^2 \le \tilde\sO\left(\|\theta_0-\theta^*\|^2 \exp\left(-\frac{\mu n}{2\sK}\right)\right)\\
            &+ \tilde\sO\left(\frac{C^2}{\mu^2n}\tilde\sigmoid(\epsilon_y)\left(\tilde\sigmoid(\epsilon_y)-1\right)\frac{\log(1.25/\delta)}{\epsilon_x^2}\right),
    \end{align*}
    where $\tilde\sO$ hides logarithmic terms in $n$.
\end{restatable}
\begin{proof}[Sketch of proof.]
    The proof, given in Appendix~\ref{appendix:convergence_IWP_SGD}, is a direct application of \citet{stich_unified} with our unbiased noisy gradient estimator and its variance bound given in Theorem~\ref{thm:our_gradient_variance}. We also use their derivation for the appropriate step-size $\gamma=\sO(\log(n)/n)$.
\end{proof}

Theorem~\ref{thm:conv_IWPSGD} shows that the last iterate of IWP-SGD is converging to the population risk minimizer when the number of examples grows to infinity. The convergence rate of IWP-SGD matches the dependency on $n$ of locally private SGD in \citet[Theorem~20, item 4]{noisy_oracle}. However, depending on the choice of loss, we may, through $C$, suffer from a higher dependency on the dimensions of $\sX$ and $\Theta$ than \citet{noisy_oracle}, which, for 1-Lipschitz, smooth and strongly-convex losses, show a linear dependency on the dimension of the model space $\Theta$. In practice, one can obtain estimators with lower variance by considering batches of examples at each iteration. 
However, it does not change the fact that each example can only be used once in the optimization process, and, thus, has limited impact on the total number of records required to approximate the solution.

The general framework of IWP-SGD can be instantiated in various settings. In particular, when applied to generalized linear models, the method admits a more tractable form, which we develop in the following subsection.

\subsection{Application to Generalized Linear Models}
\label{subsec:gen_lin_loss}
Let $\ell(\theta,x,y)=f\left(\theta^\top xy\right)$ be a loss that satisfies the generalized linear loss assumption.
In this case, by Equation~\eqref{eq:loss_estim}, the IWP loss estimators becomes
\begin{align*}
        \tilde\ell_{\epsilon,\delta}(\theta,\;\tilde x,\;\tilde y) 
        &=\tilde \sigmoid(\epsilon_y)\WW_{\sigma^2 \norm{\theta}^2}^{-1}[f]\left(\theta^\top \tilde x \tilde y\right)\\
        &+\left(1-\tilde\sigmoid(\epsilon_y)\right)\WW_{\sigma^2 \norm{\theta}^2}^{-1}[f]\left(-\theta^\top \tilde x \tilde y\right),
\end{align*}
and the IWP gradient estimator can be expressed as
\begin{align*}
        \nabla_\theta\tilde\ell_{\epsilon,\delta}(\theta,&\tilde x, \tilde y) = \tilde \sigmoid(\epsilon_y)\nabla_\theta\WW_{\sigma^2 \norm{\theta}^2}^{-1}[f]\left(\theta^\top \tilde x \tilde y\right)\\
        &+\left(1-\tilde\sigmoid(\epsilon_y)\right)\nabla_\theta\WW_{\sigma^2 \norm{\theta}^2}^{-1}[f]\left(-\theta^\top \tilde x \tilde y\right).
\end{align*}
This form involves the Weierstrass transform of the scalar valued function $f$ instead of $\ell(\theta,\cdot,y)$, which simplifies the IWP gradient estimator expression. It further yields a closed-form expression for some losses $f$ such as the quadratic or exponential losses.
\begin{example}[Quadratic loss]
    Let $f\left(\theta^\top xy\right) =\frac{1}{2}(\theta^\top xy-1)^2$ for any $(\theta,x,y)\in\Theta\times\sX\times\sY$, the IWP loss is
    \begin{align*}
        \tilde\ell_{\epsilon,\delta}(\theta,\tilde x, \tilde y) &= \tilde \sigmoid(\epsilon_y)f(\theta^\top \tilde x \tilde y)\\
        &+\left(1-\tilde\sigmoid(\epsilon_y)\right)f(-\theta^\top \tilde x \tilde y)- \frac{\sigma^2}{2}\|\theta\|^2,
    \end{align*}
    with the corresponding IWP gradient estimator
    \begin{align*}
        \nabla_\theta &\tilde\ell_{\epsilon,\delta}(\theta,\tilde x, \tilde y) = \tilde \sigmoid(\epsilon_y)\nabla_\theta f(\theta^\top \tilde x \tilde y) \\
        &+\left(1-\tilde\sigmoid(\epsilon_y)\right)\nabla_\theta f(-\theta^\top \tilde x \tilde y) - \sigma^2\theta.
    \end{align*}
\end{example}
In that case, the IWP gradient estimator is acting like $\ell_2$ regularization with the negative constant $-\sigma^2$ and requires two gradient evaluations. 
\begin{example}[Exponential loss]
    Let $f\left(\theta^\top xy\right) =e^{-\theta^\top xy}$ for any $(\theta,x,y)\in\Theta\times\sX\times\sY$, the IWP loss is
    \begin{align*}
        \tilde\ell_{\epsilon,\delta}(\theta,\tilde x, \tilde y) &= e^{-\sigma^2\|\theta\|^2/2}\tilde \sigmoid(\epsilon_y)f(\theta^\top \tilde x \tilde y)\\
        &+e^{-\sigma^2\|\theta\|^2/2}\left(1-\tilde\sigmoid(\epsilon_y)\right)f(-\theta^\top \tilde x \tilde y),
    \end{align*}
    with the corresponding IWP gradient estimator
    \begin{align*}
         \nabla_\theta &\tilde\ell_{\epsilon,\delta}(\theta,\tilde x, \tilde y) = e^{-\sigma^2\|\theta\|^2/2} \tilde \sigmoid(\epsilon_y)\nabla_\theta f(\theta^\top \tilde x \tilde y)\\
         &+e^{-\sigma^2\|\theta\|^2/2}\left(1-\tilde\sigmoid(\epsilon_y)\right)\nabla_\theta f(-\theta^\top \tilde x \tilde y)\\
         &- \sigma^2 e^{-\sigma^2\|\theta\|^2/2}\tilde \sigmoid(\epsilon_y) f(\theta^\top \tilde x \tilde y)\theta\\
        &- \sigma^2 e^{-\sigma^2\|\theta\|^2/2}\left(1-\tilde\sigmoid(\epsilon_y)\right) f(-\theta^\top \tilde x \tilde y)\theta.
    \end{align*}
\end{example}
There, the IWP gradient estimator requires two gradient evaluations and two loss evaluations. As in the previous example, there is a similar term to $\ell_2$ regularization with a negative constant of order $-\sigma^2\exp(\epsilon_y+\epsilon_x^2)$.

\section{Experiments}
\label{sec:experiments}

This section empirically validates the claimed guarantees of convergence and absence of bias of Section~\ref{sec:bias_correction}. We compare three different SGD approaches: (i) SGD - real data: on the original dataset without noise, (ii) SGD - noisy data: on the $(\epsilon,\delta)$-LDP released dataset via the mechanism of Equation~\eqref{eq:ldp_release} and (iii) IWP-SGD: Algorithm~\ref{alg:IWP-SGD} on the same $(\epsilon,\delta)$-LDP released dataset. All the experiments are binary classification problems with linear models minimizing the exponential loss with $\ell_2$ regularization. It is thus a strongly convex problem having a unique minimizer. We average 100 random draws of data and noise of the LDP mechanism for the synthetic data, and noise only for the real data. Across all three methods, we report the empirical risk on a test dataset defined as $\loss(\theta)=\frac{1}{n}\sum_i \ell(\theta,x_i,y_i)$. Additional details on the experimental setup are provided in Appendix~\ref{appendix:experiments}.

\paragraph{Synthetic Data.} We study two synthetic binary classification problems in dimension $p=2$ and $p=10$ generated with the \verb|make_classification| routine of \verb|scikit-learn| \citep{pedregosa2011scikit} having features within $[-1,1]^p$. We conduct the experiments on $n=10^6$ samples for two privacy guarantees : $(2,10^{-5})$-LDP for $p=2$ and $(5,10^{-5})$-LDP for $p=10$. 

\begin{figure}
    \centering
        \includegraphics[trim=0 0 0 0, clip,width=0.235\textwidth]{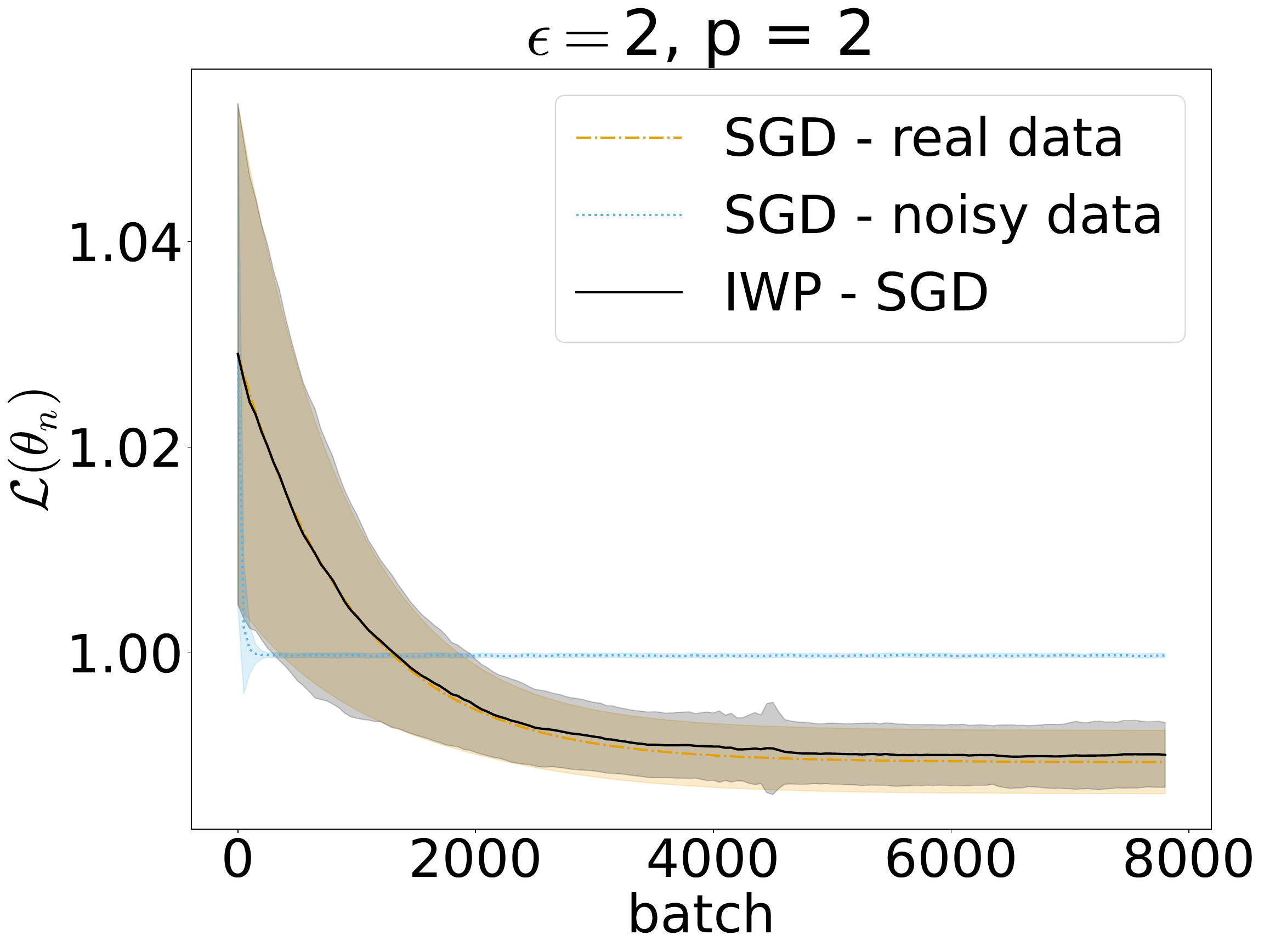}
        \includegraphics[trim=0 0 0 0, clip,width=0.235\textwidth]{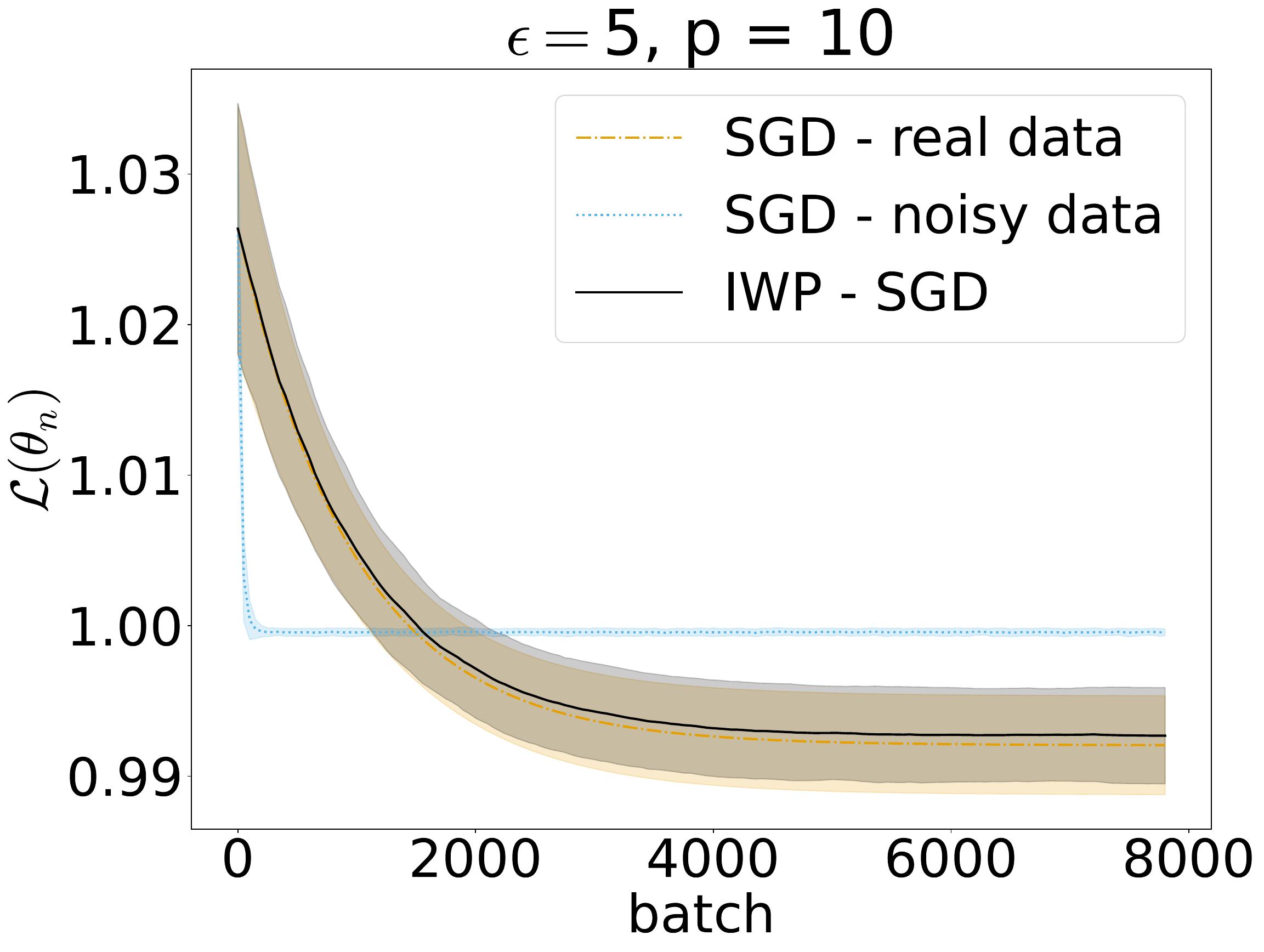}
    \captionof{figure}{Comparison of SGD convergence of the exponential loss under $(2,10^{-5})$-LDP for the 2-dimensional synthetic data and $(5,10^{-5})$-LDP for the 10-dimensional synthetic data.}
    \label{fig:exp_comp}
\end{figure}

The constant loss over the batches in Figure~\ref{fig:exp_comp} shows models fitted via SGD on noisy data converge to a different model than models fitted via SGD on the real data. 
Whereas the loss of models fitted via IWP-SGD follow the one of the models fitted on the real data. It illustrates the absence of bias for IWP-SGD and its presence for SGD on noisy data.

\paragraph{Real Data.} We study the ACSIncome and ACSPublicCoverage problems of the Folktables dataset \citep{folktables}. ACSIncome consists of predicting whether an individual’s income is above \$50~000 and ACSPublicCoverage consists of predicting individual coverage from health insurance. For both problems, we select the three variables \emph{AGEP} (age in years), \emph{SEX} and \emph{SCHL} (educational attainment). For ACSIncome we add \emph{WKHP} (usual hours worked per week over the past year) and for ACSPublicCoverage we add \emph{PINCP} (total annual income). We employ the Gaussian mechanism, which is suitable for continuous and ordinal variables. Although suboptimal for binary variables, we also apply it to the \emph{SEX} attribute for consistency and practicality. We merge the data of the five largest states yielding datasets of respectively 668~859 rows and 883~984 rows for ACSIncome and ACSPublicCoverage. The data is then randomly split into training (80\%) and test (20\%) sets.

\begin{figure}
   \centering
    \begin{subfigure}[b]{0.5\textwidth}
        \centering
        \begin{minipage}[b]{0.48\textwidth}
        \includegraphics[trim=0 0 0 0, clip,width=0.95\textwidth]{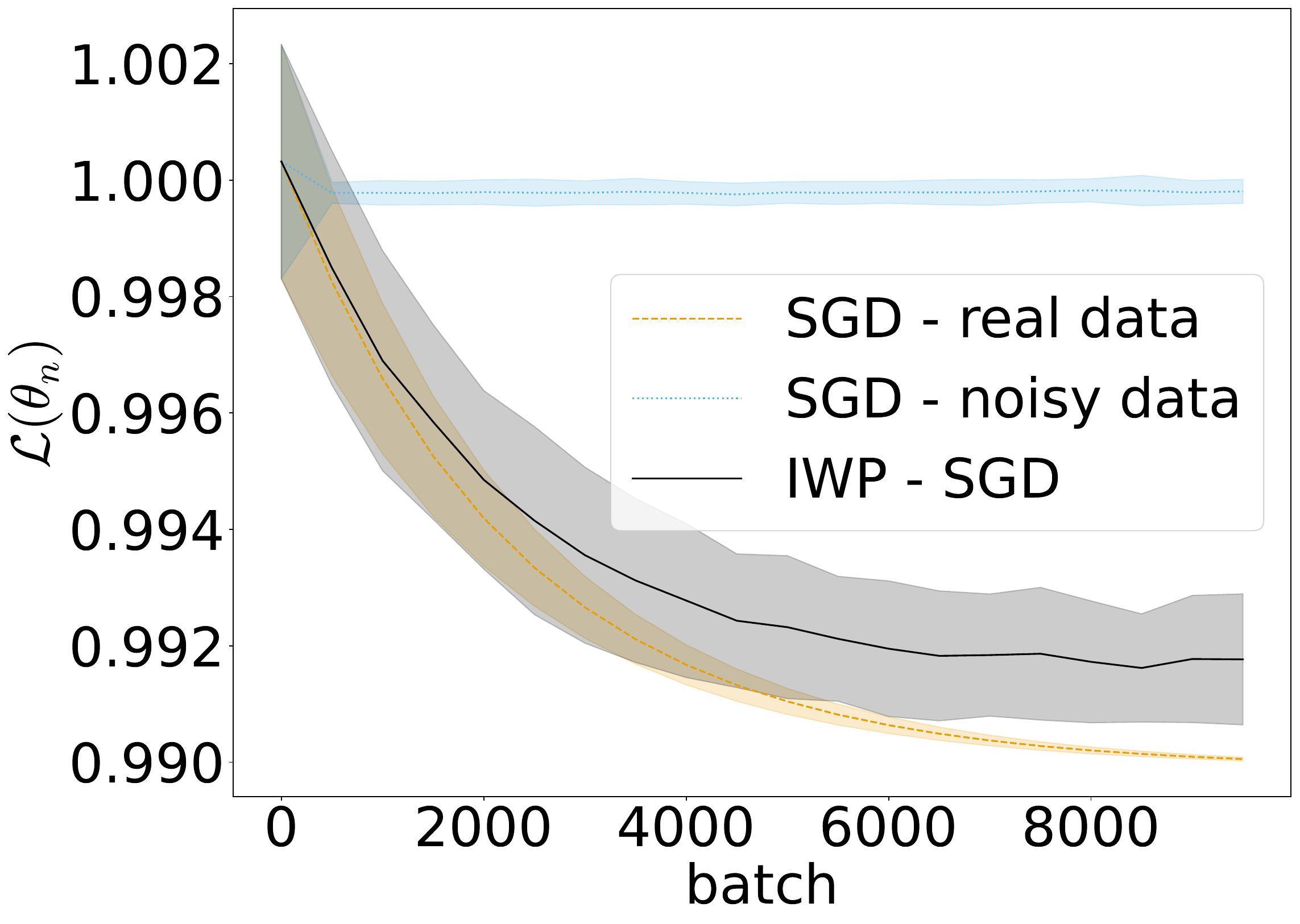}
        \end{minipage}
        \begin{minipage}[b]{0.48\textwidth}
        \includegraphics[trim=0 0 0 0, clip,width=0.95\textwidth]{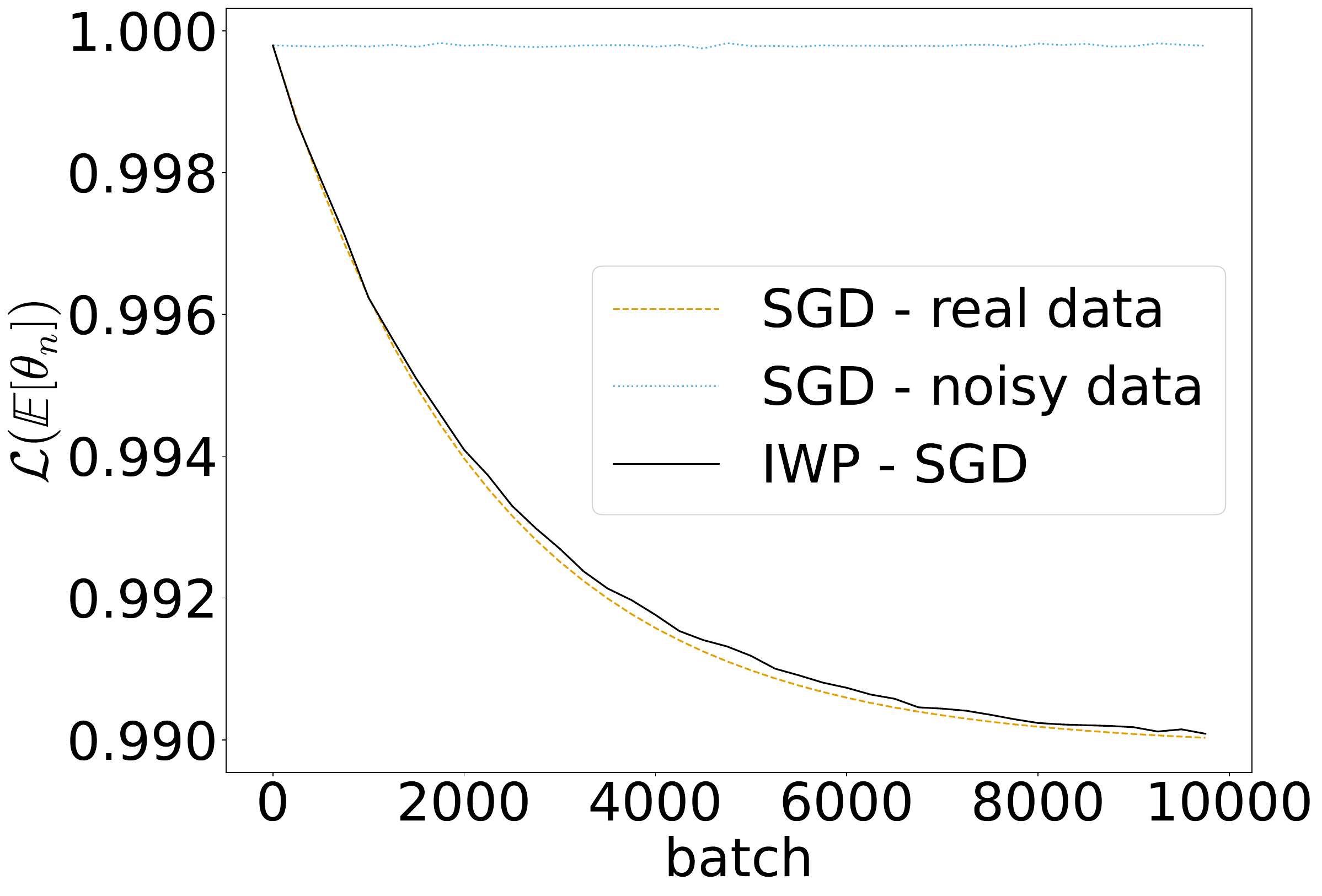}
        \end{minipage}
        \caption{ACSPublicCoverage}
        \label{fig:exp_comp_folktables_ACSPublicCoverage}
    \end{subfigure}
    \begin{subfigure}[b]{0.5\textwidth}
        \centering
        \begin{minipage}[b]{0.48\textwidth}
        \includegraphics[trim=0 0 0 0, clip,width=0.95\textwidth]{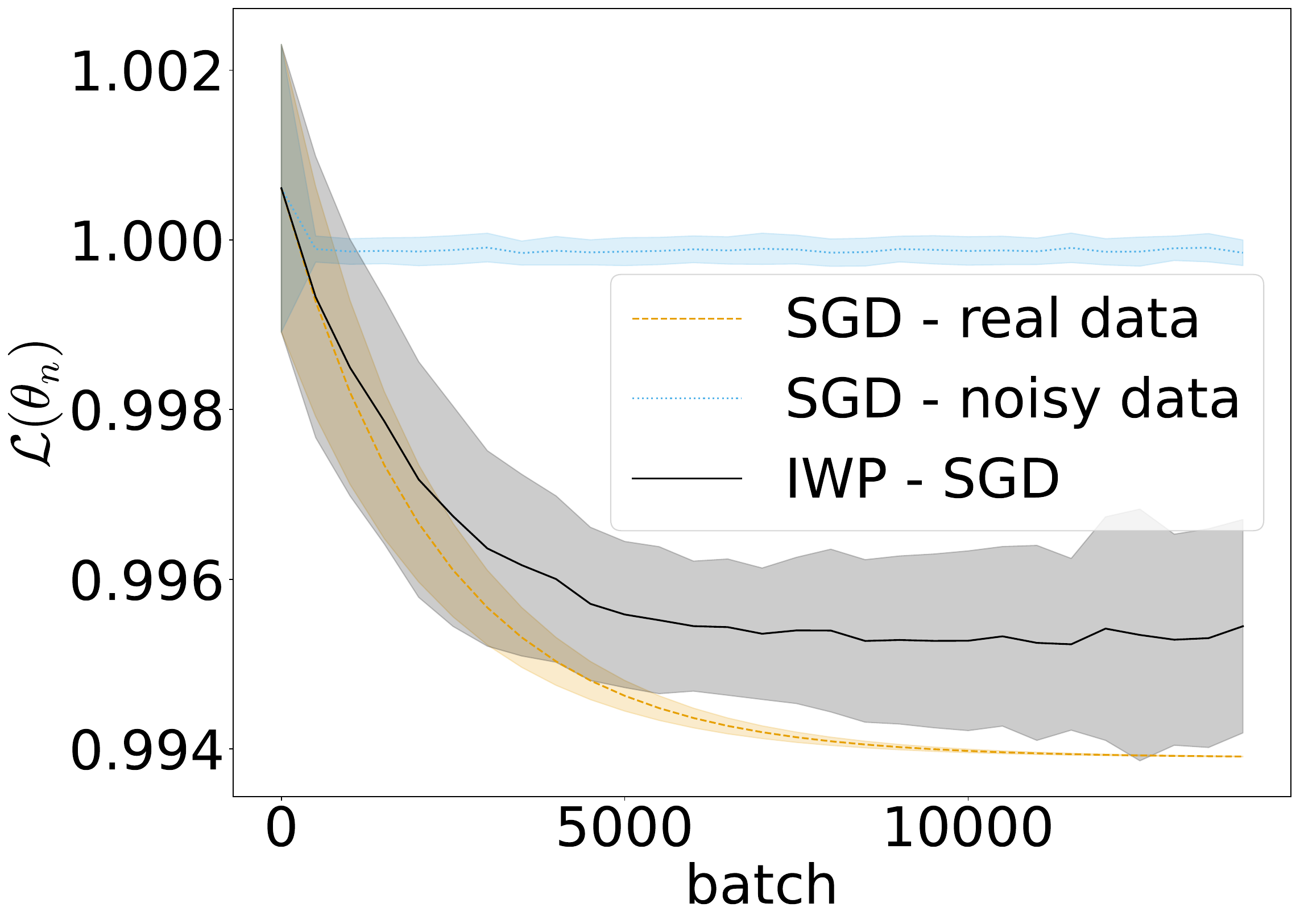}
        \end{minipage}
        \begin{minipage}[b]{0.48\textwidth}
        \includegraphics[trim=0 0 0 0, clip,width=0.95\textwidth]{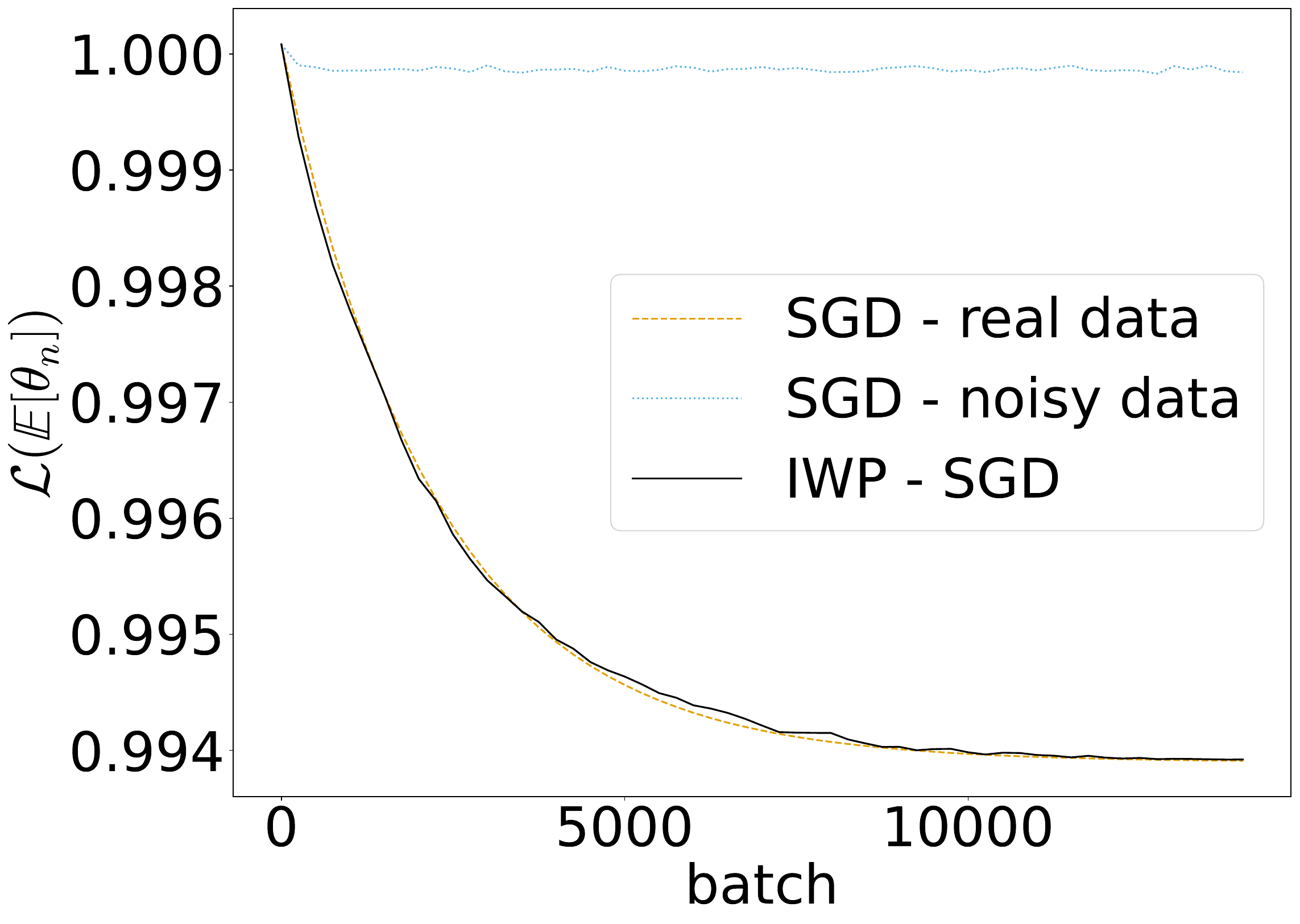}
        \end{minipage}
        \caption{ACSIncome}
        \label{fig:exp_comp_folktables_ACSIncome}
    \end{subfigure}

    \caption{Comparison of SGD convergence of the exponential loss under $(2,10^{-5})$-LDP on ACSPublicCoverage and ACSIncome. Left hand plots show the averaged loss of fitted model $\left(\theta_n\right)$ while right hand plots are showing the loss of the averaged model $\left(\EE\theta_n\right)$ over random draws.}
    \label{fig:exp_comp_folktables}
\end{figure}

Figure~\ref{fig:exp_comp_folktables} shows the average loss over fitted models $\left(\theta_n\right)$ and also the loss of the averaged model $\left(\EE\theta_n\right)$ over the random draws. It allows us to distinguish the remaining excess risk resulting from bias or variance of the IWP-SGD outputs. As with synthetic data, we remark that the constant loss of models fitted via SGD on noisy data with the number of batches illustrates the presence of bias in this setting. Whereas the IWP-SGD outputs are showing a decrease to a remaining low loss close to the one on real data. We can interpret this remaining gap as a consequence of the variance of IWP-SGD. Indeed, on the second plot, the averaged output model $\left(\EE\theta_n\right)$ for IWP-SGD is showing a null difference with the loss of the optimal model obtained through SGD.
\section{Conclusion}

In this paper, we characterized the bias that occurs when learning from an LDP-released dataset using Gaussian and Randomized Response mechanisms. Linking these mechanisms with transform operators, we derived an expression of the bias on the population risk under these LDP mechanisms. This view of privacy as a transform yielded the construction of a theoretically-grounded debiasing technique, which takes the form of a variant of SGD called IWP-SGD.

Our results show, theoretically and empirically, that the bias induced by the use of private noisy examples for SGD can be avoided at the cost of a higher variance of the noisy gradient estimator, illustrating a bias-variance tradeoff. This opens a pathway to study LDP through the lens of transform operators. Extending the framework of this paper to other locally private mechanisms presents promising avenues for future exploration.

\section*{Impact Statement}

This paper presents work that can help in the design of Machine Learning projects using locally private examples. It can help argue against the idea that directly learning from noisy examples in Differential Privacy is a problem that is too hard to be solved using practical algorithms. Future societal consequences might be the publication of locally private data for later use in fields where no public dataset exist.

\section*{Acknowledgments}
We acknowledge the support of the French National Research Agency (ANR) through the grant ANR-23-CE23-0011-01 (Project FaCTor) and ANR 22-PECY-0002 IPOP (Interdisciplinary Project on Privacy) project of the Cybersecurity PEPR.

\bibliography{icml_ref}
\bibliographystyle{icml2026}

\newpage
\appendix
\onecolumn

\section{Definitions and Theorems}
\label{app:defs}

\paragraph{Notations.}
We recall the main notations of the paper:
\begin{itemize}
    \item Data lies in $\mathcal{X} \subset \RR^{p}$, $\sY=\{-1,1\}$, parameters in $\Theta \subset \RR^k$, and $\sD$ is a joint distribution over $\sX\times \sY$.
    \item The loss function is $\ell: \Theta\times\RR^p\times\sY \rightarrow \RR$ and the associated risk $\risk(\theta)= \EE_{(x,y)\sim\sD}[\ell(\theta,x,y)]$.
    \item We denote $\norm{\cdot}$ the euclidean norm, and for any subset $\sZ\subset\RR^d$, we write $\norm{\sZ}=\sup_{z\in\sZ}\norm{z}$.
    \item For $a,b:\NN\to \RR^+$, we write $a=\sO(b)$ if there exists $C>0$ such that for all $n\in\NN$, $a(n)\leq Cb(n)$.
    \item For a complex number $z=\alpha+i\beta$ with $(\alpha,\beta)\in\RR^2$ and $i^2=-1$, we denote $\Re(z)=\alpha$ its real part, $\Im(z) =\beta$ its imaginary part, and $|z|=\sqrt{\alpha^2+\beta^2}$ its modulus.
\end{itemize}
\paragraph{Subsets of $\boldsymbol{\RR^d}$.}
We also recall some definitions on subsets of $\RR^d$ for completeness.
\begin{definition}[Convex subset of $\RR^d$]
    A subset $\sZ\subset\RR^d$ is convex if for any $z,z'\in\sZ$ and $t\in[0,1]$,
    $tz + (1-t)z' \in \sZ$.
\end{definition}
\begin{definition}[Open and closed subsets of $\RR^d$]
    A subset $\sZ\subset\RR^d$ is open if for any $z\in\sZ$, there exist $\delta>0$ such that for any $\tilde z\in\RR^d$
    such that $\norm{z-\tilde z}<\delta$,  $\tilde z \in \sZ$.
    The set
    $\sZ$ is said to be closed if $\{z\in\RR^d\:|\:z\notin\sZ\}$ is open.
\end{definition}
\begin{definition}[Compact subset of $\RR^d$]
    A subset $\sZ\subset\RR^d$ is said to be compact if it is closed and bounded ($\norm{\sZ}<\infty$).
\end{definition}
\paragraph{Regularity of functions.}
We now define classical regularity definitions for real-valued functions.
\begin{definition}[Convexity and strong convexity]
    A differentiable function $f:\RR^p\to\RR$ is $\mu$-strongly convex (where $\mu>0$) if $\forall x,y\in\RR^p$, we have $f(x)-f(y)\geq \langle \nabla f(y),x-y\rangle +\frac{\mu}{2}\|x-y\|^2$, and convex if this holds for $\mu=0$.
\end{definition}
\begin{definition}[Smoothness]
    A differentiable function $f:\RR^p\to\RR$ is $\sK$-smooth if $\forall x,x'\in\RR^p$, we have $\norm{\nabla f(x)-\nabla f(x')}\le\sK\|x-x'\|$.
\end{definition}
\begin{definition}[Laplacian]
The Laplacian of a twice differentiable function $f$ is $\Delta[f] = \sum_i \partial^2_{x_i}[f]$ and its composition $k$ times is denoted $\Delta^k[f] = (\Delta\circ\dots\circ\Delta)[f]$.
\end{definition}
\paragraph{Analytic functions.}
We define real and complex analytic functions as follows.
\begin{definition}[Real analytic functions \citep{analytic_def_2}]
    A function $f:\RR\to\RR$ is real analytic on a subset $\Omega\subset\RR$ if it is infinitely continuously differentiable on $\Omega$ and for any compact $K\subset\Omega$, there exist $A>0$ such that for any $k\in\NN^*$,
    $$\sup_{x\in K}\left| \frac{d^k f}{dx^k}(x)\right| \le A^{k+1} k!\eqsp.$$
\end{definition}
\begin{definition}[Complex analytic function \citep{complex_analysis}]
    A function $f:\CC\to\CC$ is complex analytic on a subset $K\subset\CC$ if for any $x_0\in\Omega$, it admits a convergent power series in a neighborhood of $x_0$.
\end{definition}
Two key results from complex analysis are Morera's theorem \citep[Page 122]{complex_analysis} and Cauchy's integral theorem \citep[Theorem 2, Page 109]{complex_analysis}. These will allow us to give explicit expressions of the Weierstrass transform.
\begin{theorem}[Morera's theorem]
    \label{thm:Morera}
    Let $f:\CC\to\CC$ be a complex-valued function that is continuous on $K\subset\CC$ an open set in the complex plane. If
    $$\oint_\Gamma f(x)dx = 0$$
    for every closed piecewise smooth contour $\Gamma$ in $K$, then $f$ is complex analytic.
\end{theorem}
\begin{theorem}[Cauchy's integral theorem]
    \label{thm:Cauchy}
    Let $f:\CC\to\CC$ be a complex analytic function on $K\subset\CC$ an open set in the complex plane. Then for any closed piecewise smooth contour $\Gamma$ in $K$
    $$\oint_\Gamma f(x)dx = 0.$$
\end{theorem}

In particular, a corollary of these two theorems is that for a continuous complex function $f:\CC\to\CC$, it holds that
$$ f \text{ is analytic on $K \subset \CC$} 
\textit{ if and only if }
\oint_\Gamma f(x)dx = 0,\text{ for any closed piecewise smooth contour }\Gamma\subset K\eqsp.$$

\section{The Weierstrass Transform}
\label{appendix:weierstrass_expression}

\subsection{Expression of the Weierstrass Transform -- Proof of \Cref{thm:expressions_weierstrass}}
We provide a proof of the following theorem, which gives an expression of the Weierstrass transform.
\ExprWeier*

To prove this theorem, we use the following lemma from \citet[Chapter 7, Problem 3]{PDE}. No proof is given for this problem in the initial textbook, we thus provide one in the following.
\begin{lemma}[$t\mapsto\WW_{2t}$ is analytic for continuous Gaussian growing functions]
    Let $f:\RR^p\to\RR$ a continuous function satisfying Equation~\eqref{eq:gaussian_growth} we refer to as Gaussian growth and recall: $|f(x)|\le M\exp(a\norm{x}^2)$, for any $x\in\RR^p$. Then for any $x\in\RR^p$, $t\mapsto\WW_{2t}[f](x)$ is real analytic on $]0,1/4a[$.
    \label{lemma:analytic_W}
\end{lemma}
\begin{proof}
    Let $f$ satisfying $|f(x)|\le M\exp(a\norm{x}^2)$ on $\RR^p$ and a fixed $x\in \RR^p$ throughout the proof. The overall goal is to show the analyticity of $t\mapsto\WW_{2t}[f](x)$ on the larger complex domain $\Omega = \{t\in\mathbb{C}\;s.t.\;\Re(1/4t)>a\}$ which contains the real interval $]0,1/4a[$. Indeed, the complex analyticity on a larger open set implies the real analyticity on the contained real open interval $]0,1/4a[$. The analyticity of $t\mapsto\WW_{2t}[f](x)$ on $\Omega$ is shown by verifying it is analytic on any compact subset $K\subset\Omega$. This latter objective is done using Morera's theorem (see Theorem~\ref{thm:Morera}). We must then verify two properties:
    \begin{itemize}
        \item[(i)] $t\mapsto\WW_{2t}[f](x)$ is continuous on $K$,
        \item[(ii)] $\oint_\Gamma \WW_{2t}[f](x)dt =0$ for any closed piecewise smooth contour $\Gamma$ in $K$.
    \end{itemize}
    By Morera's theorem, if these conditions are met, the function $t\mapsto\WW_{2t}[f](x)$ is analytic on $K$.
    
    \textbf{(i) Continuity of  $\boldsymbol{t\mapsto\WW_{2t}[f](x)}$ on $K$.}
    We prove this using continuity under the integral. We show that $\WW_{2t}[f](x)$ can be written as the integral of a continuous dominated function $F(t,x,y)$ with respect to its integration variable $y$. Under these conditions, $\WW_{2t}[f](x)$ is then continuous.
    
    Denote $\psi_{2t}(w) = \frac{1}{\sqrt{4\pi t}^p}\exp\big(-\frac{\norm{w}^2}{4t}\big)$, the probability density function of a centered isotropic Gaussian distribution of variance $2t\Id{p}$. We first define $F(t,x,y) = f(y)\psi_{2t}(x-y)$ for any $(t,y)\in]0,1/4a[\times\RR^p$. By definition of $\WW_{2t}$,
    \begin{align*}
        \WW_{2t}[f](x) &= \EE_{w\sim\sN(0,2t\Id{p})}\left[f(x+w)\right] = \int_{\RR^p} f(x+w)\psi_{2t}(w)dw = \int_{\RR^p} F(t,x,w)dw.
    \end{align*}
    We use the dominated convergence of $t\mapsto F(t,x,w)$ on $K$ for a fixed $x\in\RR^p$ by a function of $w$ to show that $t\mapsto\WW_{2t}[f](x)$ is continuous.
    Let $w\in \RR^p$ and $t\in K$,
    \begin{align*}
        |F(t,x,w)| &\le M\exp(a\norm{w}^2)|4\pi t|^{-p/2} \exp(-\Re(\norm{x-w}^2/4t))\\
        \cause{Denote $C_K = \sup_{t\in K}|4\pi t|^{-p/2}$.}
        &\le M C_K \exp(a\norm{w}^2 -\Re(\norm{x-w}^2/4t))
        \eqsp.
    \end{align*}
    Expanding the squared norm $\norm{x-w}^2$ gives $-\Re(\norm{x-w}^2/4t) \le \big(- \norm{x}^2 + 2\norm{x}\norm{w} - \norm{w}^2\big) \Re(1/4t)$.
    Denoting $\delta = \inf_{t\in K} \{\Re(1/4t)-a\}$, which is positive by definition of $\Omega$, we obtain
    \begin{align*}
        |F(t,x,w)| 
        &\le M C_K \exp(a\norm{w}^2 - (a+\delta)(\norm{w}^2 -2\norm{x}\norm{w}+\norm{x}^2))\\
        &= M C_K \exp(-\delta\norm{w}^2) \exp(2(a+\delta)\norm{x}\norm{w} - (a+\delta)\norm{x}^2)\\
        \cause{Denoting $D_{K,x,a} = M C_K\exp\left((a+\delta)\norm{x}^2\frac{2a+\delta}{\delta}\right)$}
        &\le D_{K,x,a} \exp(-(\delta/2)\norm{w}^2) \eqsp, %
    \end{align*}
    which is a scaled Gaussian function, and is thus integrable. Then, for the given $x\in \RR^p$,  $w\mapsto F(t,x,w)$ is uniformly dominated by an integrable function for any $t\in K$. So $t\mapsto \WW_{2t}[f]$ is continuous on $K$.
    
    \textbf{(ii) Morera's criterion.}
    Let $\Gamma$ be a closed piecewise smooth contour in $K$.
    We write $\oint_\Gamma \WW_{2t}[f](x)dt$ as the double integral $\oint_\Gamma \left(\int_{\RR^p}F(t,x,w)dw\right)dt$. 
    Since $w \mapsto | F(t, x, w) |$ is dominated by an integrable function, Fubini's theorem gives%
    \begin{align*}
        \oint_\Gamma \WW_{2t}[f](x)dt &= \oint_\Gamma \left(\int_{\RR^p}F(t,x,w)dw\right)dt
=\int_{\RR^p}\left(\oint_\Gamma F(t,x,w)dt \right)dw
      .
    \end{align*}   
    Remark that, for any $w\in\RR^p$, $t\mapsto F(t,x,w)$ is analytic on $K$ because it is the product of $t\mapsto\psi_{2t}(x-y)$ that is analytic on $K$ for any $x,y\in\RR^p$ and the function $f$ that does not depend on $t$. Then by Cauchy's integral theorem,
    $\oint_\Gamma F(t,x,w)dt = 0.$
    Consequently, it holds that
    \begin{align*}
        \oint_\Gamma \WW_{2t}[f](x)dt
        =\int_{\RR^p}\left(\oint_\Gamma F(t,x,w)dt \right)dw 
        =0
        \eqsp.
    \end{align*}
    Then, by Morera's theorem $\WW_{2t}[f](x)$ is analytic on any arbitrary $K\subset\Omega$. Finally, $t\mapsto\WW_{2t}[f](x)$ is analytic on $\Omega$ (and in particular on the real interval $]0,1/4a[$).
\end{proof}
Using this lemma, we now prove Theorem~\ref{thm:expressions_weierstrass}.
\begin{proof}[Proof of Theorem~\ref{thm:expressions_weierstrass}.]
    In this proof, we use the parameterization $\sigma^2=2t$. Then, we work on $\WW_{2t}$ and we finally re-inject $\sigma^2$ to obtain the desired result.
    By \citet[Chapter 7, Equation 1.11]{PDE}, if $f$ satisfies $|f(x)|\le M \exp(a\|x\|^2)$ on $\RR^p$, then $u(x,t)=\WW_{2t}[f]$ is an infinitely continuously differentiable solution of the following Heat equation:
    \begin{equation}
        \label{eq:heat_equation}
        \partial_t u(x,t) = \Delta_x u(x,t), \quad u(x,0)= f(x),\quad (x,t)\in\RR^p\times ]0,1/4a[,
    \end{equation}
    where $\Delta_x u(x, t)$ is the Laplacian of the function $x \mapsto u(x, t)$, and the constraint on $u(x,0)$ follows from $u(x,0)=\lim_{t\to0}\WW_{2t}[f](x)$.
    Furthermore, by Lemma~\ref{lemma:analytic_W}, the function $t\mapsto u(x,t)$ is analytic on $]0,1/4a[$. Then, for $x\in\RR^p$ and $t\in]0,1/4a[$, we have the following Taylor expansion around $t_0\in]0,1/4a[$, assuming that the series converge,
    \begin{align*}
        \WW_{2t}[f](x) = u(x,t) &= \sum_{k=0}^\infty \frac{\partial_t^k u(x,t_0)}{k!}(t-t_0)^k 
        = \sum_{k=0}^\infty \frac{\Delta^k u(x,t_0)}{k!}(t-t_0)^k
        \eqsp,
    \end{align*}
    where the second equality comes from the fact that $u$ is solution of the Heat equation \eqref{eq:heat_equation}.
    Taking the limit $t_0 \rightarrow 0$, we obtain the following Taylor expansion around $0$,
    \begin{align}
    \label{eq:expr-ww-in-2t}
        \WW_{2t}[f](x)
        &= \sum_{k=0}^\infty \frac{\Delta^k u(x,0)}{k!}t^k= \sum_{k=0}^\infty \frac{\Delta^k f(x)}{k!}t^k \eqsp.
    \end{align}
    It remains to check if this series converges for a given $x\in\RR^p$.
    As $f$ is in $\funcspace{a}$, we have $|\Delta^k[f](x)|\le A_x(4a)^k k!$.
    Consequently, the 
    \emph{root test} condition 
    $$\lim_k\sup \left| \frac{\Delta^k[f](x)}{k!}\right|^{1/k}<\infty\eqsp.$$
    is met, and the series converges for $|t|<\frac{1}{r}$, where $r=\lim_k\sup \big| \frac{\Delta^k[f](x)}{k!}\big|^{1/k}$, with the convention that $1/0 = \infty$. 
    Since $f$ is in $\funcspace{a}$, $r\le 4a$ and the series converges for $|t| \le \tfrac{1}{4a} \le \frac{1}{r}$.
    The result follows from plugging $t = \sigma^2/2$ in \eqref{eq:expr-ww-in-2t}.
\end{proof}

\subsection{Weierstrass Transform Properties}
\label{app:weierstrass_properties}
We now give two useful properties of the Weierstrass transform.
\begin{proposition}[Linearity of $\WW_{\sigma^2}$]
    Let $\sigma>0$, the Weierstrass transform is linear with respect to the function it applies to. Let $f$ and $g$ be two functions from $\RR^p$ to $\RR$ and let $\alpha\in\RR$,
    $$\WW_{\sigma^2}[\alpha f + g] = \alpha \WW_{\sigma^2}[f] + \WW_{\sigma^2}[g].$$
\end{proposition}
\begin{proof}
    By linearity of the expectation, for a given $x\in\RR^p$,
    \begin{align*}
        \WW_{\sigma^2}[\alpha f + g](x) &= \EE_{w\in\sN(0,\sigma^2\Id{p})}\left[ \alpha f(x+w) + g(x+w)\right]\\
        &=\alpha \EE_{w\in\sN(0,\sigma^2\Id{p})}\left[  f(x+w) \right] + \EE_{w\in\sN(0,\sigma^2\Id{p})}\left[  g(x+w)\right]\\
        &=\alpha \WW_{\sigma^2}[f](x) + \WW_{\sigma^2}[g](x)
        \eqsp,
    \end{align*}
    which is the result.
\end{proof}

\begin{proposition}[$\WW_{\sigma^2}$ is increasing]
    Let $\sigma>0$, the Weierstrass transform is increasing with respect to the function it applies to. Let $f$ and $g$ be two functions from $\RR^p$ to $\RR$ such that for any $x\in\RR^p$, $f(x) \le g(x)$, then for any $x\in\RR^p$,
    $$\WW_{\sigma^2}[f](x) \le \WW_{\sigma^2}[g](x) .$$
\end{proposition}
\begin{proof}
    By increasing property of the expectation, for a given $x\in\RR^p$,
    \begin{align*}
        \WW_{\sigma^2}[f](x) &=\EE_{w\in\sN(0,\sigma^2\Id{p})}\left[ f(x+w)\right]
        \le\EE_{w\in\sN(0,\sigma^2\Id{p})}\left[ g(x+w)\right]
        =\WW_{\sigma^2}[g](x) ,
    \end{align*}
    and the lemma follows.
\end{proof}

We state and prove Remark~\ref{remark:exp_is_licit}.
\begin{proposition}[The exponential loss is in $\Phi_{M_a,a}(\RR^p)$]
    Let $f(x) = \exp(-\theta^\top xy)$ for a given pair $(\theta,y)\in\Theta\times \sY$. Define $M_a=\exp\left( \frac{\norm{\Theta}^2}{4a}\right)$ for any $a>0$. The function $f$ is in $\Phi_{M_a,a}(\RR^p)$ for any $a>0$.
\end{proposition}
\begin{proof}
    We first verify the property of Equation~\eqref{eq:laplace_iterate_growth}. Let an arbitrary $a>0$. Let $k\in\NN$ and $x\in\RR^p$, we have the following Laplacian identity 
    $$\Delta^k_x f(x) = |\Delta^k_x f(x)| =  \norm{\theta}^{2k} f(x).$$
    Factorials are increasing faster than any power of a positive number. Then, denoting $k_0$ such that for any $k\in\NN$,
    $$k>k_0 \implies k! > (\norm{\theta}/4a)^{2k},$$
    we have
    $$|\Delta^k_x f(x)| \le A_x( 4a)^kk!\eqsp,$$
    with $A_x = f(x) (\norm{\theta}^2/4a)^{k_0}$.
    We now verify that the property of Equation~\eqref{eq:gaussian_growth} holds with $M_a = \exp\left( \frac{\norm{\Theta}^2}{4a}\right)$. Let $x\in\RR^p$,
    $$f(x) = \exp(-\theta^\top xy) \le \exp(\|\theta\| \|x\|)\le \exp(\|\Theta\| \|x\|).$$
    To prove that that $\exp(\|\Theta\| \|x\|) \le M_a \exp(a\norm{x}^2)$, we need that
    $$a\|x\|^2 - \norm{\Theta} \|x\| +\log(M_a) \ge 0.$$
    It forms a second degree polynomial in $\norm{x}$ with a positive quadratic constant. The inequality is then true for any $\norm{x}>0$ if the discriminant $\norm{\Theta}^2-4a\log(M_a)$ is positive. In other words, if
    $$M_a \le \exp\left( \frac{\norm{\Theta}^2}{4a}\right).$$
    Thus, for any $a > 0$, Equation~\eqref{eq:gaussian_growth} holds with $M_a = \exp\left( \frac{\norm{\Theta}^2}{4a}\right)$, and $f$ is in $\Phi_{M_a,a}(\RR^p)$ for any $a>0$.
\end{proof}

\section{Bias Characterization -- Proof of \Cref{thm:generalized_bishop}}
\label{appendix:bias_char}
We provide a proof of the following theorem.
\RegEffect*
\begin{proof}
    For a given feature-label pair $(x,y)\in\sX\times\sY$, with LDP release $(\tilde x, \tilde y)$ defined in \eqref{eq:ldp_release}, and a model $\theta\in\Theta$, we have
    \begin{align*}
        \EE_{\tilde x, \tilde y}\left[ \ell\left(\theta,\;\tilde x,\;\tilde y\right)\right] &=\BB_{\epsilon_y}\big[ z\mapsto \WW_{\sigma^2}[\ell(\theta,\cdot,z) ] (x)\big](y)\\
        &=\sigmoid(\epsilon_y)\WW_{\sigma^2}[\ell(\theta,\cdot,y)](x) + (1-\sigmoid(\epsilon_y))\WW_{\sigma^2}[\ell(\theta,\cdot,-y)](x).
    \end{align*}
    Taking the expectation with respect to $(x,y)\sim\sD$ yields
    \begin{align*}
        \tilde\risk(\theta) &=\EE_{x,y}\left[\sigmoid(\epsilon_y)\WW_{\sigma^2}[\ell(\theta,\cdot,y)](x) + (1-\sigmoid(\epsilon_y))\WW_{\sigma^2}[\ell(\theta,\cdot,-y)](x)\right]\\
        \cause{Using the Theorem~\ref{thm:expressions_weierstrass}.}
        &=\sum_{k=0}^\infty \frac{\sigma^{2k}}{2^k k!}\left\{\sigmoid(\epsilon_y)\EE_{x,y}\left[\Delta^k_x\ell(\theta,x,y)\right] + (1-\sigmoid(\epsilon_y))\EE_{x,y}\left[\Delta^k_x\ell(\theta,x,-y)\right]\right\}\\
        \cause{Isolating the term $k=0$.}
        &=\sigmoid(\epsilon_y)\EE_{x,y}\ell(\theta,x,y) + (1-\sigmoid(\epsilon_y))\EE_{x,y}\ell(\theta,x,-y)\\
        &+\sum_{k=1}^\infty \frac{\sigma^{2k}}{2^k k!}\left\{\sigmoid(\epsilon_y)\EE_{x,y}\left[\Delta^k_x\ell(\theta,x,y)\right] + (1-\sigmoid(\epsilon_y))\EE_{x,y}\left[\Delta^k_x\ell(\theta,x,-y)\right]\right\}.
    \end{align*}
    Based on this, we have
    \begin{align*}
        \tilde\risk(\theta)
        &=
        \EE_{x,y}\ell(\theta,x,y)
        +
        (\sigmoid(\epsilon_y)-1)\EE_{x,y}\ell(\theta,x,y) + (1-\sigmoid(\epsilon_y))\EE_{x,y}\ell(\theta,x,-y)\\
        &+\sum_{k=1}^\infty \frac{\sigma^{2k}}{2^k k!}\left\{\sigmoid(\epsilon_y)\EE_{x,y}\left[\Delta^k_x\ell(\theta,x,y)\right] + (1-\sigmoid(\epsilon_y))\EE_{x,y}\left[\Delta^k_x\ell(\theta,x,-y)\right]\right\}
        \\
        &=
        \EE_{x,y}\ell(\theta,x,y)
        +
        (1-\sigmoid(\epsilon_y))\big( \EE_{x,y}\ell(\theta,x,-y) - \EE_{x,y}\ell(\theta,x,y) \big) 
        \\
        &+\sum_{k=1}^\infty \frac{\sigma^{2k}}{2^k k!}\left\{\sigmoid(\epsilon_y)\EE_{x,y}\left[\Delta^k_x\ell(\theta,x,y)\right] + (1-\sigmoid(\epsilon_y))\EE_{x,y}\left[\Delta^k_x\ell(\theta,x,-y)\right]\right\}
        ,
    \end{align*}
    and the result follows by identifying $\risk(\theta) = \EE_{x,y}\ell(\theta,x,y)$.
\end{proof}

\section{Bias Correction Proofs}
\label{appendix:bias_correct}

\subsection{Inverse of Transforms -- Proof of \Cref{thm:inv_RR_weier}}
\label{appendix:inverse_transforms}

We provide a proof of the following theorem.
\InvWeierRR*
\begin{proof}
We first prove (i). Let $\tilde y\in\sY$,
    \begin{align*}
        \BB^{-1}_\epsilon[\BB_\epsilon[g]](\tilde y) &= \BB^{-1}_\epsilon\left[\sigmoid(\epsilon)g(\cdot) +(1-\sigmoid(\epsilon))g(-\:\cdot) \right](\tilde y)\\
        &= \tilde \sigmoid(\epsilon)\left\{ \sigmoid(\epsilon)g(\tilde y) +(1-\sigmoid(\epsilon))g(-\tilde y)\right\}
        +(1-\tilde \sigmoid(\epsilon))\left\{ \sigmoid(\epsilon)g(-\tilde y) +(1-\sigmoid(\epsilon))g(\tilde y)\right\}\\
        \cause{Group $g(\tilde y)$ and $g(-\tilde y)$ terms.}
        &= \left\{ \tilde \sigmoid(\epsilon)\sigmoid(\epsilon) +(1-\tilde \sigmoid(\epsilon))(1-\sigmoid(\epsilon))\right\}g(\tilde y)
        +\left\{ (1-\tilde \sigmoid(\epsilon))\sigmoid(\epsilon)+ (1- \sigmoid(\epsilon))\tilde \sigmoid(\epsilon)\right\}g(-\tilde y)\\
        &= \left\{ 1 + 2\tilde \sigmoid(\epsilon)\sigmoid(\epsilon) -\tilde \sigmoid(\epsilon) -\sigmoid(\epsilon)\right\}g(\tilde y)
        +\left\{\tilde \sigmoid(\epsilon)+ \sigmoid(\epsilon) -2\tilde \sigmoid(\epsilon)\sigmoid(\epsilon)\right\}g(-\tilde y)\\
        \cause{Developing $\sigmoid(\cdot)$ and $\tilde \sigmoid(\cdot)$.}
        &=\{1\}g(\tilde y) + \{0\}g(-\tilde y) = g(\tilde y).
    \end{align*}
    
Now we prove (ii).
    We use the parameterization $\sigma^2=2t$, it remains to reinject $\sigma^2$ to finish the proof. By Theorem~\ref{thm:expressions_weierstrass} proof, the series $t\mapsto\sum_k \frac{\Delta_x^k f(x)}{k!}t^k$ converges absolutely for $t$ in $]0,1/4a[$ for any $x\in\RR^p$. 
    Since considering the series for $-t$ leads to the same coefficients in absolute value, then the series $t\mapsto\sum_k \frac{(-1)^k\Delta_x^k f(x)}{k!}t^k$ also converges for $0<t<1/4a$.

    We thus consider this series as a candidate for the inverse of the Weierstrass transform.
    For $f$ in $\funcspace{a}$, $x\in\RR^p$ and $t\in]0,1/8a[$, using the expression of the Weierstrass transform and its candidate inverse,
    \begin{align*}
        (\WW_{2t} \circ \WW_{2t}^{-1})[f](x) &= \WW_{2t}\left[\sum_{j=0}^\infty \frac{(-t)^j}{j!} \Delta^{j}[f](\cdot)\right](x)
        =\sum_{k=0}^\infty \frac{t^k}{k!} \Delta^j\left[\sum_{k=0}^\infty \frac{(-t)^j}{j!}\Delta^{j}[f](\cdot)\right](x).
    \end{align*}
    
    We want to reorder the series and swap $\sum_{k=0}^\infty$ with $\Delta^j$ to obtain
    $$(\WW_{2t} \circ \WW_{2t}^{-1})[f](x) =\sum_{j=0}^\infty \sum_{k=0}^\infty \frac{(-t)^jt^k}{j!k!} \Delta^{k+j}[f](x).$$
    That is valid if the resulting series converges absolutely. This is indeed the case since
    \begin{align*}
        \sum_{j=0}^\infty \sum_{k=0}^\infty \frac{t^{j+k}}{j!k!} \left|\Delta^{k+j}[f](x)\right|&\le \sum_{j=0}^\infty \sum_{k=0}^\infty \frac{t^{j+k}}{j!k!} A_x (4a)^{j+k} (j+k)!\\
        \cause{Reordering (valid by positivity) with $n=k+j$.}
        &=\sum_{n=0}^\infty \sum_{j=0}^n \frac{t^n}{(n-j)!j!} A_x (4a)^n n!\\
        &=A_x \sum_{n=0}^\infty (4at)^n \sum_{j=0}^n \binom{n}{j}\\
        \cause{As $\sum_{j=0}^n \binom{n}{j}=2^n$.}
        &=A_x \sum_{n=0}^\infty (8at)^n\\
        \cause{As $8at<1$.}
        &=\frac{A_x}{1-8at}<\infty.
    \end{align*}
    Then the reordering and swap of derivative and series are valid. We can thus write
    \begin{align*}
        (\WW_{2t} \circ \WW_{2t}^{-1})[f](x) &= \sum_{j=0}^\infty \sum_{k=0}^\infty \frac{(-t)^jt^k}{j!k!} \Delta^{k+j}[f](x)\\
            \cause{Same reordering with $n=k+j$.}
        &= \sum_{n=0}^\infty \sum_{j=0}^n \frac{(-1)^j t^{n}}{(n-j)!j!} \Delta^{n}[f](x)\\
        &=\sum_{n=0}^\infty \frac{t^n\Delta^{n}[f](x)}{n!}\sum_{j=0}^n \binom{n}{j}(-1)^{j}1^{n-j} =\sum_{n=0}^\infty \frac{t^n\Delta^{n}[f](x)}{n!} (1-1)^n \\
        &=\frac{0^n\Delta^{0}[f](x)}{0!} \\
        \cause{With the convention that $0^0=1$.}
        &= f(x).
    \end{align*}
    We reparameterize with $\sigma^2=2t$, then for $a<1/4\sigma^2$, $(\WW_{\sigma^2} \circ \WW_{\sigma^2}^{-1})[f] = f$.
\end{proof}

\subsection{Commutativity of the gradient operator with the transforms}
\label{appendix:IWP_grad_swap}
\begin{proposition}
Let $\epsilon,\delta>0$, and let $\ell$ be a loss which satisfies~\ref{ass:weierstrass_ok_assumption} with $a<1/2\sigma^2$. Then $\TT_{\epsilon,\delta}$ and $\TT_{\epsilon,\delta}^{-1}$ applied to $\theta \mapsto \ell(\theta, \cdot, \cdot)$, commute with $\nabla_\theta$.
\end{proposition}
\begin{proof}
    Let a tuple $(\theta,x,y)\in\Theta\times\sX\times\sY$, we express $\TT^{-1}$ with $\BB^{-1}$ and $\WW^{-1}$:
    \begin{align}
    \nonumber
        \nabla_\theta \TT_{\epsilon,\delta}^{-1} & [\ell(\theta,\cdot,\cdot)](x,y) =\nabla_\theta\left[ \BB_{\epsilon_y}^{-1}\left[ z\mapsto\WW_{\sigma^2}^{-1}[\ell(\theta,\cdot,z)]( x) \right]( y)\right]\\
    \nonumber
        \cause{Replacing $\WW^{-1}$ with its expression.}
    \nonumber
        &= \nabla_\theta\left[ \BB_{\epsilon_y}^{-1}\left[ z\mapsto \sum_{k=0}^\infty \frac{(-1)^k\sigma^{2k}}{2^k k!} \Delta^k_x \ell(\theta,x,z) \right]( y)\right]\\
    \nonumber
        \cause{Replacing $\BB^{-1}$ with its expression.}
    \nonumber
        &= \nabla_\theta \left[ \tilde S(\epsilon_y)\sum_{k=0}^\infty \frac{(-1)^k\sigma^{2k}}{2^k k!} \Delta^k_x \ell(\theta,x,y) + \left(1-\tilde S(\epsilon_y)\right)\sum_{k=0}^\infty \frac{(-1)^k\sigma^{2k}}{2^k k!} \Delta^k_x \ell(\theta,x,-y)\right]\\
        \cause{By linearity of gradients with the finite sum.}
        \label{eq:before_grad_swap}
        &=\tilde S(\epsilon_y)\nabla_\theta \left[ \sum_{k=0}^\infty \frac{(-1)^k\sigma^{2k}}{2^k k!} \Delta^k_x \ell(\theta,x,y)\right] +  \left(1-\tilde S(\epsilon_y)\right)\nabla_\theta \left[\sum_{k=0}^\infty \frac{(-1)^k\sigma^{2k}}{2^k k!} \Delta^k_x \ell(\theta,x,-y)\right].
    \end{align}
    We now check for $y$ and $-y$ if we can swap the series and gradients. The reasoning is the same for both and we then present it only for $y$. By series differentiation \citep[Theorem 7.17]{math_analysis}, if the series $ \sum_{m=1}^\infty \frac{\sigma^{2m}}{2^m m!} |\partial_{\theta_j}\Delta^m_x \ell(\theta,x,y)|$ converges for each component $j\in\{1,\dots,k\}$, then
    \begin{align}
        \label{eq:swap_series_grad}
        \nabla_\theta \left[ \sum_{k=0}^\infty \frac{(-1)^k\sigma^{2k}}{2^k k!} \Delta^k_x \ell(\theta,x,y)\right] = \sum_{k=0}^\infty \frac{(-1)^k\sigma^{2k}}{2^k k!} \nabla_\theta\Delta^k_x \ell(\theta,x,y).
    \end{align}
     As $\Delta^k$ is a finite sum of iterated derivatives, it commutes with $\nabla_\theta$. We then need to check that for $j\in\{1,\dots,k\}$,
     $$ \sum_{m=0}^\infty \frac{\sigma^{2m}}{2^m m!} \Delta^m_x \left[ \partial_{\theta_j}\ell(\theta,x,y)\right]$$
     is a convergent series. By hypothesis on the loss, we have $\partial_{\theta_j}\ell(\theta,x,y)\in\funcspace{a}$, then 
    \begin{align*}
        \sum_{m=0}^\infty \frac{\sigma^{2m}}{2^m m!} \left|\Delta^m_x\left[\partial_{\theta_j} \ell(\theta,x,y)\right]\right| 
        \le \sum_{m=0}^\infty \frac{\sigma^{2m}}{2^m m!} A_x (4a)^m
        =A_x \sum_{m=0}^\infty (2a\sigma^2)^m
        =\frac{A_x}{1-2a\sigma^2} <\infty,
    \end{align*}
    where the last inequality follows from $2a\sigma^2<1$.
    Injecting Equation~\ref{eq:swap_series_grad} in Equation~\eqref{eq:before_grad_swap} and swapping $\Delta^k$ and $\nabla_\theta$ yields
    \begin{align}
    \nonumber
        \nabla_\theta \TT_{\epsilon,\delta}^{-1}  [\ell(\theta,\cdot,\cdot)](x,y)
        & =
        \tilde S(\epsilon_y) \sum_{k=0}^\infty \frac{(-1)^k\sigma^{2k}}{2^k k!} \nabla_\theta \Delta^k_x \ell(\theta,x,y) +  \left(1-\tilde S(\epsilon_y)\right)\sum_{k=0}^\infty \frac{(-1)^k\sigma^{2k}}{2^k k!} \nabla_\theta \Delta^k_x \ell(\theta,x,-y)
        \\
    \nonumber
        & = \BB_{\epsilon_y}^{-1}\left[ z\mapsto\WW_{\sigma^2}^{-1}[\nabla_\theta\ell(\theta,\cdot,z)]( x) \right]( y)
        .
    \end{align}
    The exact same reasoning can be carried out with $\TT$ replacing $\tilde \sigmoid$ by $\sigmoid$ and the terms $(-1)^k$ by $1$ which does not affect the convergence of series involved.
\end{proof}

\subsection{Unbiasedness of $\tilde\ell_{\epsilon,\delta}$ and $\nabla_\theta\tilde\ell_{\epsilon,\delta}$ -- Proof of \Cref{thm:unbiased}}
\label{app:unbias_IWP}
We prove the following theorem.
\UnbiasedGrad*
\begin{proof}
    We first prove that for any $h:\sX\times\sY\to\RR$ such that for any $y\in\sY$, $x\mapsto h(x,y)\in\funcspace{a}$, 
    \begin{equation}
        \label{eq:inverse_T_on_h}
        \TT_{\epsilon,\delta}[\TT^{-1}_{\epsilon,\delta}[h(\cdot,\cdot)]](x,y) = h(x,y).
    \end{equation}
    We can write
    \begin{align*}
        \TT_{\epsilon,\delta}[\TT^{-1}_{\epsilon,\delta}[h(\cdot,\cdot)]](x,y) &= \BB_{\epsilon_y}\left[ \WW_{\sigma^2}\left[ \BB^{-1}_{\epsilon,\delta} \left[\WW^{-1}_{\sigma^2}[h(\cdot,\cdot)]\right]\right]\right](x,y)\\
        \cause{As $\BB^{-1}_{\epsilon_y}$ forms the sum of two functions in $\funcspace{a}$, it then commutes with $\WW^{-1}_{\sigma^2}$.}
        &=\BB_{\epsilon_y}\left[ \WW_{\sigma^2}\left[\WW^{-1}_{\sigma^2} \left[\BB^{-1}_{\epsilon,\delta} [h(\cdot,\cdot)]\right]\right]\right](x,y)\\
        \cause{Simplifying $\WW\circ\WW^{-1}$.}
        &=\BB_{\epsilon_y}\left[\BB^{-1}_{\epsilon,\delta} [h(\cdot,\cdot)]\right](x,y)\\
        \cause{Simplifying $\BB\circ\BB^{-1}$.}
        &=h(x,y).
    \end{align*}
    Considering Equation~\eqref{eq:inverse_T_on_h} with $h:(x,y)\mapsto\ell(\theta,x,y)$ for a given $\theta\in\Theta$ yields
    \begin{align*}
        \EE_{(\tilde x,\tilde y)} \left[\tilde \ell_{\epsilon,\delta}(\theta,\tilde x,\tilde y)\right] &= \TT_{\epsilon,\delta}[\TT^{-1}_{\epsilon,\delta}[\ell(\theta,\cdot,\cdot)]](x,y)\\
        \cause{By Equation~\eqref{eq:inverse_T_on_h}.}
        &=\ell(\theta,x,y).
    \end{align*}
    Now, for the gradient, we also have the following.
    \begin{align*}
        \EE_{(\tilde x,\tilde y)}\left[ \nabla_\theta \tilde \ell_{\epsilon,\delta}(\theta,\tilde x,\tilde y)\right] &= \TT_{\epsilon,\delta}\left[ \TT^{-1}_{\epsilon,\delta}[\nabla_\theta\ell(\theta,\cdot,\cdot)\right](x,y)\\
        \cause{By commutativity of $\nabla_\theta$ and $\TT^{-1}_{\epsilon,\delta}$.}
        &=\TT_{\epsilon,\delta}\left[ \nabla_\theta\TT^{-1}_{\epsilon,\delta}[\ell(\theta,\cdot,\cdot)\right](x,y)\\
        \cause{By commutativity of $\nabla_\theta$ and $\TT_{\epsilon,\delta}$.}
        &=\nabla_\theta \TT_{\epsilon,\delta}\left[ \TT^{-1}_{\epsilon,\delta}[\ell(\theta,\cdot,\cdot)\right](x,y)\\
        \cause{By Equation~\eqref{eq:inverse_T_on_h}.}
        &=\nabla_\theta\ell(\theta,x,y).
    \end{align*}
\end{proof}
\subsection{Variance of $\boldsymbol{\nabla\tilde\ell_{\epsilon,\delta}(\theta,\tilde x, \tilde y)}$ -- Proof of \Cref{thm:our_gradient_variance}}
\label{appendix:variance_IWP}
We provide a proof of the following result.
\begin{lemma}[Variance of $\WW^{-1}_{\sigma^2} \lbrack f\rbrack(x+w)$]
    \label{lemma:weierstrass_variance}
    Let $f\in \funcspace{a}$ such that $\left(\WW^{-1}_{\sigma^2}[f]\right)^2$ is in $\funcspace{a}$. Let $\sigma^2\in]0,1/4a[$, $x\in\RR^p$ and $w \sim \sN(0, \sigma^2 \Id{p})$. The variance of $\WW^{-1}_{\sigma^2}[f](x+w)$ is
    \begin{align*}
    \VV_w\left(\WW^{-1}_{\sigma^2}[f](x+w)\right) &= \WW_{\sigma^2}\left[\left(\WW^{-1}_{\sigma^2}[f]\right)^2 \right](x)-f^2(x)\\
    &= 2\int_0^{\sigma^2/2} \WW_{2s} \norm{\WW^{-1}_{2s} [\nabla f](\cdot)}^2(x)ds.
    \end{align*}
    
    Furthermore, if there exists $C \ge 0$ such that  $\sup_{x\in\sX}\sup_{0<s<t}\norm{ \WW_{2s}^{-1} [\nabla f](x) }\le C$, it holds that
    \begin{align*}
    \VV_w\left(\WW^{-1}_{\sigma^2}[f](x+w)\right) \le C^2\sigma^2.
    \end{align*}
\end{lemma}
\begin{proof}
    For the proof, we use the parameterization $2t=\sigma^2$. Denote $g_t(x) = \WW_{2t}^{-1}[f](x)$ and $v(x,t) = \EE_w[g_t^2(x+w)] = \WW_{2t}[g_t^2](x)$. We first have that $\EE_w\left( g_t(x+w)-f(x)\right)^2 = v(x,t)-f^2(x)$. So we focus on the term $v(x,t)$ that we will compute in the following integral form by integrating over $t$:
    \begin{align}
        \label{eq:integral_rep}
        v(x,t)
        &= 
        v(x, 0)+\int_0^t \partial_s v(x,s)ds
        = 
        f^2(x)+\int_0^t \partial_s v(x,s)ds
        .
    \end{align}
    We develop $\partial_s v(x,s)$:
    \begin{align*}
        \partial_s v(x,s) &= \partial_s \WW_{2s}[g_s^2](x)\\
        \cause{By hypothesis, $g_s^2 = (\WW_{2s}[f])^2$ is in $\funcspace{a}$}
        &=\partial_s\left[ \sum_{k=0}^\infty \frac{s^k}{k!} \Delta^k[g_s^2](x)\right]\\
        &=\sum_{k=0}^\infty \frac{s^{k-1}}{(k-1)!} \Delta^k[g_s^2](x) + \sum_{k=0}^\infty \frac{s^k}{k!} \Delta^{k}[\partial_s g_s^2](x)\\
        &=\sum_{k=0}^\infty \frac{s^{k}}{k!} \Delta^{k+1}[g_s^2](x) + \sum_{k=0}^\infty \frac{s^k}{k!} \Delta^{k}[\partial_s g_s^2](x)\\
        &=\WW_{2s}[\Delta g_s^2](x) + \WW_{2s}[\partial_s g_s^2](x) =\WW_{2s}[\Delta g_s^2 + \partial_s g_s^2](x) .\\
    \end{align*}
    Now we develop the derivative $\partial_s g_s^2 = 2g_s\partial_s g_s = -2g_s\Delta g_s$ and $\Delta g_s^2 = 2g_s\Delta g_s + 2\|\nabla g_s\|^2$ to reinject it in the previous derivation:
    \begin{align*}
        \partial_s v(x,s) &=\WW_{2s}[2g\Delta g_s + 2\|\nabla g_s\|^2-2g_s\Delta g_s](x) = 2\WW_{2s}[\|\nabla g_s\|^2](x).\\
    \end{align*}
    Since $\nabla g_s = \nabla \WW_{2s}^{-1}[f] = \WW_{2s}^{-1}[\nabla f]$, we further simplify the expression of $\partial_s v(x,s)$ in the integral representation \eqref{eq:integral_rep}, which yields the desired result:
    \begin{align*}
        v(x,t) &= f^2(x)+\int_0^t 2\WW_{2s}[\|\WW_{2s}^{-1}[\nabla f](\cdot)\|^2](x)ds.
    \end{align*}
    Now, assume $\sup_{x\in\sX}\sup_{0<s<t}\|\WW_{2s}^{-1}[\nabla f](x)\|\le C$, denoting $\psi_{2s}(w) = \frac{1}{\sqrt{4\pi s}}\exp\left(-\frac{\norm{w}^2}{4s}\right)$ the probability density function of a centered isotropic Gaussian distribution of variance $2s\Id{p}$, we have 
    \begin{align*}
        \WW_{2s}[\|\WW_{2s}^{-1}[\nabla f](\cdot)\|^2](x) = \int \|\WW_{2s}^{-1}[\nabla f](x-w)\|^2\psi_{2s}(w)dw \le \sup_{x\in\sX}\sup_{0<s<t}\|\WW_{2s}^{-1}[\nabla f](x)\|^2\int \psi_{2s}(w)dw = C^2.
    \end{align*}
    And the result follows from $v(x,t) \le f^2(x)+2\int_0^t C^2ds =  f^2(x)+2tC^2$.
\end{proof}

We now have an unbiased estimator of any function $f\in \funcspace{a}$ at any point $x\in \RR^p$ from a Gaussian-perturbed release of the point $x+w$ for which we can compute the variance exactly.

The variance can be derived in closed-form for known functions like:
\begin{itemize}
    \item if $f(x)=\frac{1}{2} x^\top A x + b^\top x + c$, then $\VV_w(\WW^{-1}_{2t}[f](x+w)) =  f^2(x) + 2t\norm{\Sigma x+b}^2 + 2t^2 \Tr(\Sigma^2)$ with $\Sigma = (A+A^\top)/2$,
    \item if $f(x) = \exp(\alpha^\top x),\;\alpha\in\RR^p$, then  $\VV_w(\WW^{-1}_{2t}[f](x+w)) =  \exp(2a^\top x + 2t \|a\|^2)$.
\end{itemize}
Particular cases of the examples are respectively the mean squared error and the exponential loss (see subsection~\ref{subsec:gen_lin_loss}).

Similarly, the existence of an inverse $\BB^{-1}_\epsilon$, means that for any function $g:\{-1,1\}\to\RR$, any $y\in\{-1,1\}$ and $\epsilon>0$, 
$$\BB^{-1}_\epsilon[g](\tilde y),\quad \tilde y \sim \sB_\epsilon(y)$$
is an unbiased estimator of $g(y)$. We give an exact expression for its variance in the following theorem.
\begin{lemma}[Variance of $\BB^{-1}_\epsilon \lbrack g \rbrack (\sB_\epsilon(y))$]
    \label{lemma:RR_variance}
    Let $g \colon \{-1,1\} \rightarrow \RR$, $\epsilon>0$ and $y\in\{-1,1\}$. The variance of $\BB^{-1}_\epsilon[g](\tilde y)$  with $\tilde y = \sB_\epsilon(y)$ is
    \begin{align*}
    \VV_{\sB_\epsilon}\left(\BB^{-1}_\epsilon[g](\tilde y)\right) = \tilde\sigmoid(\epsilon)(\tilde\sigmoid(\epsilon)-1)(g(1)-g(-1))^2\eqsp.
    \end{align*}
\end{lemma}
\begin{proof}
    For clarity, we denote $\sigmoid\equiv\sigmoid(\epsilon)$ and $\tilde \sigmoid\equiv\tilde\sigmoid(\epsilon)$.
    \begin{align*}
        \VV_{\sB_\epsilon}\left(\BB^{-1}_\epsilon[g](\tilde y)\right) &=\EE_{\sB_\epsilon}\left[ \left(\BB^{-1}_\epsilon[g](\tilde y) - g(y)\right)^2\right]\\
        &=\sigmoid\left(\BB^{-1}_\epsilon[g](y) - g(y)\right)^2+(1-\sigmoid)\left(\BB^{-1}_\epsilon[g](- y) - g(y)\right)^2\\
        &=\sigmoid\left(\tilde \sigmoid g(y) + (1-\tilde\sigmoid)g(-y) - g(y)\right)^2+(1-\sigmoid)\left(\tilde \sigmoid g(-y)+(1-\tilde\sigmoid)g(y) - g(y)\right)^2\\
        &=\sigmoid\left((\tilde \sigmoid -1) g(y) + (1-\tilde\sigmoid)g(-y)\right)^2+(1-\sigmoid)\left(\tilde \sigmoid g(-y)-\tilde\sigmoid g(y)\right)^2\\
        &=\sigmoid(1-\tilde\sigmoid)^2\left(g(y) -g(-y)\right)^2+(1-\sigmoid)\tilde \sigmoid^2\left(g(-y)-g(y)\right)^2\\
        \cause{$\sigmoid$ and $\tilde \sigmoid$ are replaced by their expressions and $(g(y)-g(-y))^2 = (g(1)-g(-1))^2$ for any $y\in\{1,-1\}$.}
        &=\left\{ \sigmoid(1-\tilde\sigmoid)^2 + (1-\sigmoid)\tilde\sigmoid^2\right\}(g(1)-g(-1))^2\\
        &=\left\{ \frac{e^\epsilon}{e^\epsilon+1}\frac{1}{(e^\epsilon-1)^2} + \frac{1}{e^\epsilon+1}\frac{e^{2\epsilon}}{(e^\epsilon-1)^2}\right\}(g(1)-g(-1))^2\\
        &=\frac{e^\epsilon}{(e^\epsilon-1)^2}\left\{ \frac{1}{e^\epsilon+1}+ \frac{e^\epsilon}{e^\epsilon+1}\right\}(g(1)-g(-1))^2\\
        \cause{As $\frac{e^\epsilon}{(e^\epsilon-1)^2}=\tilde\sigmoid(\epsilon)(\tilde\sigmoid(\epsilon)-1)$.}
        &= \tilde\sigmoid(\epsilon)(\tilde\sigmoid(\epsilon)-1)(g(1)-g(-1))^2 \eqsp.
    \end{align*}
\end{proof}
Let us now restate the variance of the IWP gradient estimator and provide the proof.
\OurGradientVar*
\begin{proof}
    We first decompose the variance for each component of $\nabla\tilde\ell_{\epsilon,\delta}(\theta,\tilde x, \tilde y)$.
    \begin{align*}
        \EE \|\nabla_\theta \tilde \ell(\theta, \tilde x, \tilde y) \|^2 - \|\nabla\ell(\theta, x,  y)\|^2&= \sum_{j=1}^k \left( \EE\bigabs{\partial_{\theta_j} \tilde\ell_{\epsilon,\delta}(\theta, \tilde x, \tilde y)}^2 - \bigabs{\partial_{\theta_j} \ell(\theta, x,  y)}^2 \right)\\
        &= \sum_{j=1}^k \underbrace{\VV\left( \partial_{\theta_j} \tilde\ell_{\epsilon,\delta}(\theta, \tilde x, \tilde y)\right)}_{(A_j)}\eqsp.\\
    \end{align*}
    We are interested in computing $(A_j)$ for any $j\in\{1,\dots,k\}$. We can decompose it using total variance law:
    \begin{align*}
        (A_j) = \underbrace{\EE_{\tilde y} \left[ \VV_{\tilde x}\left( \partial_{\theta_j} \tilde\ell_{\epsilon,\delta}(\theta, \tilde x, \tilde y)\;\Bigg|\;\tilde y\right)\right]}_{(a_j)} \;+\; \underbrace{\VV_{\tilde y} \left( \EE_{\tilde x}\left[\partial_{\theta_j} \tilde\ell_{\epsilon,\delta}(\theta, \tilde x, \tilde y)\;\Bigg|\;\tilde y\right]\right)}_{(b_j)}\eqsp.
    \end{align*}
    We start with the term $(a_j)$. For that, we need the following expression of $\partial_{\theta_j} \tilde\ell_{\epsilon,\delta}(\theta, \tilde x, \tilde y)$:
    \begin{align*}
        \partial_{\theta_j} \tilde\ell_{\epsilon,\delta}(\theta, \tilde x, \tilde y) &=\partial_{\theta_j}\left[\WW^{-1}_{\sigma^2} \left[z\mapsto\BB^{-1}_{\epsilon_y}[\ell(\theta,z,\cdot)](\tilde y)\right](\tilde x)\right]\\
        &= \WW^{-1}_{\sigma^2} \left[z\mapsto \BB^{-1}_{\epsilon_y}[\partial_{\theta_j}\ell(\theta,z,\cdot)](\tilde y)\right](\tilde x).
    \end{align*}
    Then, using Lemma~\ref{lemma:weierstrass_variance} with $f(x) = \BB_{\epsilon_y}^{-1}\left[\partial_{\theta_j}\ell(\theta,x,\cdot)\right](\tilde y)$ and $t={\sigma^2/2}$, we have
    \begin{align*}
        \VV_{\tilde{x}}\left(\partial_{\theta_j}\ell(\theta,\tilde x ,\tilde y)\;\big|\;\tilde y\right) &=\WW_{\sigma^2}\left[\left( \WW^{-1}_{\sigma^2} \left[z\mapsto \BB^{-1}_{\epsilon_y}[\partial_{\theta_j}\ell(\theta,z,\cdot)](\tilde y)\right](\tilde x)\right)^2 \right] - \left(\BB_{\epsilon_y}^{-1}\left[\partial_{\theta_j}\ell(\theta,x,\cdot)\right](\tilde y)\right)^2\\
        &=2\int_0^{\sigma^2/2} \WW_{2s}[\|\WW_{2s}^{-1}[\nabla_x \BB_{\epsilon_y}^{-1}(\partial_{\theta_j}\ell(\theta,\cdot,\tilde y))]\|^2](x)ds.
    \end{align*}
    We now inject this variance term in $(a_j)$, which gives
    \begin{align*}
        (a_j)
        & = \EE_{\tilde y} \left[ 2\int_0^{\sigma^2/2} \WW_{2s}[\|\WW_{2s}^{-1}[\nabla_x \BB_{\epsilon_y}^{-1}(\partial_{\theta_j}\ell(\theta,\cdot,\tilde y))]\|^2](x)ds\right]\\
        \cause{Developping the gradient norm $\norm{ \nabla_x h(x) }^2 = \textstyle{\sum_{i=1}^p} (\partial_{x_i} h(x))^2$, and swapping $\partial_{x_i}$ and $\BB^{-1}$.}
        &=\sum_{i=1}^p \EE_{\tilde y} \left[ 2\int_0^{\sigma^2/2} \WW_{2s}[(\WW_{2s}^{-1}[\BB_{\epsilon_y}^{-1}(\partial_{x_i}\partial_{\theta_j}\ell(\theta,\cdot,\tilde y))])^2](x)ds\right]\\
        \cause{Swapping $\WW$ and $\EE_{\tilde y}$.}
        &=\sum_{i=1}^p 2\int_0^{\sigma^2/2} \WW_{2s}\left[\EE_{\tilde y} \Big[ \left(\WW_{2s}^{-1}[\BB_{\epsilon_y}^{-1}(\partial_{x_i}\partial_{\theta_j}\ell(\theta,\cdot,\tilde y))]\right)^2 \Big] \right](x)ds\\
        &=\sum_{i=1}^p 2\int_0^{\sigma^2/2} \WW_{2s}\left[\EE_{\tilde y} \Big[ \left(\BB_{\epsilon_y}^{-1}[\WW_{2s}^{-1}(\partial_{x_i}\partial_{\theta_j}\ell(\theta,\cdot,\tilde y))]\right)^2 \Big] \right](x)ds\eqsp.\\
    \end{align*}
    Now we can use the formula of $\VV(\BB^{-1}(g(\tilde y))$ from Lemma~\ref{lemma:RR_variance} with $g(y) = \WW_{2s}^{-1}(\partial_{x_i}\partial_{\theta_j}\ell(\theta,\cdot,y))$ and the fact that $\EE_{\tilde y} [\BB_{\epsilon_y}^{-1}[\WW_{2s}^{-1}[\partial_{x_i}\partial_{\theta_j}\ell(\theta,\cdot,\tilde y)]] ] = \WW_{2s}^{-1}[\partial_{x_i}\partial_{\theta_j}\ell(\theta,\cdot,\tilde y)]$, which gives
    \begin{align*}
        (a_j) 
        & =\sum_{i=1}^p 2\int_0^{\sigma^2/2} \WW_{2s}\left[
         \VV_{\tilde y} (\BB_{\epsilon_y}^{-1}[\WW_{2s}^{-1}[\partial_{x_i}\partial_{\theta_j}\ell(\theta,\cdot,\tilde y)]])
         +(\EE_{\tilde y} [\BB_{\epsilon_y}^{-1}[\WW_{2s}^{-1}[\partial_{x_i}\partial_{\theta_j}\ell(\theta,\cdot,\tilde y)]] ])^2
         \right](x)ds\\
         &=\sum_{i=1}^p 2\int_0^{\sigma^2/2} \tilde\sigmoid(\epsilon_y)\left(\tilde\sigmoid(\epsilon_y)-1\right)\WW_{2s}\left[
          \left(\WW_{2s}^{-1}[\partial_{x_i}\partial_{\theta_j}\ell(\theta,\cdot,1)] - \WW_{2s}^{-1}[\partial_{x_i}\partial_{\theta_j}\ell(\theta,\cdot,-1)]\right)^2\right](x)ds\\
         &\quad +\sum_{i=1}^p 2\int_0^{\sigma^2/2} \WW_{2s}\left[\left(\WW_{2s}^{-1}[\partial_{x_i}\partial_{\theta_j}\ell(\theta,\cdot,y)]\right)^2
         \right](x)ds.
    \end{align*}
    We rearrange the terms to make the squared norm appear:
    \begin{align*}
        (a_j) 
        & =2\tilde\sigmoid(\epsilon_y)\left(\tilde\sigmoid(\epsilon_y)-1\right)\int_0^{\sigma^2/2} \WW_{2s}\big[ \|\WW_{2s}^{-1}(\nabla_x \partial_{\theta_j}\ell(\theta,\cdot,1) - \WW_{2s}^{-1}(\nabla_x \partial_{\theta_j}\ell(\theta,\cdot,-1)\|^2 \big] (x)ds\\
         &\quad+2\int_0^{\sigma^2}\WW_{2s} \big[ \|\WW_{2s}^{-1}[\nabla_x \partial_{\theta_j}\ell(\theta,\cdot,y)]\|^2 \big] (x)ds.
    \end{align*}
    Finally, remarking that $\EE_{\tilde x}[\partial_{\theta_j} \tilde\ell_{\epsilon,\delta}(\theta, \tilde x, \tilde y)\;\big|\;\tilde y] = \partial_{\theta_j} \tilde\ell_{\epsilon,\delta}(\theta, x, \tilde y)$ and using \Cref{lemma:RR_variance} again gives
    \begin{align*}
        (b_j) &= \VV_{\tilde y} \left( \BB^{-1}_{\epsilon_y}\left[ \partial_{\theta_j} \ell(\theta,x,\cdot)\right](\tilde y)\right)=\tilde\sigmoid(\epsilon_y)\left(\tilde\sigmoid(\epsilon_y)-1\right)\left( \partial_{\theta_j} \ell(\theta,x,1) -\partial_{\theta_j} \ell(\theta,x,-1)\right)^2.
    \end{align*}
    Plugging the results of $(a_j)$ and $(b_j)$ in the formula of $(A_j)$ and summing over $j\in\{1.\dots,k\}$ yields the following expression of the variance of the IWP gradient estimator:
    \begin{align*}
    \nonumber
    &\EE\norm{\nabla_\theta \tilde\ell_{\epsilon,\delta}(\theta,\tilde x,\tilde y)-\nabla_\theta \ell(\theta,x,y)}^2
    \\
     & \quad = 2\int_0^{\sigma^2/2} \WW_{2s} \| \WW_{2s}^{-1}[\nabla_x\nabla_\theta\ell(\theta,\cdot,y)] \|^2(x)ds\\
     \nonumber
     & \qquad + \tilde\sigmoid(\epsilon_y)\left(\tilde\sigmoid(\epsilon_y)-1\right) \| \nabla_\theta\ell(\theta,x,1) - \nabla_\theta\ell(\theta,x,-1) \|^2\\
     \nonumber
    & \qquad + 2\tilde\sigmoid(\epsilon_y)\left(\tilde\sigmoid(\epsilon_y)-1\right)\int_0^{\sigma^2/2} \WW_{2s} \| \WW_{2s}^{-1}[\nabla_x\nabla_\theta\ell(\theta,\cdot,1)-\nabla_x\nabla_\theta\ell(\theta,\cdot,-1)]\|^2(x)ds\eqsp,
    \end{align*}
    where the matrix norm is the Frobenius norm.
    Recalling the definition
    \begin{align*}
        C =\sup_{(\theta,x,y)\in\Theta\times\sX\times\sY,s<\sigma^2/2}\max\big\{ \norm{\nabla_\theta \ell(\theta,x,y)},\norm{ \WW_{2s}^{-1}\left[\nabla_\theta \nabla_x \ell(\theta,\cdot,y)\right](x)}\big\},
    \end{align*}
    we can bound the integrands with $C$ and use the increasing property of $\WW$ to bound the variance and get the result.
\end{proof}

\subsection{Application to Generalized Linear Models (GLM)}
\label{app:GLM}
We provide proofs of derivations for the case of GLM (subsection~\ref{subsec:gen_lin_loss}). 

\paragraph{Loss.}
Recall that in GLM the loss is $\ell(\theta,\tilde x,\tilde y) = f(\theta^\top \tilde x\tilde y)$. It simplifies the expression of iterated Laplacians:
$$ \Delta^k_xf(\theta^\top \tilde x\tilde y) = \norm{\theta}^{2k} f^{(2k)}(\theta^\top \tilde x\tilde y)\eqsp.$$
Then, the Weierstrass inverse has the following expression.
\begin{align*}
    \WW_{2t}^{-1} [\ell(\theta,\cdot,\tilde y)](\tilde x) &= \sum_{k=0}^\infty \frac{\Delta^k_x \ell(\theta,\tilde x,\tilde y) }{k!} (-t)^k = \sum_{k=0}^\infty \frac{\Delta^k_xf(\theta^\top \tilde x\tilde y)}{k!} (-t)^k\\
    &=\sum_{k=0}^\infty \frac{\norm{\theta}^{2k} f^{(2k)}(\theta^\top \tilde x\tilde y)}{k!} (-t)^k = \sum_{k=0}^\infty \frac{(-t\norm{\theta}^2)^k }{k!} f^{(2k)}(\theta^\top \tilde x\tilde y)\\
    &=\WW_{2t\norm{\theta}^2}^{-1}[f](\theta^\top \tilde x\tilde y).
\end{align*}
Recall Equation~\eqref{eq:loss_estim}:
\begin{align*}
    \tilde\ell_{\epsilon,\delta}(\theta,\tilde x,\tilde y) &= \TT_{\epsilon,\delta}^{-1}[\ell(\theta,\cdot,\cdot)](\tilde x,\tilde y) = \WW_{\sigma^2}^{-1}\left[\BB_{\epsilon_y}[\ell(\theta,\cdot,\cdot)](\cdot,\tilde y)\right](\tilde x)\\
    \cause{Applying $\BB^{-1}_{\epsilon_y}$ first.}
    &=\WW_{\sigma^2}^{-1} \left[\tilde \sigmoid(\epsilon_y) \ell(\theta,\cdot,\tilde y)  +\left(1-\tilde\sigmoid(\epsilon_y)\right)\ell(\theta,\cdot,-\tilde y)\right](\tilde x)\\
    \cause{By linearity of $\WW_{2t}^{-1}$.}
    &=\tilde \sigmoid(\epsilon_y)\WW_{\sigma^2}^{-1} [\ell(\theta,\cdot,\tilde y)](\tilde x)+\left(1-\tilde\sigmoid(\epsilon_y)\right)\WW_{\sigma^2}^{-1} [\ell(\theta,\cdot,-\tilde y)](\tilde x)\eqsp.
\end{align*}
Replacing $\WW_{2t}^{-1} [\ell(\theta,\cdot,\tilde y)](\tilde x) = \WW_{2t\norm{\theta}^2}^{-1}[f](\theta^\top \tilde x\tilde y)$ with $2t=\sigma^2$ yields
$$\tilde\ell_{\epsilon,\delta}(\theta,\tilde x, \tilde y) =\tilde \sigmoid(\epsilon_y)\WW_{\sigma^2 \norm{\theta}^2}^{-1}[f]\left(\theta^\top \tilde x \tilde y\right)+\left(1-\tilde\sigmoid(\epsilon_y)\right)\WW_{\sigma^2 \norm{\theta}^2}^{-1}[f]\left(-\theta^\top \tilde x \tilde y\right).$$

\paragraph{Gradient.}
We can't directly differentiate Equation~\eqref{eq:loss_estim} with respect to $\theta$ by swaping Weierstrass transform and gradient ($\nabla_\theta\WW_{2t}^{-1}[\ell(\theta,\cdot,\tilde y)]=\WW_{2t}^{-1}[\nabla_\theta\ell(\theta,\cdot,\tilde y)]$) because here $t$ depends on $\theta$. We thus write the derivative explicitly,
\begin{align*}
    \nabla_\theta \tilde\ell_{\epsilon,\delta}(\theta,\tilde x, \tilde y) &= \tilde \sigmoid(\epsilon_y)\nabla_\theta \WW_{\sigma^2 \norm{\theta}^2}^{-1}[f]\left(\theta^\top \tilde x \tilde y\right)+\left(1-\tilde\sigmoid(\epsilon_y)\right)\nabla_\theta\WW_{\sigma^2 \norm{\theta}^2}^{-1}[f]\left(-\theta^\top \tilde x \tilde y\right).
\end{align*}
It boils down to computing $\nabla_\theta \WW_{\sigma^2 \norm{\theta}^2}^{-1}[f]\left(\theta^\top \tilde x \tilde y\right)$ for a given $\tilde y\in\{-1,1\}$. Let $j\in\{1,\dots,p\}$,
\begin{align*}
    \nabla_\theta \WW_{\sigma^2 \norm{\theta}^2}^{-1}[f]\left(\theta^\top \tilde x \tilde y\right) &= \nabla_\theta \left[  \sum_{k=0}^\infty \frac{(-\tau\norm{\theta}^2)^k}{k!} f^{(2k)}(\theta^\top \tilde x \tilde y)\right],\quad \color{gray}{\text{Denoting } \tau = \sigma^2/2.}\\
    &=\sum_{k=0}^\infty  \frac{(-\tau)^k}{k!} \nabla_\theta \left[\norm{\theta}^{2k} f^{(2k)}(\theta^\top \tilde x \tilde y)\right]\\
    &=\sum_{k=0}^\infty  \frac{(-\tau)^k}{k!} \left\{ 2k\norm{\theta}^{2k-2}\theta f^{(2k)}(\theta^\top \tilde x \tilde y) 
    + \norm{\theta}^{2k} \tilde x \tilde y f^{(2k+1)}(\theta^\top \tilde x \tilde y) \right\}\\
    &=2\theta \sum_{k=0}^\infty  \frac{k(-\tau\norm{\theta}^2)^k}{\norm{\theta}^2 k!} f^{(2k)}(\theta^\top \tilde x \tilde y) 
    + \tilde x \tilde y\sum_{k=0}^\infty  \frac{(-\tau\norm{\theta}^2)^k}{k!} f^{(2k+1)}(\theta^\top \tilde x \tilde y)\\
    &=-2\tau\theta \sum_{k=1}^\infty  \frac{(-\tau\norm{\theta}^2)^{k-1}}{(k-1)!} f^{(2k)}(\theta^\top \tilde x \tilde y) 
    + \tilde x \tilde y\sum_{k=0}^\infty  \frac{(-\tau\norm{\theta}^2)^k}{k!} f^{(2k+1)}(\theta^\top \tilde x \tilde y)\\
    &=-2\tau\theta \sum_{k=0}^\infty  \frac{(-\tau\norm{\theta}^2)^k}{k!} f^{(2k+2)}(\theta^\top \tilde x \tilde y) 
    + \tilde x \tilde y\sum_{k=0}^\infty  \frac{(-\tau\norm{\theta}^2)^k}{k!} f^{(2k+1)}(\theta^\top \tilde x \tilde y).
\end{align*}
Recognizing the Weierstrass transforms and replacing $\tau$, we obtain
\begin{align}
\label{eq:expression-grad-wminusone}
    \nabla_\theta \WW_{\sigma^2 \norm{\theta}^2}^{-1}[f]\left(\theta^\top \tilde x \tilde y\right) 
    &=-\sigma^2\theta \WW_{\sigma^2\norm{\theta}^2}^{-1}[f''](\theta^\top \tilde x \tilde y) + \tilde x \tilde y\WW_{\sigma^2\norm{\theta}^2}^{-1}[f'](\theta^\top \tilde x \tilde y).
\end{align}

\subsection{Uniform Bounds of $\boldsymbol{\| \WW_{2s}^{-1}\nabla_\theta \nabla_x \ell(\theta,x,y) \|}$}
\label{app:C_bound_GLM}
In this subsection, we derive uniform bounds on $\norm{ \WW_{2s}^{-1}\nabla_\theta \nabla_x \ell(\theta,x,y) }$. From the expression \eqref{eq:expression-grad-wminusone}, we have
$$\nabla_\theta \WW_{\sigma^2 \norm{\theta}^2}^{-1}[f]\left(\theta^\top x y\right) = x y\WW_{\sigma^2\norm{\theta}^2}^{-1}[f'](\theta^\top x y) -\sigma^2\theta \WW_{\sigma^2\norm{\theta}^2}^{-1}[f''](\theta^\top x y).$$
It remains to differentiate again with respect to $x$:
\begin{align}
    \nonumber
    \nabla_x\nabla_\theta \WW_{\sigma^2 \norm{\theta}^2}^{-1}[f]\left(\theta^\top x y\right) &=\nabla_x\left[ x y\WW_{\sigma^2\norm{\theta}^2}^{-1}[f'](\theta^\top x y)\right] - \sigma^2\theta\nabla_x \left[ \WW_{\sigma^2\norm{\theta}^2}^{-1}[f''](\theta^\top x y)\right]\\
    &=y\WW_{\sigma^2\norm{\theta}^2}^{-1}[f'](\theta^\top x y)\Id{p} + x \theta^\top\WW_{\sigma^2\norm{\theta}^2}^{-1}[f''](\theta^\top x y)%
    - \sigma^2 y \theta\theta^\top \WW_{\sigma^2\norm{\theta}^2}^{-1}[f'''](\theta^\top x y).
    \label{eq:derivative_w_inv_x_theta}
\end{align}
We can now distinguish multiple choices for the function $f$.
\paragraph{Quadratic: $\boldsymbol{f(z) = \frac{1}{2} (z-1)^2}$.}
Here, $f'(z) = (z-1)$, $f''(z)=1$ and $f'''(z)=0$, thus~\eqref{eq:derivative_w_inv_x_theta} gives
$$\nabla_x\nabla_\theta \WW_{\sigma^2 \norm{\theta}^2}^{-1}[f]\left(\theta^\top x y\right) = x\theta^\top + y(\theta^\top x y - 1)\Id{p}.$$
So we can derive the following bound:
\begin{align*}
    \norm{ \nabla_x\nabla_\theta \WW_{\sigma^2 \norm{\theta}^2}^{-1}[f]\left(\theta^\top x y\right)} &\le \norm{x}\norm{\theta} + p(\norm{x}\norm{\theta}+1)\\
    &\le \sO\left( p\norm{\sX}\norm{\Theta}\right),
\end{align*}
where the matrix norm is the Frobenius norm.

\paragraph{Exponential: $\boldsymbol{f(z) = \exp(-z)}$.}
Here, $f'(z) = f'''(z) = -\exp(-z)$ and $f''(z) = \exp(-z)$, thus~\eqref{eq:derivative_w_inv_x_theta} gives
$$\nabla_x\nabla_\theta \WW_{\sigma^2 \norm{\theta}^2}^{-1}[f]\left(\theta^\top x y\right) 
= e^{-\sigma^2 \norm{\theta}^2/2}e^{-\theta^\top x y}\left( x\theta^\top -y\Id{p} + \sigma^2\theta \theta^\top y\right).$$
So we can derive the following bound:
\begin{align*}
    \norm{ \nabla_x\nabla_\theta \WW_{\sigma^2 \norm{\theta}^2}^{-1}[f]\left(\theta^\top x y\right)} &\le e^{-\sigma^2 \norm{\theta}^2/2}e^{\norm{\theta}\norm{x}}\left( \norm{x}\norm{\theta} + p + \sigma^2\norm{\theta}^2\right)\\
    \cause{Using $\norm{x}\norm{\theta}-\sigma^2 \norm{\theta}^2/2 \le \norm{x}^2/2\sigma^2$.}
    &\le \exp\left(\frac{\norm{x}^2}{2\sigma^2}\right)\left( \norm{x}\norm{\theta} + p + \sigma^2\norm{\theta}^2\right)\\
    &\le \exp\left(\frac{\norm{\sX}^2}{2\sigma^2}\right)\left(p+\norm{\sX}\norm{\Theta}+ \sigma^2\norm{\Theta}^2\right)\\
    \cause{Replacing $\sigma^2$ with its expression, the $\norm{\sX}^2$ simplify.}
    &\le \sO\left(\exp\left(\frac{\epsilon_x^2}{\log(1.25/\delta)}\right)\left(p+\norm{\sX}\norm{\Theta}+ \sigma^2\norm{\Theta}^2\right)\right),
\end{align*}
where the matrix norm is the Frobenius norm.

\subsection{In Absence of Closed-form Expression for $\WW^{-1}$}
\label{app:no_closed_form}
When no closed-form expression of $\WW^{-1}[f]$ is given, we fall into two cases
\begin{enumerate}
    \item $f$ is in $\funcspace{a}$,
    \item $f$ is not in $\funcspace{a}$.
\end{enumerate}
In the two cases, we use the following approximation of $\WW_{2t}^{-1}[f]$:
$$ g^K_t(x) = \sum_{k=0}^K \frac{\Delta^k f(x)}{k!}(-t)^k.$$
In case 1 ($f$ is in $\funcspace{a}$), we can bound the bias of $g^K_t$ uniformly on $\sX$ using the fact that $\lim_{K\to \infty}g^K_t(x) = \WW_{2t}^{-1}[f](x)$,
\begin{align*}
    \text{bias}_K &= \sup_{x\in\sX} |\WW_{2t}\left[g^K_t\right](x)-f(x)| = \sup_{x\in\sX} \left|\WW_{2t}\left[g^K_t - \WW_{2t}^{-1}[f]\right](x)\right|\\
    &\le \sup_{x\in\sX} \left|g^K_t(x) - \WW_{2t}^{-1}[f](x)\right| = \sup_{x\in\sX} \left| \sum_{k=K+1}^\infty \frac{\Delta^k f(x)}{k!}(-t)^k\right|\\
    \cause{By Equation~\eqref{eq:laplace_iterate_growth}, it exist $D>0$ such that}
    &\le D \sup_{x\in\sX} A_x (4at)^{K+1}.
\end{align*}
Then, the bias of using $g^K_t$ instead of $\WW_{2t}^{-1}[f]$ is exponentially decreasing with $K$ since $a<1/4t$.\\
In case 2 ($f$ is not in $\funcspace{a}$), we do not have any guarantee that increasing $K$ will result in a better approximation. Instead, we can estimate the $bias_K$ for small $K\in\{1,2,3,\dots\}$ and take the optimal truncation \citep{devil_invention}:
\begin{align*}
    K^* &= \arg\min_K\left\{ bias_K = \sup_{x\in\sX} \left|\WW_{2t}\left[g^K_t\right](x)-f(x)\right|\right\}\\
    &= \arg\min_K\left\{\sup_{x\in\sX} \left| \int_0^t \frac{(-s)^{K}}{K!} \WW_{2s}\left[ \Delta^{K+1}f\right](x)ds\right|\right\}\eqsp.
\end{align*}
    We present this case for the example of the log loss.
\begin{example}[log loss]
    \label{ex:logloss}
    Consider $f\left(\theta^\top xy\right) =\log(1+\exp(-\theta^\top xy))$ for any $(\theta,x,y)\in\Theta\times\sX\times\sY$. $f$ is not in $\funcspace{a}$. Figure~\ref{fig:trunc_error} shows the estimate of the truncation error $\varepsilon_K$ via numerical integration and Monte-Carlo sampling approximation. We restrict the study to $t\le 25$ which, for the unit ball $\norm{\sX}\le 1$, is true for any $\epsilon\ge 1$. For low $t$, choosing $K\in\{2,3\}$ can be better and for larger one $K=1$ is showing a smaller bias.

\begin{figure}
    \centering
        \includegraphics[trim=0 0 0 0, clip,width=0.6\textwidth]{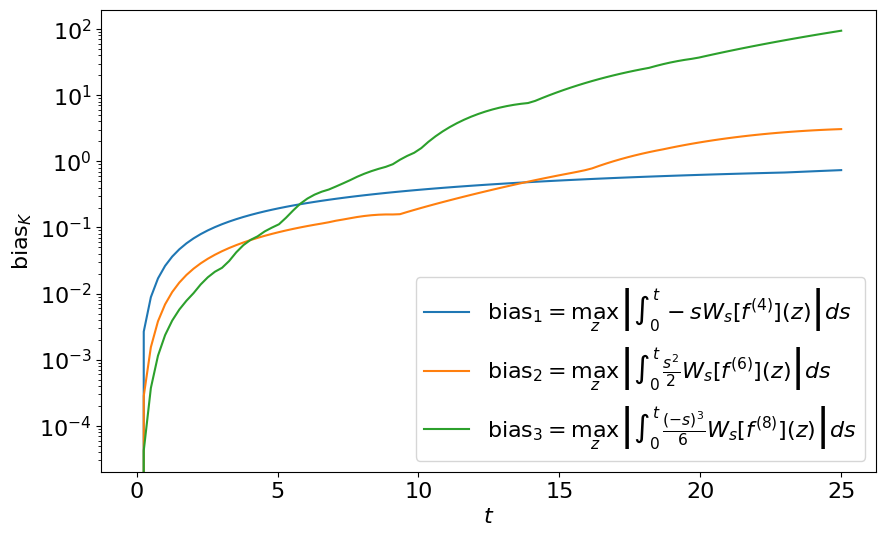}
    \captionof{figure}{Approximation of the truncation error $bias_K$ for $K\in\{1,2,3\}$.}
    \label{fig:trunc_error}
\end{figure}

\end{example}

\subsection{Convergence of IWP-SGD -- Proof of \Cref{thm:conv_IWPSGD}}
\label{appendix:convergence_IWP_SGD}
We provide a proof of the following theorem.
\IWPSGDConv*

\begin{proof}
First, the algorithm is $(\epsilon,\delta)$-LDP by the post-processing theorem as the data is first privatized using the $(\epsilon,\delta)$-LDP release \eqref{eq:ldp_release} before being used for the SGD iterations. The rest of the proof is about the convergence guarantees. For any $t \geq 0$ the iterates of SGD are given by:
\begin{align}
 \theta_{t+1} = \Pi_\Theta(\theta_t- \gamma g_t), \label{eq:sgd}
\end{align}
where
 $\gamma>0$ denotes a step-size and $g_t$ is the IWP gradient estimator computed on the $(\epsilon,\delta)$-LDP release defined in Equation~\eqref{eq:ldp_release}
 $$ g_t = \TT_{\epsilon,\delta}^{-1}[\nabla_\theta\ell(\theta,\cdot,\cdot)](\tilde x,\tilde y).$$ 
For any $t\ge0$, the expectation of $g_t$ is
\begin{align}
 \EE_{(x,y) \sim \sD}\EE_{(\tilde x,\tilde y)}\left[g_t \mid \theta_0,\dots,\theta_{t}\right]= \nabla \risk(\theta_t), \label{def:unbiased}
\end{align}
and the squared gradient satisfies 
\begin{align*}
\EE_{(x,y) \sim \sD}\EE_{(\tilde x,\tilde y)}\left[\|g_t\|^2  \mid \theta_0,\dots,\theta_{t}\right] 
=
\norm{ \nabla \risk(\theta_t) }^2 + 
\EE_{(x,y) \sim \sD} \EE_{(\tilde x,\tilde y)}[ \norm{ g_t - \nabla \risk(\theta_t) }^2 \mid \theta_0,\dots,\theta_{t}] 
\eqsp.
\end{align*}
Denote $\risk^\star = \min_{\theta\in\Theta} \risk(\theta)$. Following \Cref{thm:our_gradient_variance}, and bounding $ \|\nabla \risk(\theta_t)\|^2 \le 2\sK(\risk(\theta_t) - \risk^\star)$, we obtain
\begin{align*}
& \EE_{(x,y) \sim \sD}\EE_{(\tilde x,\tilde y)}\left[\|g_t\|^2 \mid \theta_0,\dots,\theta_{t}\right]
\\ 
& \quad \le \|\nabla \risk(\theta_t)\|^2 + 2\int_0^{\sigma^2/2} \WW_{2s} \| \WW_{2s}^{-1}[\nabla_x\nabla_\theta\ell(\theta,\cdot,y)]\|^2(x)ds
\\
    & \qquad + \frac{2e^{\epsilon_y}}{(e^{\epsilon_y}-1)^2}\int_0^{\sigma^2/2} \WW_{2s} \| \WW_{2s}^{-1}[\nabla_x\nabla_\theta\ell(\theta,\cdot,1) - \nabla_x\nabla_\theta\ell(\theta,\cdot,-1)]\|^2(x)ds\\
    & \qquad + \frac{e^{\epsilon_y}}{(e^{\epsilon_y}-1)^2} \| \nabla_\theta\ell(\theta,x,1) - \nabla_\theta\ell(\theta,x,-1)\|^2\\
   & \quad \le 
   2\sK(\risk(\theta_t) - \risk^\star)
    + 2\int_0^{\sigma^2/2} \WW_{2s} \| \WW_{2s}^{-1}[\nabla_x\nabla_\theta\ell(\theta,\cdot,y)]\|^2(x)ds\\
    & \qquad + \frac{2e^{\epsilon_y}}{(e^{\epsilon_y}-1)^2}\int_0^{\sigma^2/2} \WW_{2s} \| \WW_{2s}^{-1}[\nabla_x\nabla_\theta\ell(\theta,\cdot,1) - \nabla_x\nabla_\theta\ell(\theta,\cdot,-1)]\|^2(x)ds\\
    & \qquad + \frac{e^{\epsilon_y}}{(e^{\epsilon_y}-1)^2} \| \nabla_\theta\ell(\theta,x,1) - \nabla_\theta\ell(\theta,x,-1)\|^2\eqsp.
\end{align*}
Using the assumption that $\sup_{\theta,x,y}\sup_{0<s<\sigma^2/2} \| \WW_{2s}^{-1}\nabla_\theta \nabla_x \ell(\theta,x,y) \| \le C$ and $\sup_{\theta,x,y}\| \nabla_\theta \ell(\theta,x,y) \| \le C$,
\begin{align}
    \label{def:smooth}
    \EE_{(x,y)\sim\sD}\EE_{(\tilde x,\tilde y)}\left[\|g_t\|^2 \mid \theta_0,\dots,\theta_{t}\right] &\le 2\sK(\risk(\theta_t) - \risk^\star) + \frac{4C^2e^{\epsilon_y}}{(e^{\epsilon_y}-1)^2}+C^2\sigma^2\left(1 + \frac{4e^{\epsilon_y}}{(e^{\epsilon_y}-1)^2}\right)\eqsp. 
\end{align}
We denote $A=\frac{4C^2e^{\epsilon_y}}{(e^{\epsilon_y}-1)^2}+C^2\sigma^2\left(1 + \frac{4e^{\epsilon_y}}{(e^{\epsilon_y}-1)^2}\right)$ in the following.
\paragraph{Deriving a Recursion.}

Let $t \geq 0$. Then
\begin{align*}
 \EE \left[\norm{\theta_{t+1}-\theta^\star}^2\mid \theta_0,\dots,\theta_{t}\right] &\stackrel{\eqref{eq:sgd}}{=}\EE \left[\norm{\Pi_\Theta(\theta_{t}-\gamma g_t)-\theta^\star}^2\mid \theta_0,\dots,\theta_{t}\right]\\
 \cause{As $\Theta$ is a convex bounded set and $\theta^*\in\Theta$, we use contraction of the projection.}
 &\le \EE \left[\norm{\theta_{t}-\gamma g_t-\theta^\star}^2\mid \theta_0,\dots,\theta_{t}\right]\\
 &=
 \EE\left[\norm{\theta_{t}-\theta^\star}^2 - 2\gamma \langle g_t,\theta_t-\theta^\star\rangle+ \gamma^2 \norm{g_t}^2 \mid \theta_0,\dots,\theta_{t}\right] \\
 &\stackrel{\eqref{def:unbiased}}{=}\norm{\theta_{t}-\theta^\star}^2 - 2 \gamma \langle\nabla \risk(\theta_t),\theta_t-\theta^\star\rangle + \gamma^2 \EE\left[\norm{g_t}^2 \mid \theta_0,\dots,\theta_{t}\right] \\
& \stackrel{(\ref{def:smooth})}{\leq} \norm{\theta_{t}-\theta^\star}^2 - 2 \gamma \left(\frac{\mu}{2} \norm{\theta_t-\theta^\star}^2 + \risk(\theta_t)- \risk^\star \right) + \gamma^2 \bigl(2\sK(\risk(\theta_t)-\risk^\star)+ \gamma^2A,
\end{align*}
where we also used $\mu$-strong convexity in the last inequality. By re-arranging and taking expectation on both sides, we get:
\begin{align*}
 \EE {\norm{\theta_{t+1}-\theta^\star}^2} & \leq (1-\mu \gamma) \EE {\norm{\theta_{t}-\theta^\star}^2} - 2 \gamma (1- \sK\gamma) (\EE{\risk(\theta_t)} - \risk^\star) + \gamma^2A,
\end{align*}
and by observing $(1-\sK\gamma) \geq \frac{1}{2}$ for $\gamma \leq \frac{1}{2\sK}$,
\begin{align}
 \EE {\norm{\theta_{t+1}-\theta^\star}^2} & \leq (1-\mu \gamma) \EE {\norm{\theta_{t}-\theta^\star}^2} - \gamma (\EE{\risk(\theta_t)} - \risk^\star) + \gamma^2A \eqsp. 
 \label{eq:one_step}
\end{align}

\paragraph{Unrolling the Recurrence.}
We can relax \eqref{eq:one_step} to $\EE{\norm{\theta_{t+1}-\theta^\star}^2} \leq (1-\mu \gamma) \EE{\norm{\theta_t - \theta^\star}^2} + \gamma^2A$ and obtain after unrolling the recurrence for any $n\ge1$,
\begin{align}
\nonumber
\EE{\norm{\theta_{n}-\theta^\star}^2} &\leq (1-\mu \gamma)^n \norm{\theta_0 - \theta^\star}^2 + \gamma^2A\sum_{i=0}^{n-1} (1-\mu\gamma)^i\\
\nonumber
&\leq (1-\mu \gamma)^n \norm{\theta_0 - \theta^\star}^2 + \frac{\gamma}{\mu}A\\
&\leq (1-\mu \gamma)^n \norm{\theta_0 - \theta^\star}^2 + \frac{\gamma C^2}{\mu}\left(\frac{4e^{\epsilon_y}}{(e^{\epsilon_y}-1)^2}+\sigma^2\left(1 + \frac{4e^{\epsilon_y}}{(e^{\epsilon_y}-1)^2}\right)\right) \eqsp.
\label{eq:sgd_norm}
\end{align} 
This intermediate results shows that SGD with constant stepsizes reduces the initial error term $\norm{\theta_0 - \theta^\star}^2$ linearly, but only converges towards a $\sO\left(\frac{\gamma}{\mu}\frac{C^2e^{\epsilon_y}}{(e^{\epsilon_y}-1)^2}(\sigma^2+1) \right)$-neighborhood of $\theta^\star$.

\paragraph{Choosing the Step-size.}
To obtain a convergence guarantee that holds for arbitrary accuracy, we need to choose the stepsize $\gamma$ carefully:
\begin{itemize}
 \item If $\frac{1}{2\sK} \geq \frac{1}{\mu n}\log \max\left(2,\frac{\mu^2 \norm{\theta_0-\theta^\star}^2 n}{A}\right)$ then we choose $\gamma = \frac{1}{\mu n}\log \max\left(2,\frac{\mu^2 \norm{\theta_0-\theta^\star}^2 n}{A}\right)$.
 \item If otherwise $\frac{1}{2\sK} < \frac{1}{\mu n}\log \max\left(2,\frac{\mu^2 \norm{\theta_0-\theta^\star}^2 n}{A}\right)$ then we pick $\gamma = \frac{1}{2\sK}$.
\end{itemize}
With these choices of $\gamma$, we can show
 \begin{align}
 \nonumber
 \EE{\norm{\theta_{n}-\theta^\star}^2}  &=  \tilde \sO \left(  \norm{\theta_0-\theta^\star}^2 \exp \left[- \frac{\mu n}{2\sK} \right] + \frac{A}{\mu^2 n} \right)\\
 &=  \tilde \sO \left(  \norm{\theta_0-\theta^\star}^2 \exp \left[- \frac{\mu n}{2\sK} \right] + \frac{C^2e^{\epsilon_y}}{\mu^2n(e^{\epsilon_y}-1)^2}(\sigma^2+1) \right).
 \label{eq:sgd_conv_app_gamma}
 \end{align}
Where the $\tilde \sO(\cdot)$ notation hides logarithmic factors in $n$.
Replacing $\sigma^2$ with its value $\frac{8C^2\log(1.25/\delta)}{\epsilon_x^2}$ in Equations~\eqref{eq:sgd_norm}~and~\eqref{eq:sgd_conv_app_gamma} yields the desired results.
\end{proof}

\section{Experiments}
\label{appendix:experiments}
In this section, we give more details about experiments of Section~\ref{sec:experiments}. Given a test dataset of $m$ samples $D'_m = \{ (x_i,y_i)\}_{i=1}^m$, the accuracy of the linear classification model $\theta$ on the test set is denoted $\sA(\theta)$ and defined as
$$\sA(\theta) = \frac{1}{m}\sum_{i=1}^m \mathbbm{1}(y_i\theta^\top x_i >0).$$

\paragraph{Synthetic Data.} Recall that we study two synthetic binary classification problems in dimension $p=2$ and $p=10$ generated with the \verb|make_classification| routine of \verb|scikit-learn| having features within $[-1,1]^p$. We conduct the experiments on $n=10^6$ samples for two privacy guarantees : $(2,10^{-5})$-LDP for $p=2$ and $(5,10^{-5})$-LDP for $p=10$. The $\ell_2$ regularization constant is $\lambda=5$ with the regularized loss $\ell(\theta,x,y) +\lambda\norm{\theta}^2/2$. We average batches of size 128 and use a common learning rate of $\gamma=10^{-4}$.

\paragraph{Real Data.} Recall that, we study the ACSIncome and ACSPublicCoverage problems of the Folktables dataset \citep{folktables}.  Both are based on ACS data (like UCI Adult), illustrating the fact that we can reuse, in a task-agnostic way, the same private releases when reusing the same data points. ACSIncome consists of predicting whether an individual’s income is above \$50~000 and ACSPublicCoverage consists of predicting individual coverage from health insurance. For both problems, we select the two variables \emph{AGEP} (age in years) and \emph{SCHL} (educational attainment). For ACSIncome we add \emph{WKHP} (usual hours worked per week over the past year) and for ACSPublicCoverage we add \emph{PINCP} (total annual income). All features are continuous or ordinal, allowing the use of the Gaussian mechanism. We merge the data of the five largest states yielding datasets of respectively 668~859 rows and 883~984 rows for ACSIncome and ACSPublicCoverage. The data is then randomly split into training (80\%) and test (20\%) sets. The $\ell_2$ regularization constant is $\lambda=10$ with the regularized loss $\ell(\theta,x,y) +\lambda\norm{\theta}^2/2$. We average batches of size 50 and use a common learning rate of $\gamma=2\cdot10^{-5}$ for ACSPublicCoverage and ACSIncome. Figure~\ref{fig:Acc_exp_comp_folktables} is showing the accuracy convergence across batches for these experiments.

\begin{figure}
   \centering
    \begin{subfigure}[b]{0.48\textwidth}
        \centering
        \includegraphics[trim=0 0 0 0, clip,width=0.95\textwidth]{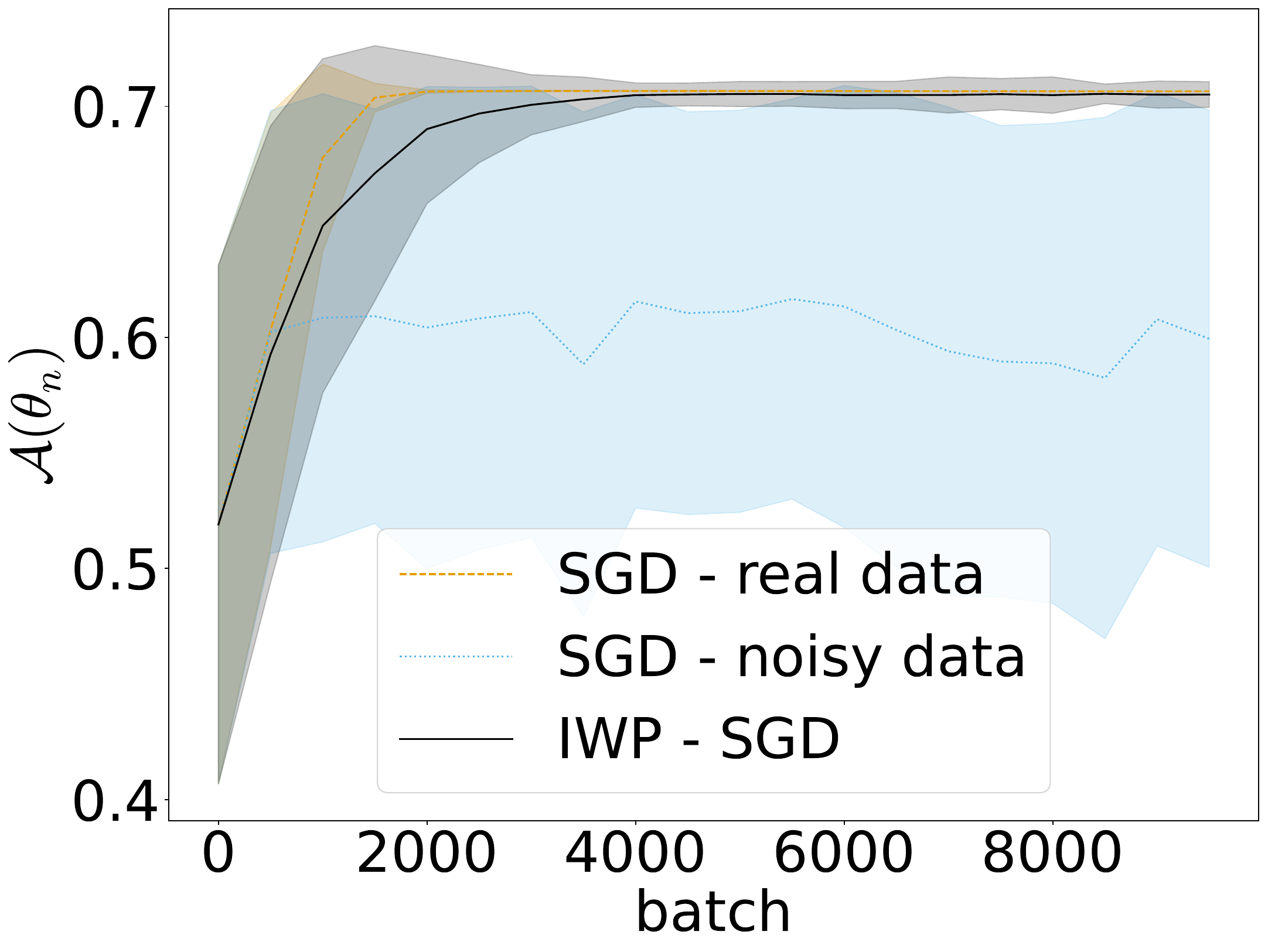}
        \caption{ACSPublicCoverage}
        \label{fig:Acc_exp_comp_folktables_ACSPublicCoverage}
    \end{subfigure}
    \begin{subfigure}[b]{0.48\textwidth}
        \centering
        \includegraphics[trim=0 0 0 0, clip,width=0.95\textwidth]{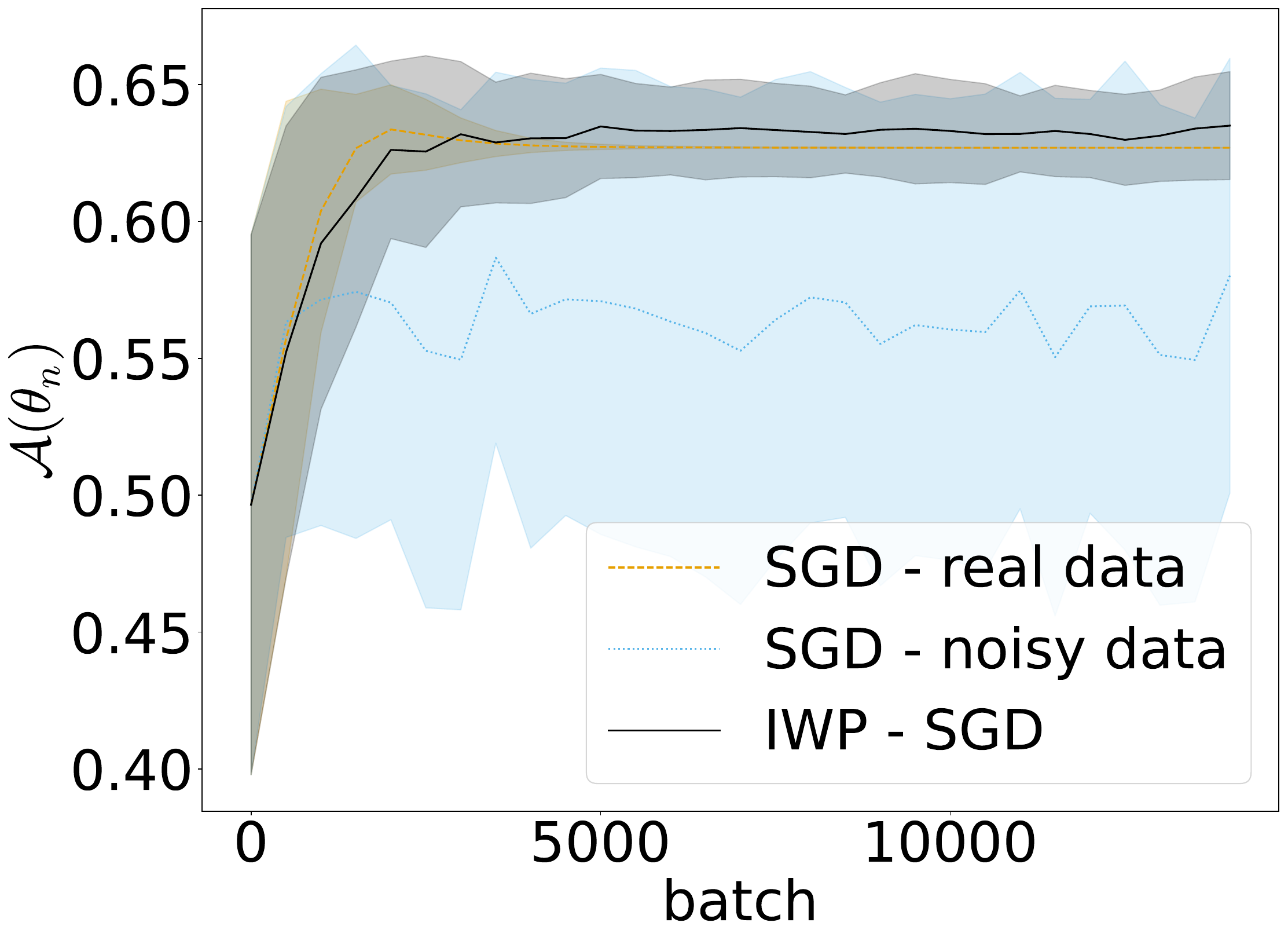}
        \caption{ACSIncome}
        \label{fig:Acc_exp_comp_folktables_ACSIncome}
    \end{subfigure}

    \caption{Comparison of Accuracy convergence of the model fitted on exp loss under $(2,10^{-5})$-LDP on ACSPublicCoverage and ACSIncome.}
    \label{fig:Acc_exp_comp_folktables}
\end{figure}

\subsection{Experiments Using the Log Loss}

Using the approximation of $\WW^{-1}$ described in Appendix~\ref{app:no_closed_form}, we applied our experiments on synthetic and real-world datasets to the log loss (with same batch sizes, regularization constants, and learning rates). Figures~\ref{fig:logloss_comp},~\ref{fig:logloss_comp_folktables}~and~\ref{fig:Acc_logloss_comp_folktables} show similar results compared to the experiments on the exponential loss in Section~\ref{sec:experiments}.
\begin{figure}
    \centering
        \includegraphics[trim=0 0 0 0, clip,width=0.44\textwidth]{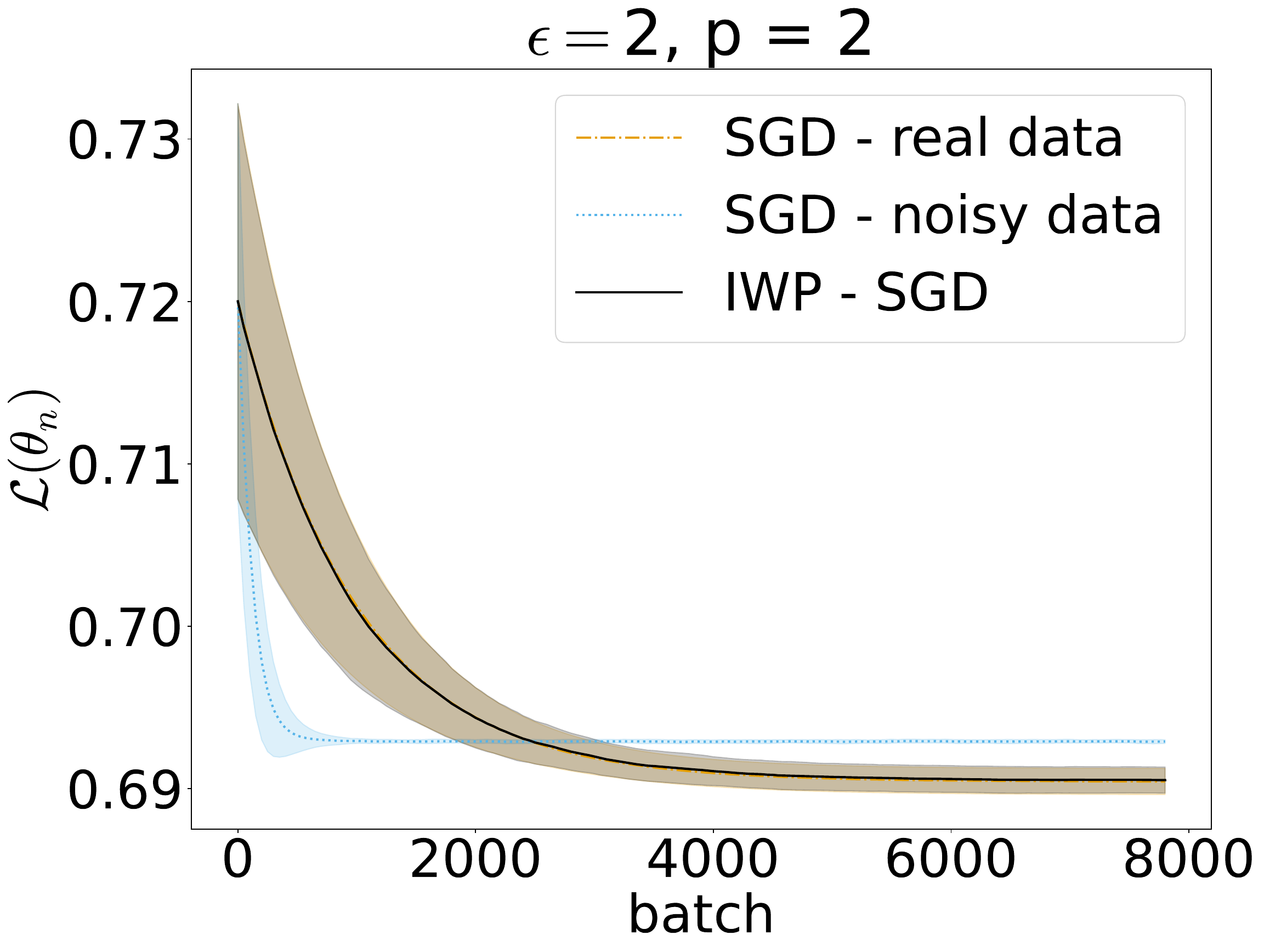}
        \includegraphics[trim=0 0 0 0, clip,width=0.44\textwidth]{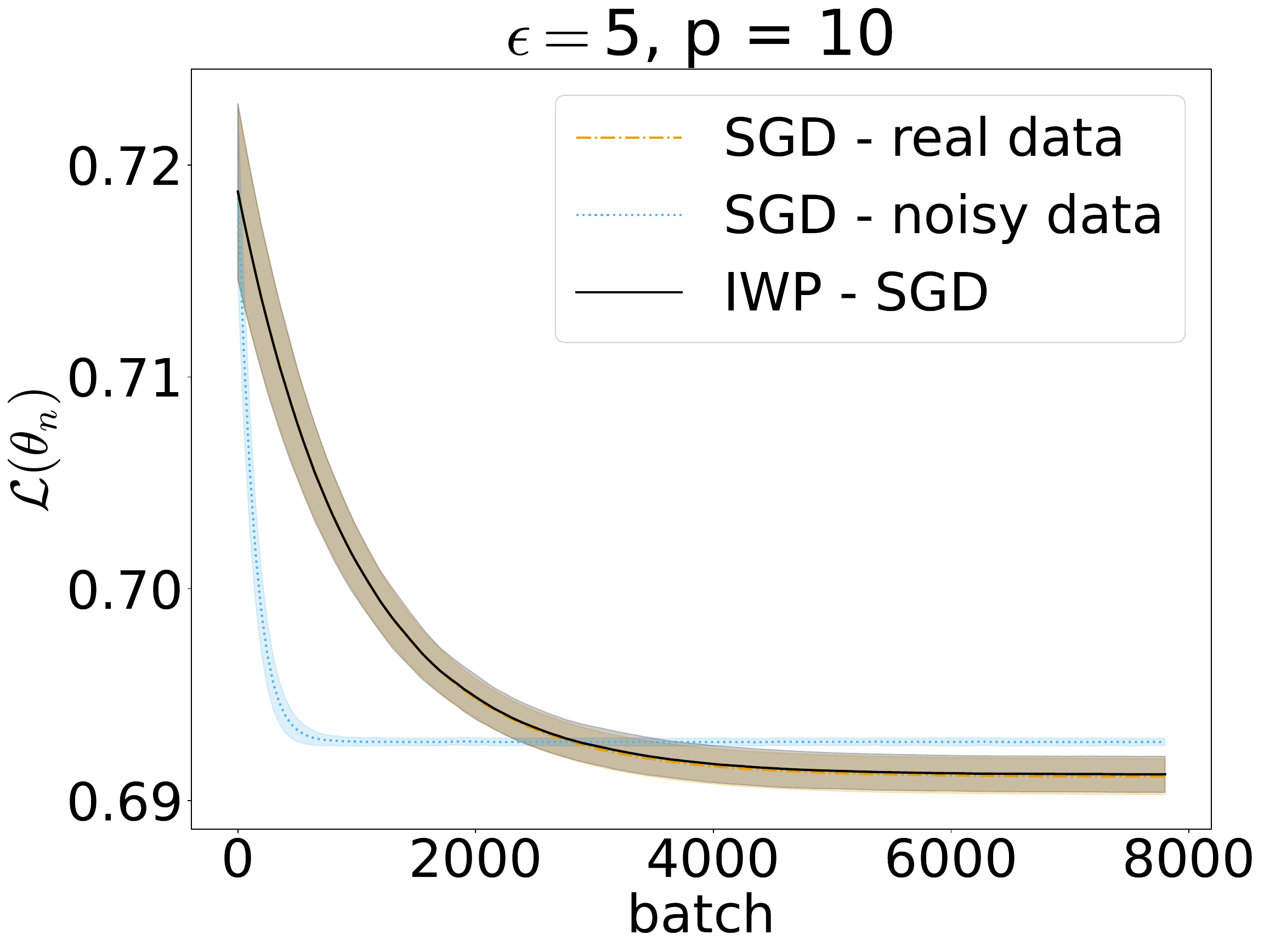}
    \captionof{figure}{Comparison of SGD convergence of the log loss under $(2,10^{-5})$-LDP for the 2-dimensional synthetic data and $(5,10^{-5})$-LDP for the 10-dimensional synthetic data.}
    \label{fig:logloss_comp}
\end{figure}

\begin{figure}
   \centering
    \begin{subfigure}[b]{0.9\textwidth}
        \centering
        \begin{minipage}[b]{0.48\textwidth}
        \includegraphics[trim=0 0 0 0, clip,width=0.9\textwidth]{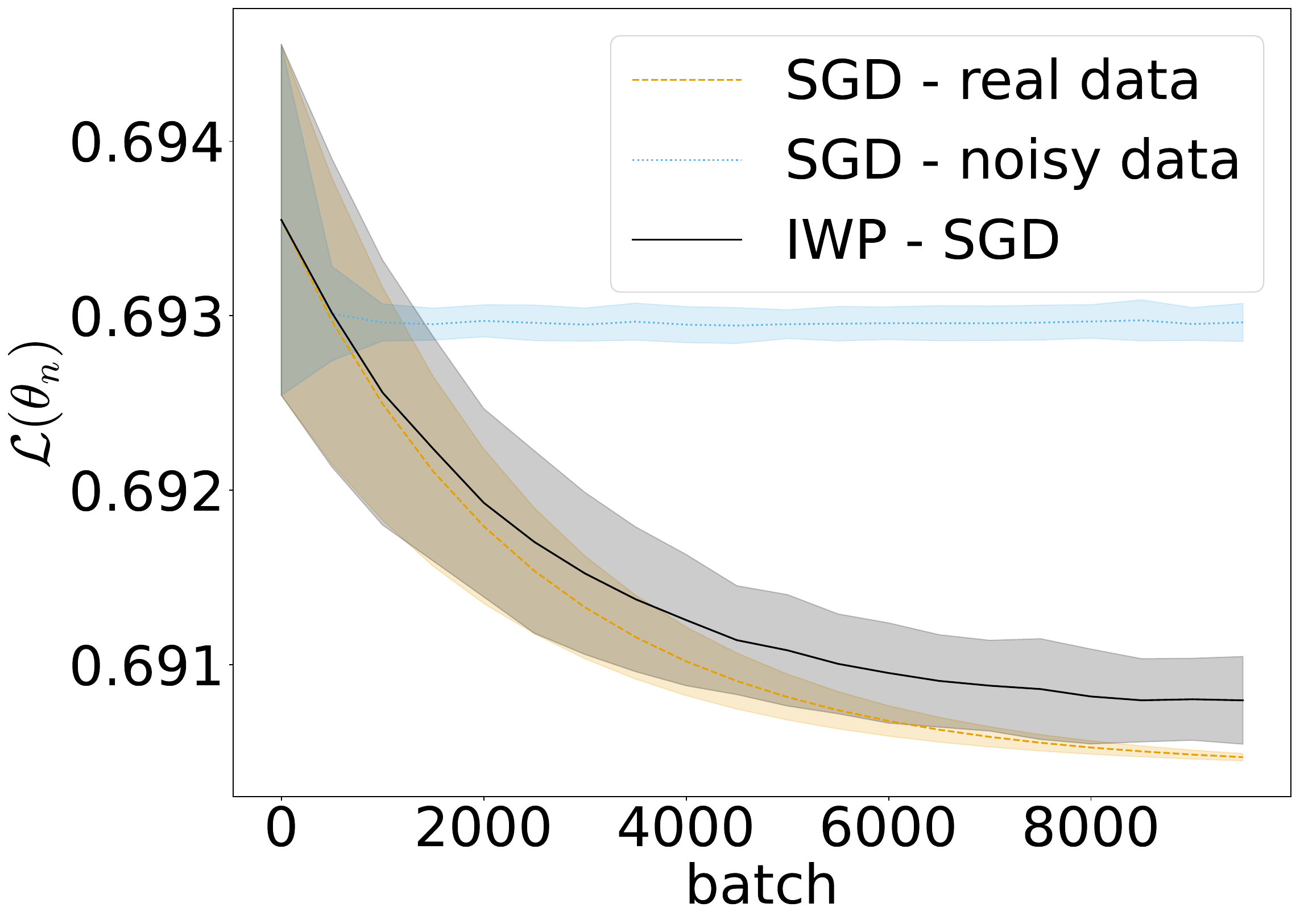}
        \end{minipage}
        \begin{minipage}[b]{0.48\textwidth}
        \includegraphics[trim=0 0 0 0, clip,width=0.9\textwidth]{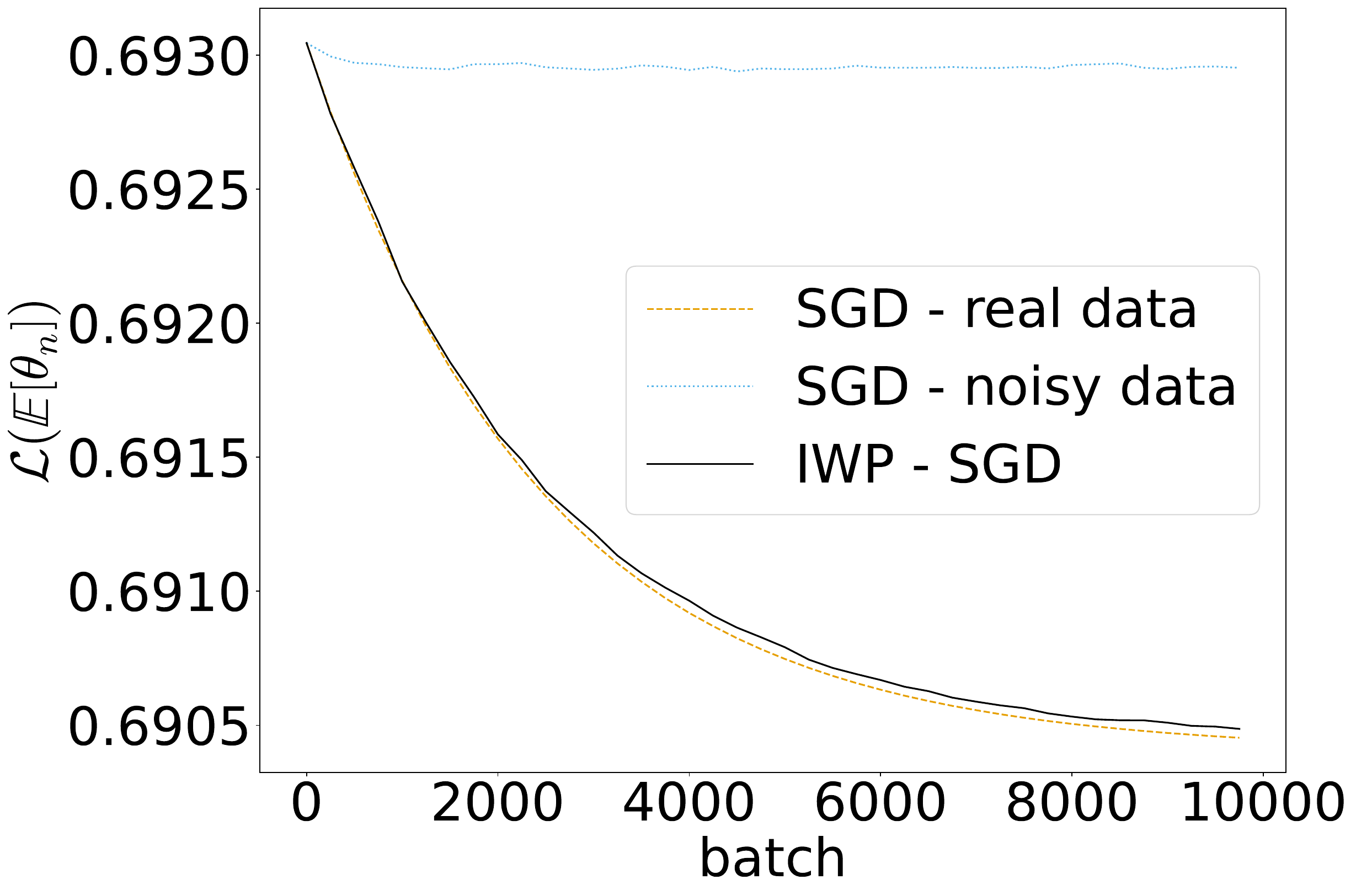}
        \end{minipage}
        \caption{ACSPublicCoverage}
        \label{fig:logloss_comp_folktables_ACSPublicCoverage}
    \end{subfigure}
    \begin{subfigure}[b]{0.9\textwidth}
        \centering
        \begin{minipage}[b]{0.48\textwidth}
        \includegraphics[trim=0 0 0 0, clip,width=0.9\textwidth]{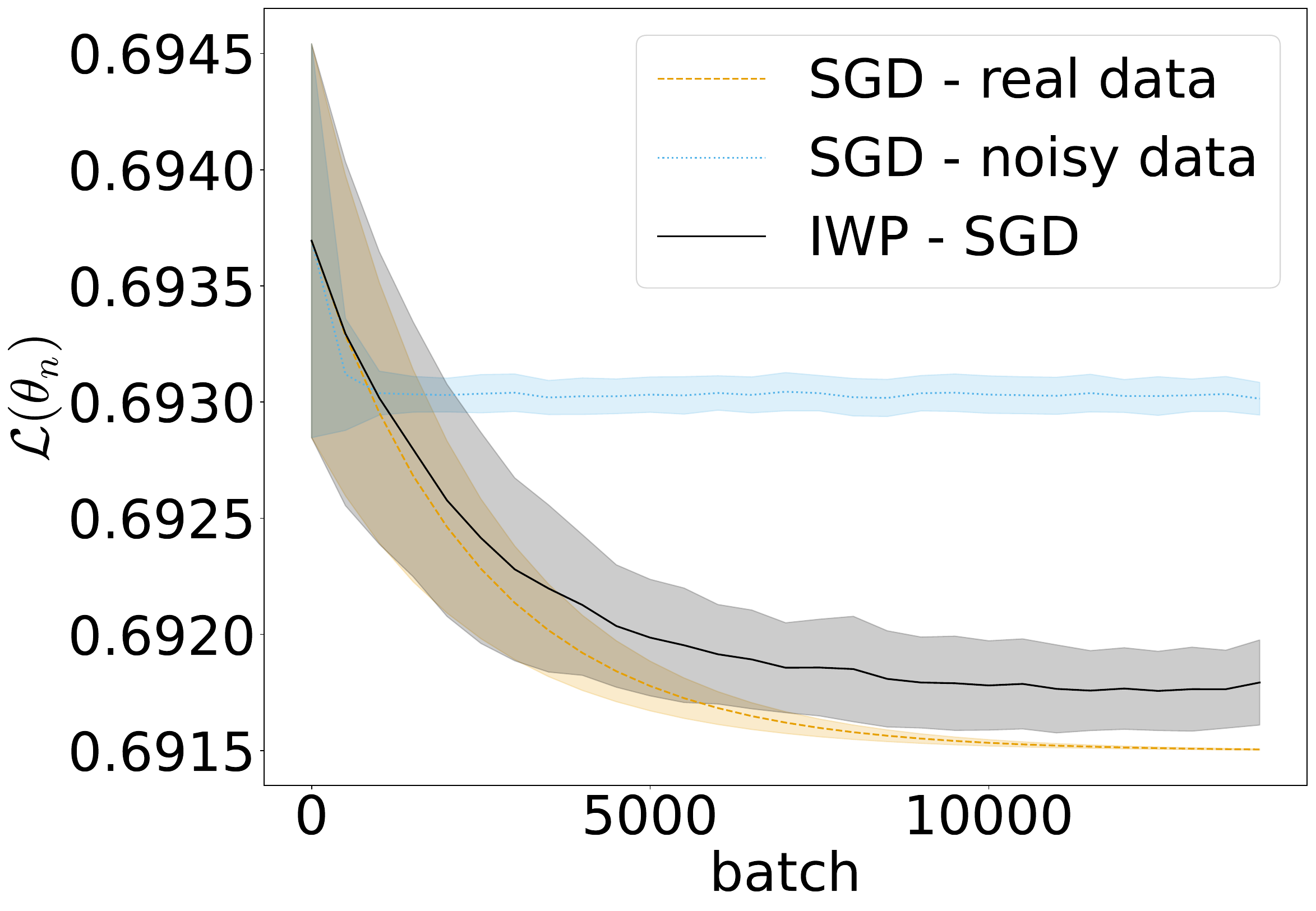}
        \end{minipage}
        \begin{minipage}[b]{0.48\textwidth}
        \includegraphics[trim=0 0 0 0, clip,width=0.9\textwidth]{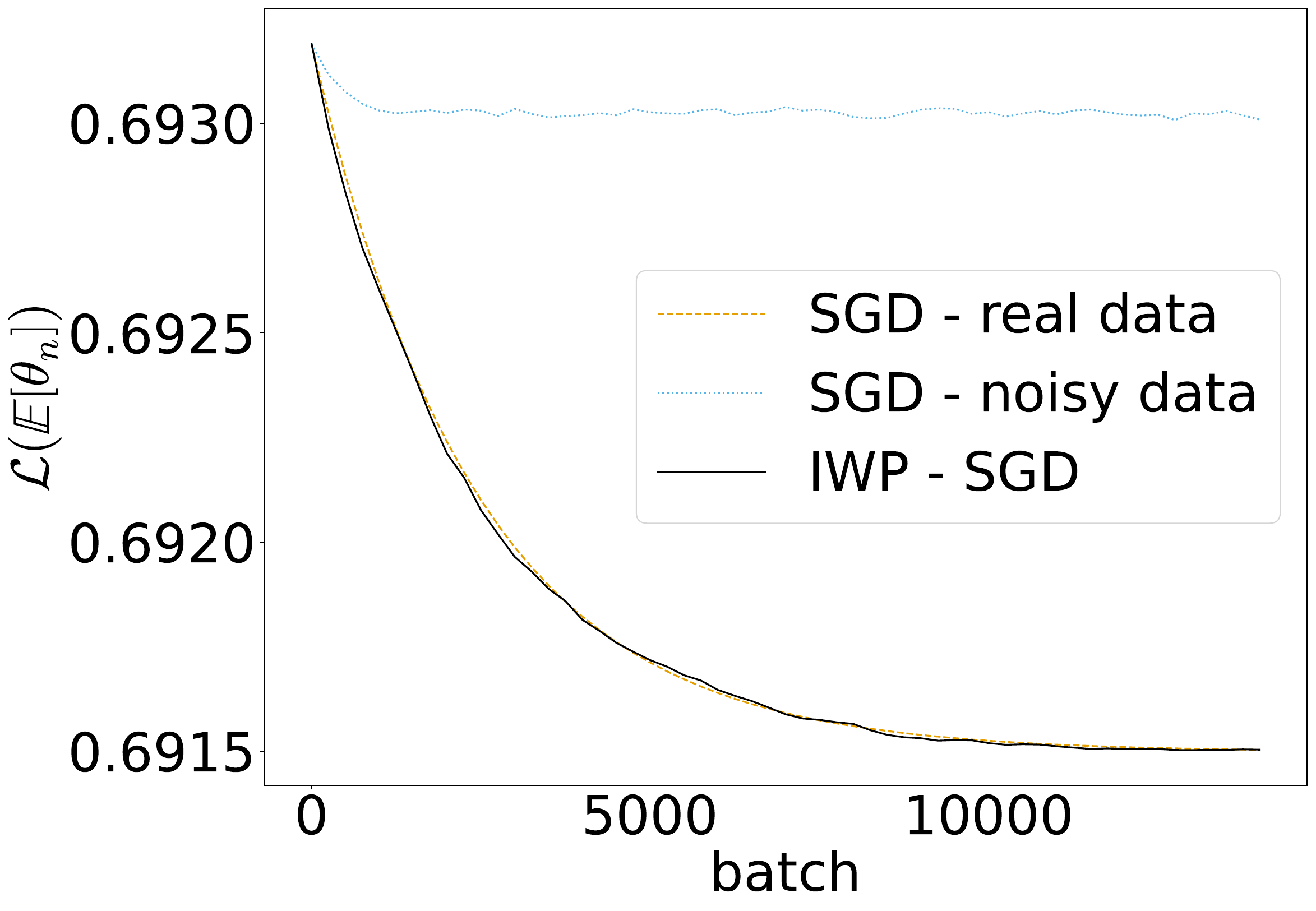}
        \end{minipage}
        \caption{ACSIncome}
        \label{fig:logloss_comp_folktables_ACSIncome}
    \end{subfigure}

    \caption{Comparison of SGD convergence of the log loss under $(2,10^{-5})$-LDP on ACSPublicCoverage and ACSIncome.}
    \label{fig:logloss_comp_folktables}
\end{figure}

\begin{figure}
   \centering
    \begin{subfigure}[b]{0.48\textwidth}
        \centering
        \includegraphics[trim=0 0 0 0, clip,width=0.95\textwidth]{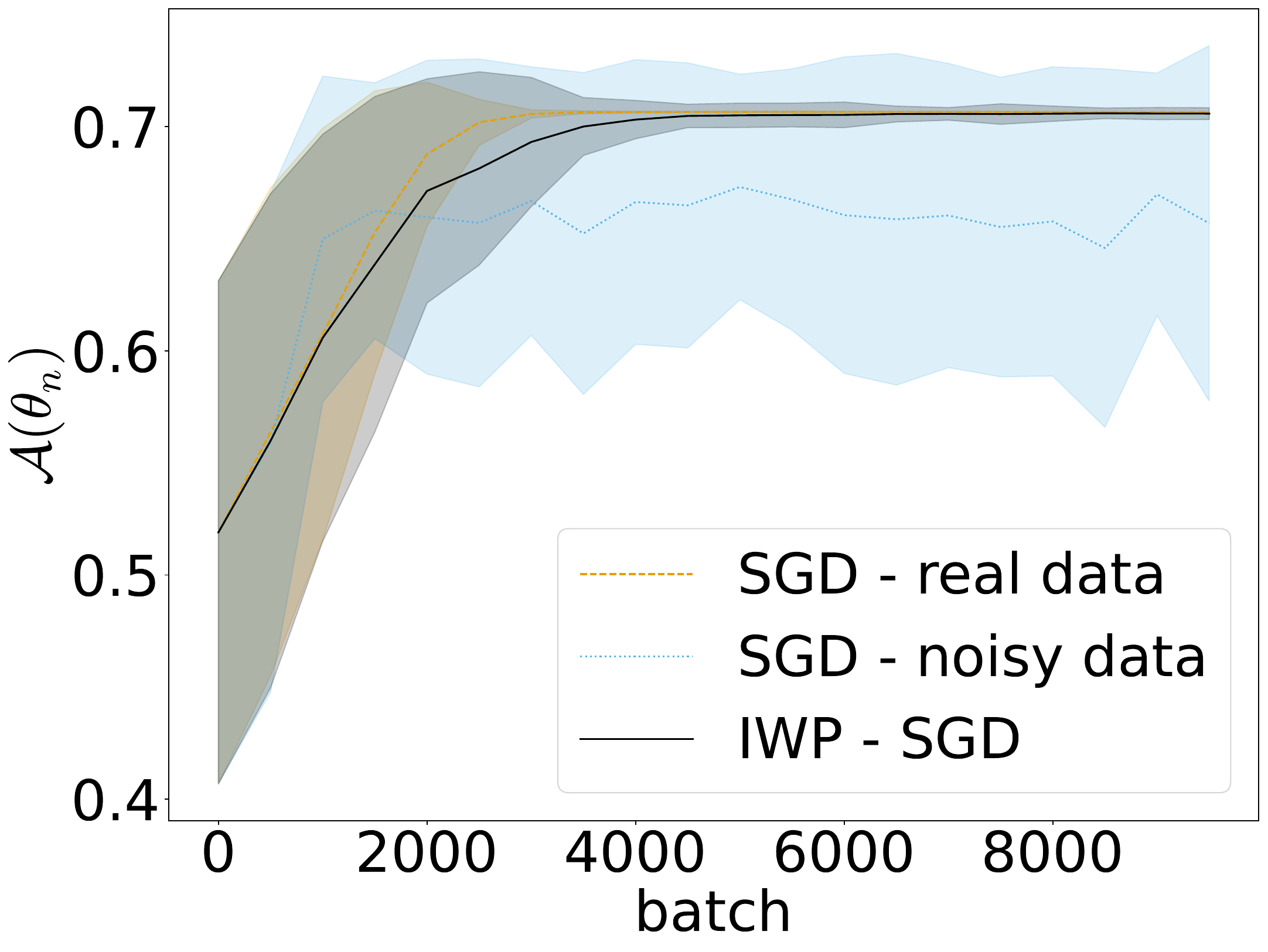}
        \caption{ACSPublicCoverage}
        \label{fig:Acc_logloss_comp_folktables_ACSPublicCoverage}
    \end{subfigure}
    \begin{subfigure}[b]{0.48\textwidth}
        \centering
        \includegraphics[trim=0 0 0 0, clip,width=0.95\textwidth]{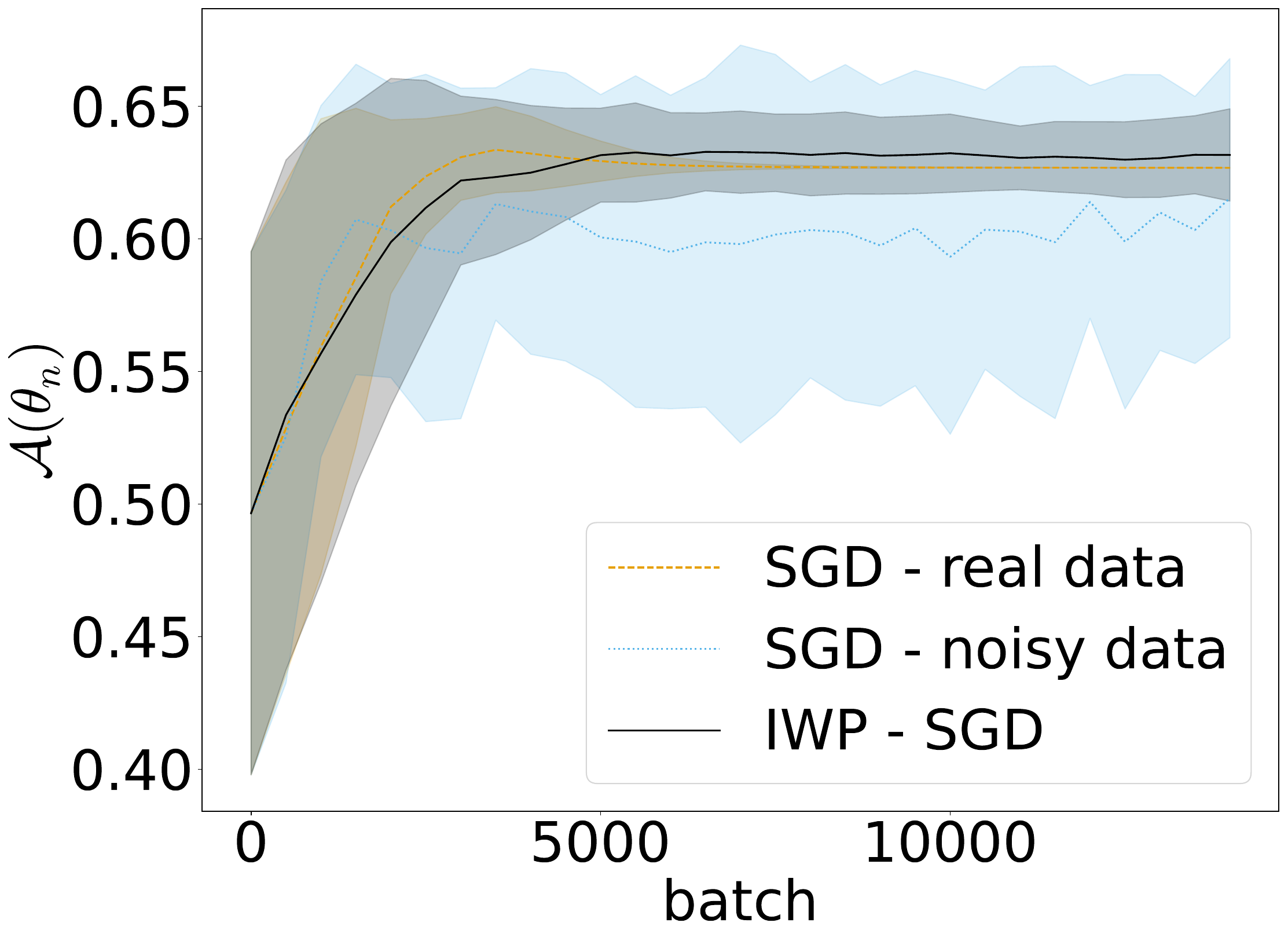}
        \caption{ACSIncome}
        \label{fig:Acc_logloss_comp_folktables_ACSIncome}
    \end{subfigure}

    \caption{Comparison of Accuracy convergence of the model fitted on log loss under $(2,10^{-5})$-LDP on ACSPublicCoverage and ACSIncome.}
    \label{fig:Acc_logloss_comp_folktables}
\end{figure}

\subsection{Experiments on Regression}

The model presented on the paper can be generalized to regression. In this case, we only use the Weierstrass transform because the target is also continuous, and we have:
$$\TT_{\epsilon,\delta}[h](x,y) = \WW_{\sigma_{\epsilon_x,\delta_x}^2}\left[\WW_{\sigma_{\epsilon_y,\delta_y}^2}\left[h \right](\cdot,y) \right](x)$$
with $\epsilon=\epsilon_x+\epsilon_y$ and $\delta = \delta_x+\delta_y$. In the linear regression model where $\sY\subset\RR$, we have $\ell(\theta,x,y) = \frac{1}{2}(\theta^\top x - y)^2$ and there is no bias in the gradient to correct with respect to the labels. Indeed, it is linear with respect to $y$:
$$\nabla_\theta\ell(\theta,x,y) = x\theta^\top x - xy.$$

We then study a variant of the ACSIncome consisting in the prediction of the individual's income as a continuous value instead of the threshold at \$50~000. For this, we adapt our method to continuous output by replacing the Randomized Response transform by a second Weierstrass transform and we consider the Mean Square Error for the loss. We use $\ell_2$ regularization with the constant $\lambda=10$ forming a regularized loss $\ell(\theta,x,y) +\lambda\norm{\theta}^2/2$. We average batches of size 128 and use a learning rate of $\gamma=5\cdot10^{-6}$. Figure~\ref{fig:regression_folktables} shows the results of this experiment. The conclusions are the same as in binary classification, IWP-SGD converges to the same model as SGD - real data but with an increased variance whereas SGD - noisy data converges to a different solution, illustrating the presence of a bias.

\begin{figure}
    \centering
        \includegraphics[trim=0 0 0 0, clip,width=0.8\textwidth]{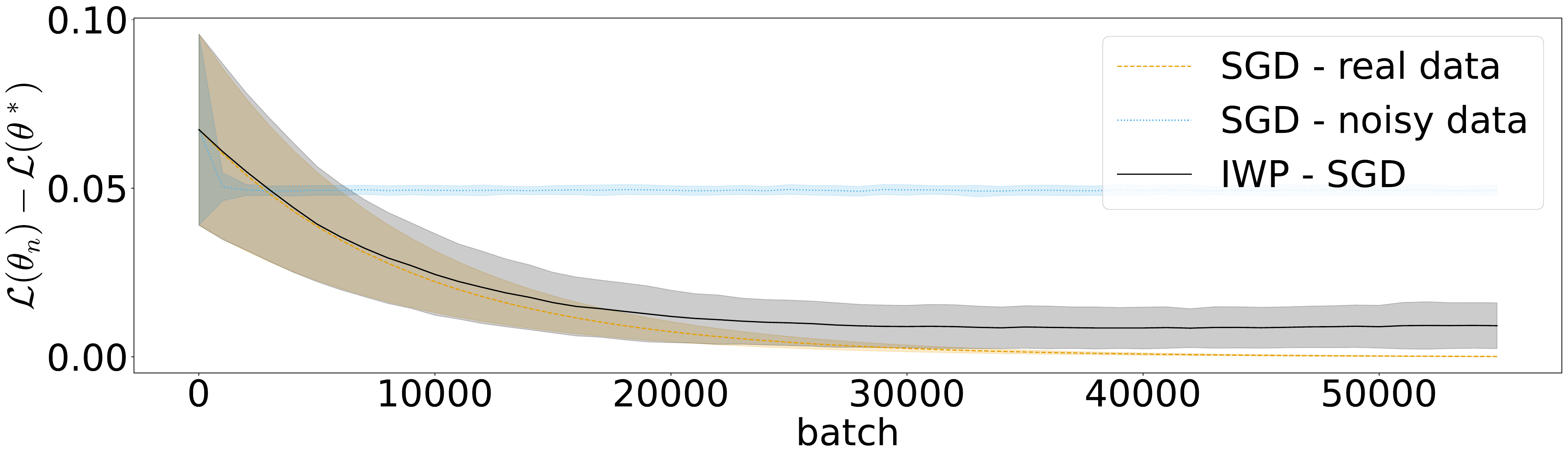}
        \includegraphics[trim=0 0 0 0, clip,width=0.8\textwidth]{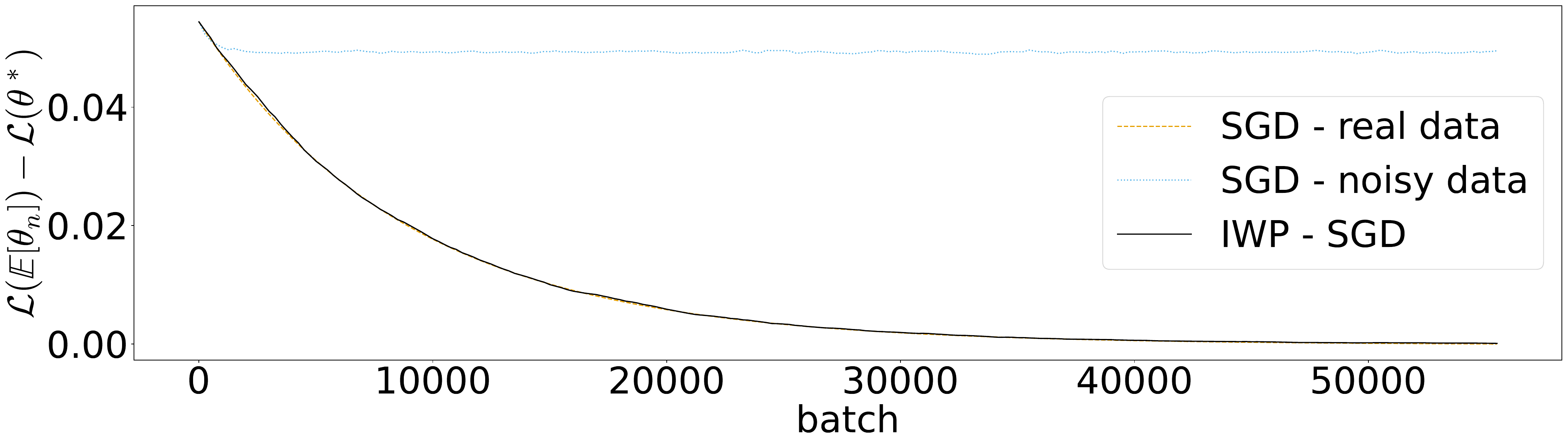}
    \captionof{figure}{Comparison of SGD convergence under $(2,10^{-5})$-LDP on ACSIncome linear regression variant.}
    \label{fig:regression_folktables}
\end{figure}

\end{document}